\documentclass[Afour,sageh,times]{sagej}
\usepackage[utf8]{inputenc}

\usepackage[utf8]{inputenc}
\usepackage{graphicx}

\usepackage{cite}

\usepackage{soul}

\usepackage{subcaption}
\usepackage{multirow}

\usepackage{graphicx}  
\usepackage{caption}   
\usepackage{ragged2e}  

\usepackage[caption=false,font=footnotesize,position=top]{subfig}
\usepackage{xcolor}
\usepackage{import}
\usepackage{booktabs}
\usepackage{pifont}

\usepackage{lipsum}
\usepackage[normalem]{ulem}

\usepackage{optidef}
\usepackage{amsmath}
\usepackage{bm}
\usepackage{bbold}

\usepackage{array}
\usepackage{color}
\usepackage{colortbl}

\usepackage{algpseudocode}
\usepackage[linesnumbered,ruled,vlined]{algorithm2e}

\usepackage{amsmath,amssymb,amsfonts,amsthm,bbm}
\usepackage{stmaryrd}
\usepackage{mathrsfs}
\usepackage{array}
\usepackage[caption=false, font=footnotesize, position=top]{subfig}
\usepackage[usestackEOL]{stackengine}
\usepackage{textcomp}
\usepackage{stfloats}
\usepackage{url}
\usepackage{verbatim}
\usepackage{graphicx}
\usepackage{cite}
\usepackage{enumitem}
\usepackage{accents}
\usepackage{cases}
\usepackage{setspace}
\usepackage{optidef}

\newcommand{\AB}[1]{\textcolor{black}{#1}}
\newcommand{\ABB}[1]{\textcolor{black}{#1}}
\newcommand{\ab}[1]{\textcolor{black}{#1}}
\newcommand{\xmark}{\ding{55}}%
\newcommand{\cmark}{\ding{51}}%

\usepackage[font=small]{caption}

\newtheorem{definition}{Definition}
\newtheorem{assumption}{Assumption}
\newtheorem{remark}{Remark}
\newtheorem{theorem}{Theorem}

\newtheorem{problem}{Problem}

\def\volumeyear{2025}
\begin{document}

\runninghead{Bonetti, Proia, Guidetti, and Sabattini}

\title{A Traffic Management System for Large and Heterogeneous Vehicles in Narrow Industrial Environments}

\author{Alessandro Bonetti\affilnum{1}, Silvia Proia\affilnum{1}, Simone Guidetti\affilnum{2}, and Lorenzo Sabattini\affilnum{1}}

\affiliation{\affilnum{1}Department of Sciences and Methods for Engineering (DISMI), University of Modena and Reggio Emilia, Italy.\\\affilnum{2}Gruppo TecnoFerrari S.p.a. con socio unico, Italy.}

\corrauth{Silvia Proia, Department of Sciences and Methods for Engineering (DISMI), University of Modena and Reggio Emilia, Italy}

\email{silvia.proia@unimore.it}

\begin{abstract}
The coordination of Automated Guided Vehicles (AGVs) in high-density industrial environments represents a critical challenge within Logistics 4.0, as traditional traffic management methods often lead to inefficiencies \AB{caused by negotiation-based priority assignment}. \AB{To overcome the resulting limitations, this paper presents an innovative AGV traffic management system based on a Lifelong Multi-Agent Path Finding (L-MAPF) algorithm operating on roadmaps generated with Non-Uniform Rational B-Splines (NURBS) curves. The approach guarantees locally optimal coordination and ensures safe operation of large and heterogeneous AGVs. Building on this concept,}
the proposed framework integrates a modified version of the Bounded Horizon Conflict Based Search (CBS) technique within a Rolling Horizon Conflict Resolution strategy, utilizing an extended time horizon for each agent to enable effective conflict resolution in corridors identified by a topological map.
\AB{In contrast to state-of-the-art methods for AGV fleet traffic management}, the proposed solution is designed for real-world, non-standardized \AB{(i.e., non
grid-like)} industrial settings characterized by narrow bidirectional corridors and high-traffic density, where AGVs of various sizes and capabilities operate simultaneously. 
\ab{Key contributions include an anytime conflict resolution strategy with adaptive time horizon regulation, an execution layer for safe and standard-compliant interaction with real AGVs, and an advanced mechanism for deadlock detection and resolution.} Experimental results obtained in realistic industrial environments demonstrate higher throughput, with improvements of up to 11\% over a conventional rule-based traffic management system, a state-of-the-art industrial method, and a priority-based L-MAPF variant, while maintaining continuous operation and improved efficiency.
\end{abstract}


\maketitle
\section{Introduction} 
\label{intro}

In the Logistics 4.0 paradigm \citep{winkelhaus2020logistics}, which aims at creating intelligent, interoperable, and autonomous environments, the problem of managing and controlling fleets of mobile robots is attracting enormous interest \citep{proia2022safe,azadeh2019robotized}.

In order to develop an automated industrial system, Automated Guided Vehicles (AGVs) play a crucial role due to their ability to optimize material handling tasks, improve inventory management, enhance safety, and increase operational efficiency \citep{articleAGV}.

Effective AGV operation requires addressing two primary concerns: task assignment and traffic management \citep{robotics10020072}. On the one hand, task assignment consists of dynamically assigning AGVs to specific missions, such as transporting goods between different locations, while considering factors like vehicle availability, task urgency, and overall system efficiency. 
On the other hand, traffic management focuses on generating safe and efficient paths for AGVs and coordinating their movement throughout the environment to ensure smooth operation. Key aspects include establishing the right-of-way and managing AGV interactions in shared spaces, with the objective of preventing collisions and deadlocks while minimizing downtime.


However, to boost productivity and operational effectiveness, traditional rule-based traffic management systems are no longer sufficient to meet the growing demands of modern logistics operations. Specifically, rule-based coordination on AGV paths can lead to inefficiencies, such as bottlenecks \AB{and inappropriate priority handling}, reducing overall efficiency \citep{10132864}. To overcome \AB{the above} limitations, research and industrial sectors are increasingly turning to more advanced, data-driven traffic management solutions that can adapt in real-time to changing conditions and improve performance without relying on fixed rules \citep{DERYCK2020152}. In this context, the application of Multi-Agent Path Finding (MAPF) algorithms is essential to facilitate the navigation of AGVs \AB{\citep{WAGNER20151,dergachev2021distributed, OKUMURA2022103752}.}
\AB{The MAPF problem, a well-established combinatorial challenge in robotics and artificial intelligence, focuses on identifying collision-free trajectories for agents with predetermined start and goal locations within a specified environment \citep{zhang2024priority}.} Traditional MAPF approaches primarily address a static, ``one-shot" version of the problem. However, to handle dynamic scenarios where tasks change over time \AB{and where operating conditions are affected by uncertainties such as vehicle alarms, manual interventions, or safety scanner activations}, the Lifelong MAPF (L-MAPF) framework is introduced. Specifically, the lifelong approach continuously resolves MAPF instances in real-time, accounting for a series of tasks throughout the operational lifespan of the agents \citep{li2021lifelong}. By coordinating the simultaneous movement of multiple AGVs, L-MAPF algorithms ideally prevent collisions and deadlocks in dynamic environments.


\begin{figure}
     \centering
\captionsetup{justification=justified}
     \begin{subfigure}[b]{\linewidth}
         \centering
         \includegraphics[width=0.9\textwidth]{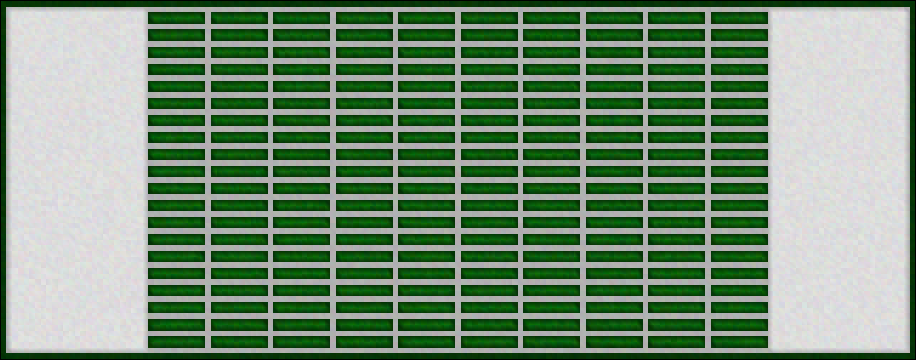}
         \caption{$ $}
         \label{fig:1.2}
     \end{subfigure}
     \hfill
     \begin{subfigure}[b]{\linewidth}
         \centering
         \includegraphics[width=0.9\textwidth]{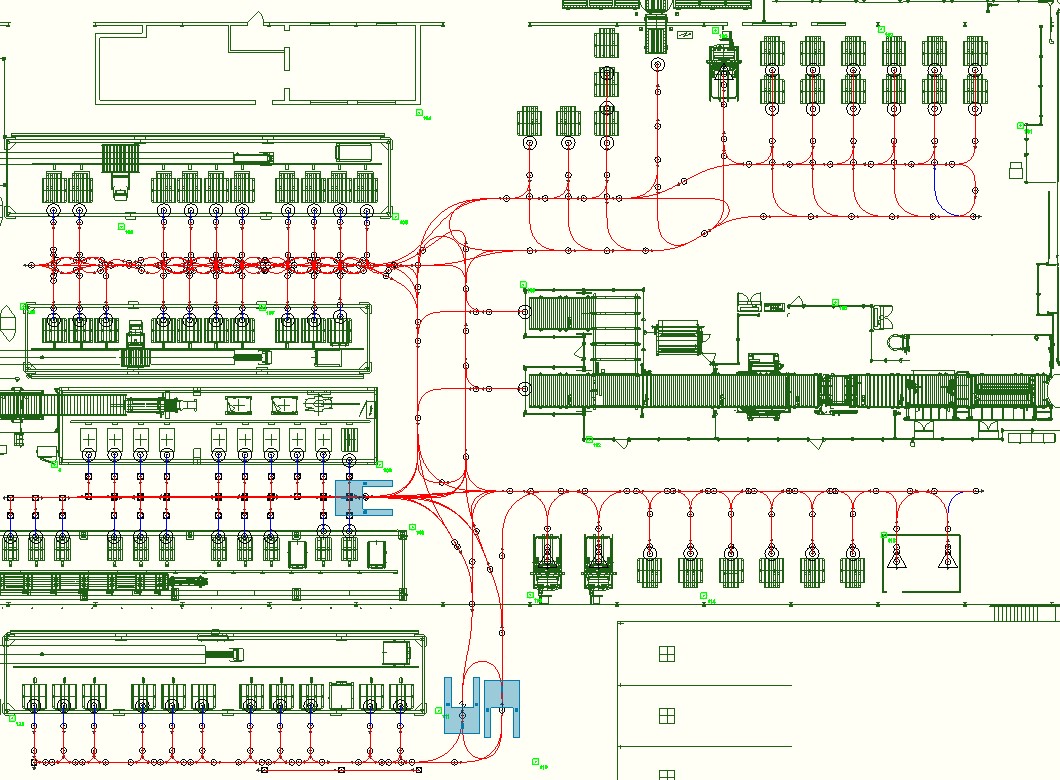}
         \caption{$ $}
         \label{fig:1.1}
     \end{subfigure}
        \caption{
        The figure illustrates a standardized warehouse environment modeled as a gridmap of vehicle-sized cells \citep{stern2019mapf} (a) compared to a non-standardized complex layout, which depicts a plant of our use case derived from a real application (b). Specifically, (b) presents the plant configuration, where green lines represent stations, machinery, battery chargers, and other relevant features within the layout, red lines indicate the roadmap, and blue shapes represent AGV footprints.}
        \label{fig:intro}
        \vspace{-0.5cm}
\end{figure}

Nevertheless, the majority of L-MAPF algorithms have been applied to standardized, \AB{i.e., grid-like,} environments \citep{song2023anytime,stern2019mapf} (see Fig.\ref{fig:1.2}), and rely on \AB{assumptions} that may limit their applicability to more complex systems typically encountered in industrial settings \citep{gao2023review}, \AB{such as homogeneous vehicles and unit-time actions.} \ABB{The latter implies that either all agents move at identical speeds over uniform edges, or that motion profiles are artificially adjusted so that each traversal consumes exactly one discrete time step \citep{andreychuk2022multi}. While this abstraction simplifies theoretical analysis, it significantly constrains real-world applicability. Industrial AGVs usually operate on roadmaps, where traversal time naturally results from the chosen velocity profile and the segment length. Hence, enforcing unit-time actions prevents accurate modeling of heterogeneous traversal times and may lead to inefficient coordination.} 
Additionally, the inherent constraints of real-world implementations, coupled with unpredictable conditions, can still lead to deadlocks \citep{app11146494}. Therefore, to effectively mitigate the above issues, even the most advanced L-MAPF systems must integrate mechanisms for detecting and resolving deadlocks, \AB{ensuring conflict-free and continuous AGV operation.}

In this paper, we address the challenges of traffic management for large and heterogeneous AGVs operating in complex, non-standardized industrial environments (see Fig.\ref{fig:1.1}). \ABB{The ultimate goal of the present work is to design and validate an integrated, deployable architecture to support AGV operations in automated warehouses and intralogistics plants, advancing the state-of-the-art in practical multi-agent coordination.} \ABB{In fact,} although we present a complete traffic management system, our primary focus is on fleet coordination. 
\ABB{Specifically, we propose a novel coordination strategy based on an L-MAPF algorithm designed to periodically compute collision-free trajectories that meet the demands of real-world scenarios. To ensure safe execution of the L-MAPF solutions under execution uncertainties, we introduce a path allocator module that bridges the gap between discrete trajectories and continuous AGV motion.}
Additionally, we introduce a cutting-edge strategy for deadlock detection and resolution that ensures continuous operational flow. Extensive testing in both simulations and real-world environments demonstrates real-time capability of the proposed solution and significant improvements in throughput compared to \ABB{a} traditional rule-based coordination \ABB{strategy, a state-of-the-art industrial solution, and an alternative L-MAPF method}.

\AB{The remainder of this paper is structured as follows.
Section \ref{cont} summarizes the related works and sheds light on the main contributions of this work, positioning them with respect to the state-of-the-art. Section \ref{pre} presents the problem statement and introduces key concepts from graph theory used to model the environment, as well as the fundamental notions required to represent the AGVs, together with the related assumptions. The architecture of the proposed traffic management system is then presented in Section \ref{Traffic Management System Architecture}. The environment model, organized into two distinct layers, is described in Section \ref{EnvironmentModel}. Section \ref{PathandTrajectoryDefinition} details the path generation and the trajectory definition. The proposed methodology, including the mathematical formalization of the L-MAPF coordinator, is provided in Section \ref{lampfcoordinator}. Section \ref{enabler} illustrates the path allocator, which ensures the safe execution of the coordinated trajectories, while Section \ref{DeadlockDetectorandHandler} discusses the deadlock detection and resolution module. The experimental setup and the corresponding results are presented in Section \ref{set}. Finally, concluding remarks are provided in Section \ref{Concl}.}


\begin{table*}[t!]
\centering
\vspace{2mm}
\caption{Summary of works related to coordination strategies for AGV traffic management.}
\label{table1}
\resizebox{\textwidth}{!}{%
\begin{tabular}{@{}cccccccccc@{}}
\toprule
\textbf{Algorithm} & \textbf{Ref. No}  & \textbf{Max} & \textbf{Graph} & \textbf{\AB{Local}} & \textbf{\AB{Anytime}} & \textbf{Large \&} & \textbf{Non-Standardized}& \textbf{High-Traffic} & \textbf{Real} \\ 
\textbf{Type} &   & \textbf{Simulated Vehicles} &  & \textbf{\AB{Optimality}} & \textbf{\AB{Strategy}} & \textbf{Heterogeneous Agents} & \textbf{Environment}& \textbf{Density} & \textbf{Implementation} \\ 
\midrule
 & \citet{digani2015ensemble}  & 30 & Topological map + Roadmap & \AB{\xmark} & \AB{\xmark} & \xmark & \xmark & \xmark & \xmark
 \\

\textbf{Rule-based} & \citet{10419190}& 10 & Roadmap & \AB{\xmark} & \AB{\xmark} & \xmark & \xmark & \cmark & \xmark 
 \\

 & \citet{10132864}  & 8& Topological map + Roadmap & \AB{\xmark} & \AB{\xmark}& \xmark & \cmark & \cmark & \cmark \\

& \citet{9355019}  & 3 & Roadmap & \AB{\xmark} & \AB{\xmark} & \cmark& \xmark & \xmark & \xmark\\

\midrule
 & \citet{9782398}  & 20 & Gridmap & \AB{\xmark} & \AB{\xmark} & \xmark & \xmark & \xmark & \cmark
  \\

  & \citet{XIE2024102945}  & 20 & Gridmap & \AB{\xmark} & \AB{\xmark} & \xmark & \xmark & \xmark & \xmark
\\
  
 \textbf{Search-based}& \citet{9171550}  & 1008 & Topological map + Roadmap & \AB{\cmark} & \AB{\xmark} & \xmark & \xmark & \cmark & \xmark
  \\
& \AB{\citet{song2023anytime} }& \AB{150} & \AB{Topological map + Gridmap} & \AB{\cmark} & \AB{\cmark} & \AB{\xmark} & \AB{\xmark} & \AB{\cmark} & \AB{\xmark }\\

& \textit{ours} & 20 & Topological map + Roadmap & \AB{\cmark} & \AB{\cmark} & \cmark & \cmark & \cmark & \cmark\\
\bottomrule
\end{tabular}
} 
\end{table*}

\section{Related Works and Paper Contributions}
\label{cont} 

\subsection{Related Works}
\label{relwork}

In the scientific and industrial fields, the efficient coordination of high-density AGV fleets in irregular and non-standardized industrial settings presents significant challenges due to the constrained solution space and limited system configurations available to ensure both optimal and continuous operation while avoiding collisions and deadlocks. Typically, in real AGV systems, large and heterogeneous vehicles are restricted to following a predefined set of virtual or physical paths, known as a roadmap. This is particularly necessary for large, heavy vehicles operating in confined or irregular spaces, such as narrow corridors, where safety and robustness must be ensured.

Over the past decade, a variety of studies have emerged in the literature concerning the coordination of AGV fleets, which are generally classified into two main categories based on their solvers: rule-based and search-based methods (see Table \ref{table1}) \citep{gao2023review}.

Rule-based coordination builds upon and enhances the traditional concept of industrial traffic management systems by utilizing traffic and negotiation rules to assign priority \AB{and thus to solve conflicts between vehicles}. For instance, in~\citet{digani2015ensemble}, conflicts among AGVs are managed using a hybrid approach that integrates a priority-based negotiation mechanism with a resource allocation strategy, where resources are defined as intersection areas within a structured and organized industrial layout. Despite simulations being conducted with up to thirty AGVs, vehicles are homogeneous and the coordination is performed within fairly standardized settings \AB{without providing any formal guarantee of optimality}. \AB{The experiments are also limited to low-traffic conditions, with no more than four vehicles per sector.}


An alternative method, proposed by \citet{10419190}, presents an improved dynamic resource reservation that exploits multiple dynamic reservations of shared resource points combined with a set of traffic rules that detects, classifies, and solves conflicts as they eventually arise. However, such rules lack general applicability, and the achieved coordination \AB{relies on heuristics}. Furthermore, the approach has only been tested in highly simplified environments with homogeneous vehicles and has not yet been implemented in real-world scenarios.

To overcome the above limitations, \citet{10132864} introduces a Coordination Diagram to solve conflicts by establishing negotiation rules for each timestep within a given time horizon. The proposed space-time approach enables the application of simpler rules that can be easily adapted to real-world, non-standardized, high-traffic environments. However, this coordination strategy has only been tested with eight homogeneous vehicles and does not guarantee \AB{any degree of optimality}.

A more adaptive approach is illustrated in \citet{9355019}, where the authors combine Behavior Trees with Reinforcement Learning to adaptively choose the most effective rule-based strategy from a set of available optional strategies according to diverse, dynamic, and complex situations in Industry 4.0 settings. Despite its potential, this study is limited to a specific and simplified environment involving only three operational vehicles, indicating that further developments are necessary for effective AGV coordination in real industrial scenarios.

Therefore, rule-based solvers exhibit high-speed computation and are capable of addressing medium-scale problems, \AB{but they provide no formal performance guarantees~\citep{gao2023review}}.
\AB{They also require substantial expert configuration to ensure consistent coordination across different industrial plants.}

On the other hand, search-based methods eliminate the need for predefined rules and employ algorithms to methodically address the challenge of vehicle coordination. For instance, in \citet{9782398}, a Coloured Petri Net is employed to coordinate a fleet of AGVs, effectively eliminating active deadlocks and preventing passive ones. This approach has been successfully simulated with up to twenty vehicles and tested in real facilities. However, it \AB{does not guarantee optimality and} is limited to standardized environments modeled as gridmaps. 

\AB{Building on the existing literature,~\citet{XIE2024102945} introduces a centralized planner based on an improved conflict-free A* algorithm for AGV fleet coordination. The proposed algorithm successfully manages the coordination of twenty vehicles, but it does not guarantee optimality and operates solely within a standardized environment that do not reflect real industrial constraints.}

For achieving optimal navigation of AGVs, traffic management of AGV fleets must rely on algorithms specifically designed to address the L-MAPF problem. 
\AB{One remarkable example, presented in~\citet{9171550} has effectively applied this methodology to an AGV coordination framework. In particular, the work employs a hierarchical approach to predict traffic flow within sectors of a topological map and to plan vehicle paths accordingly.}
\AB{Conflicts are resolved in a decentralized manner through the use of Conflict Based Search (CBS) and Cooperative A* algorithms, enabling} the coordination of a large number of agents, potentially up to a thousand, in a \AB{locally} optimal manner.
\AB{However, the considered fleet is homogeneous, and the environment remains highly standardized, consisting of sectors featuring a single intersection connected by unidirectional lanes without intersecting stations. Each sector contains a roadmap represented as a fully discretized graph composed of unit-time segments, where both intersections and lane portions are modeled as edges with uniform traversal duration.}
Furthermore, a significant simplifying assumption is introduced: ``for each new task, the pick-up station and its corresponding working station must differ from those of any unaccomplished tasks already assigned''. In numerous industrial settings, particularly in operations involving palletizers and pallet wrappers, it is common practice to install fewer machines than the number of AGVs in operation. This is due to the high operating speeds of the machines and the imperative to minimize costs. Consequently, coordinating two or more AGVs is often necessary to serve the same working station, ensuring a smooth and efficient workflow. As a result, this method struggles to coordinate vehicles in irregular spaces with narrow, bidirectional corridors and fails to efficiently manage situations where multiple vehicles need to service the same destination station. 

\AB{Building on the Rolling Horizon Conflict Resolution (RHCR) framework \citep{li2021lifelong}, \citet{song2023anytime} introduces an Anytime L-MAPF algorithm on topological maps. The approach first computes locally bounded sub-optimal solutions using Enhanced CBS (ECBS) on a gridmap and then refines them iteratively by applying CBS to the topological map under a fixed computation timeout. In this manner, fast initial solutions are generated and progressively improved over time, enhancing scalability under high-density conditions. However, the proposed framework achieves higher coordination efficiency only in highly standardized environments composed exclusively of corridors, and it further relies on the simplifying assumption that each action has unit time duration. Moreover, the evaluation is limited to simulation, where up to one hundred fifty homogeneous agents are coordinated through a periodic replanning interval of ten seconds. Such an update cycle remains incompatible with the responsiveness required in dynamic industrial operations characterized by frequent task reassignment, human intervention, and uncertainty. Although the method provides a mechanism for gradual refinement of sub-optimal solutions, it still lacks real-time adaptability and neglects heterogeneity in kinematic constraints and motion capabilities.}

Hence, search-based methods can mitigate deadlock formation and can theoretically attain optimal performance using L-MAPF algorithms. However, current solutions in the literature fall short in effectively coordinating fleets of large and heterogeneous AGVs within non-standardized environments characterized by narrow, bidirectional corridors, columns, and other irregularities, particularly under conditions of high-traffic density. Many existing approaches rely on strong assumptions that are not typically applicable to complex industrial settings.

More broadly, existing L-MAPF algorithms exhibit several structural limitations, including the common assumption of homogeneous agents and unit-time actions (\citet{li2021lifelong,madar2022leveraging,9981785,10.5555/3091125.3091243,damani2021primal}, \ABB{\citet{11127445}}). \ABB{Moreover}, most formulations operate on gridmaps, where each cell represents a vehicle-sized unit, which considerably constrains the operational space for larger AGVs such as forklifts. The resulting discretization often leads to inefficiencies and limits applicability in non-standardized industrial environments characterized by narrow corridors, irregular geometries, and obstacles including columns or machinery. \ABB{The underlying assumptions, although beneficial for theoretical analysis, computational tractability, and formal guarantees, ultimately limit the applicability of L-MAPF algorithms in real-world coordination scenarios involving heterogeneous fleets and complex continuous roadmaps.}

\ABB{Moving toward the MAPF domain, several solutions can be found in the literature that relax core assumptions, such as unit-time actions and grid-based representations.} \ABB{A notable study is the work of \citet{andreychuk2022multi}, which introduces Continuous Time CBS (CCBS), a continuous-time MAPF formulation that allows agents to follow continuous trajectories in both space and time, without relying on discrete grid representations or unit-timesteps. The authors claim optimality and completeness under their modeling assumptions. However, some recent discussions in the literature have pointed out potential theoretical limitations of CCBS \citep{li2025cbs}, and the approach remains computationally demanding, which may restrict its applicability under strict real-time constraints, particularly in lifelong operational settings.}

\ABB{To address scalability in continuous-time MAPF on roadmaps with variable-duration actions, \citet{9811344} propose a prioritized planning framework termed Prioritized Safe-Interval Path Planning with Continuous Time Conflicts (PSIPP/CTCs).
The method introduces CTCs, defined as pairs of roadmap elements associated with time intervals that lead to collisions. By precomputing and storing CTCs offline, PSIPP/CTCs enables efficient safe-interval reasoning and achieves impressive scalability. However, the formulation assumes homogeneous agents without explicitly modeling kinematic capabilities. Moreover, as a prioritized approach, PSIPP/CTCs does not provide any optimality guarantees and may suffer from priority-order sensitivity, which can become critical in dense, highly interactive, and complex traffic scenarios.} \ABB{Similarly, \citet{zhou2025loosely} propose a loosely synchronized rule-based MAPF framework designed to handle asynchronous actions of variable duration. Although the method scales to large agent populations under runtime constraints, it explicitly relies on rule-based coordination and trades optimality for scalability, providing no optimality guarantees. Moreover, the formulation models agents as points occupying graph vertices rather than physical entities with geometric footprints.} 

\ABB{A further relevant contribution is provided by \citet{wen2022cl}, which extends classical grid-based MAPF to continuous workspaces for car-like robots with kinodynamic constraints. Their Car-Like CBS (CL-CBS) combines a high-level body conflict tree with Spatiotemporal Hybrid-State A* (SHA*) as the low-level planner. Although the method has been demonstrated in simulation and validated on physical robots, its computational complexity remains significant for real-time operation in moderately dense environments. Moreover, the limited set of motion primitives may restrict flexibility, and the absence of collision checking during analytical expansions phase of SHA* may raise safety concerns in precise industrial maneuvers, such as docking or operation in tight spaces near the goal.} \ABB{Recently, to explicitly account for kinodynamic constraints and smooth speed profiles on curve-based roadmaps, \citet{yan2024multi} introduce a three-level MAPF-based motion planner that integrates Priority Based Search (PBS) \citep{ma2019searching}, Safe Interval Path Planning (SIPP) \citep{phillips2011sipp}, and a Bézier-curve-based optimization layer. The method demonstrates improved solution quality and scalability compared to baselines in a traffic intersection environment. However, it relies on prioritized search and therefore does not provide any global optimality guarantees. Moreover, the framework models homogeneous agents and has only been tested in a simplified intersection environment.}

\ABB{Overall, recent MAPF research has achieved significant progress in relaxing modeling assumptions, including grid-based abstractions, unit-time actions, and simplified agent constraints and geometries. 
However, existing formulations typically do not jointly address all these aspects within a unified framework and therefore do not provide coordination architectures that simultaneously support heterogeneous agent geometries and capabilities, non-grid industrial layouts, variable-duration actions, execution uncertainty, optimality guarantees, and lifelong task allocation under real-time constraints.
Consequently, a gap remains between advances in MAPF theory and integrated traffic management architectures capable of reliable deployment in large-scale industrial AGV systems.}

\subsection{Contributions}

As illustrated in Section \ref{relwork}, the literature lacks solutions capable of coordinating \ABB{in real-time} large and heterogeneous AGVs in non-standardized, high-traffic environments, e.g., the scenario depicted in Fig. \ref{fig:1.1}, while ensuring \AB{locally} optimal performance. On the one hand, rule-based methods are capable of effectively addressing problems of up to medium scale. However, their reliance on static and unyielding rules prevents them from ensuring general applicability or providing any formal performance guarantees. This limitation becomes particularly critical in traditional AGV traffic management systems, where coordination efficiency directly affects the overall productivity of plants and warehouses. Additionally, the majority of works either rely on oversimplifications or lack real-world implementation, making them unsuitable for use in operational facilities. 

On the other hand, search-based methods can systematically coordinate fleets of AGVs and provide guarantees on solution quality through L-MAPF algorithms. However, applicability remains limited in real-world environments, as most methods rely on restrictive assumptions such as homogeneous agents and unit-time action durations. The resulting simplifying assumptions, together with grid-based environment representations, make theoretical analysis feasible but do not capture the complexity of industrial scenarios involving large, heterogeneous vehicles operating on continuous roadmaps within non-standardized environments.


\ab{Although recent MAPF research has progressively relaxed individual assumptions, e.g., through formulations that move beyond unit-time actions and grid-based representations \citep{andreychuk2022multi,9811344,zhou2025loosely} or kinodynamic extensions for car-like agents \citep{wen2022cl,yan2024multi}, 
most existing approaches focus on isolated variants of the problem and therefore do not account for the combined impact of multiple realistic industrial constraints on coordination complexity.}

Motivated by the above considerations, this paper extends our previous work \citep{bonetti2024agv} and introduces an integrated AGV traffic management architecture tailored to real industrial environments. \ab{The main contribution is formulated at the system level and addresses the preservation of real-time solvability and locally optimal coordination under simultaneous integration of heterogeneous agent geometries and kinematics, continuous roadmap representations, non-uniform traversal times, and rolling-horizon replanning in high-density traffic conditions. Combining all requirements introduces strong interdependencies that invalidate assumptions underlying classical MAPF and L-MAPF formulations and prevent direct applicability of existing approaches under realistic industrial constraints.}

The proposed system relies on an L-MAPF algorithm operating on roadmaps constructed using Non-Uniform Rational B-Splines (NURBS), ensuring geometric continuity and compatibility with non-holonomic constraints. A modified version of Bounded Horizon CBS is integrated within a rolling-horizon conflict resolution (RHCR) framework, \ab{where the planning horizon is extended at the individual agent level to achieve effective conflict resolution in corridors identified through a topological abstraction.}


Unlike our previous work \citep{bonetti2024agv}, which, to the best of the authors’ knowledge, is the only one in the related literature addressing AGV coordination with guaranteed \AB{locally} optimal performance in real industrial environments, \ab{the present study introduces additional components and extensions required for deployment in operational settings. The contribution consists in defining an architectural decomposition that supports operation under realistic industrial conditions, together with the algorithmic extensions required to ensure consistency across the different functional components of the system, including environment modeling, planning, execution, and conflict management. Each component addresses a specific limitation of existing approaches and contributes to overall system performance. The main contributions are summarized as follows:}


\noindent


\textit{i)} introducing an anytime extension of Bounded Horizon CBS that preserves real-time responsiveness while progressively improving coordination quality toward eventual optimality and completeness. \ab{The proposed strategy computes locally optimal bounded-horizon solutions under strict computational constraints and progressively extends the planning horizon as additional computation becomes available. Differently from conventional anytime MAPF approaches, which typically refine globally suboptimal solutions, the proposed method preserves local optimality at every iteration while converging toward the solution produced by standard CBS;}



\noindent
\textit{ii)} introducing a path allocator module that formalizes the execution layer of the traffic management system\ab{, enabling  interaction between discrete L-MAPF coordination plans and continuous, uncertain motion of industrial vehicles while ensuring safe and standard-compliant operation with real AGVs. The proposed module fills a gap in existing AGV traffic management literature, where execution mechanisms connecting coordination plans and physical vehicle behavior are generally not explicitly modeled;}

\noindent
\textit{iii)} presenting a novel deadlock detection and resolution strategy to enhance the system’s robustness, particularly in dynamic environments where tasks may change unexpectedly. This strategy is based on another of our works \citep{10.1007/978-3-031-76428-8_50} and further elaborated upon in the current manuscript. Our approach ensures real-time adaptability, allowing AGVs to efficiently navigate high-traffic areas and minimize operational interruptions;

\noindent
\textit{iv)} \ABB{adopting and integrating an environment model that supports heterogeneous AGV operation and better captures the geometric characteristics of spatial interactions, ensuring precise collision detection in confined areas without compromising computational efficiency;}

\noindent
\textit{v)} 
validating the approach in real-world industrial facilities and benchmarking it against a traditional rule-based traffic management system, a state-of-the-art industrial solution, and a comparative L-MAPF method. The results demonstrate superior performance in terms of throughput, optimality of the proposed coordination framework, and effectiveness of the deadlock detection and resolution strategy. Additionally, our solution consistently achieves real-time performance in high-traffic, cluttered environments and adapts well to medium-scale and standardized settings, offering flexibility across various logistics and production scenarios.

\section{Preliminaries}
\label{pre}

In this section, we outline \AB{the problem statement (Section~\ref{ps})}, the key concepts to model the environment (Section~\ref{def}), and assumptions that ensure the feasibility of the proposed approach (Section~\ref{ass}).

\color{black}
\subsection{\AB{Problem Statement}}
\label{ps}



We consider a non-standardized industrial environment where a fleet of large and heterogeneous AGVs operates along a roadmap that captures the geometric and operational constraints of the facility, including narrow bidirectional corridors, intersections, and irregular layouts.
The vehicles differ in size, kinematic behavior, and motion capabilities, and are required to perform a sequence of pick-up and drop-off operations within a shared workspace.

The objective of the traffic management system is to generate coordinated motion plans that enable all vehicles to navigate safely and efficiently while completing their assigned tasks.
The problem integrates path planning and fleet coordination into a unified framework that ensures collision-free and dynamically feasible motion, while simultaneously detecting and resolving deadlocks and minimizing overall traversal cost in terms of throughput and travel time.
Safe operation must be guaranteed by avoiding both spatial and temporal conflicts, maintaining continuous traffic flow even under high-density traffic conditions.

Industrial environments are inherently dynamic, as tasks may be reassigned, temporary obstacles can appear, and operational priorities may change during execution. The coordination framework must therefore operate in real-time, continuously adapting vehicle trajectories to the evolving system state, variations in task assignment or workspace configuration, and situations that may result in deadlocks such as blocked corridors or mutual agent waiting.

\color{black}

\subsection{Definitions}
\label{def}

\AB{To better understand the different steps of the proposed approach,  we introduce the following definitions, grouped into environment-related~\citep{bondy2008graph} and AGV-related concepts.}

\AB{We first introduce the definitions related to the representation of the environment, which formalize the roadmap structure and its topological properties.}

\begin{definition}[Directed Graph]
A directed graph \(\mathcal{G}\) is characterized by a set \(\mathcal{V}(\mathcal{G})\) of vertices or nodes and a set  \(\mathcal{E}(\mathcal{G}) \subseteq \mathcal{V}(\mathcal{G}) \times \mathcal{V}(\mathcal{G})\) of edges.
\end{definition}

\begin{definition}[Subgraph]
A subgraph \(\mathcal{K}\) of \(\mathcal{G}\) is a graph whose vertex set \(\mathcal{V}(\mathcal{K})\) and edge set \(\mathcal{E}(\mathcal{K})\) are subsets of \(\mathcal{V}(\mathcal{G})\) and \(\mathcal{E}(\mathcal{G})\), respectively.
\end{definition}

\begin{definition}[Path]
A path of length \(L\) is defined on \(\mathcal{G}\)
as an ordered sequence of vertices $\lbrace v_1, v_2, \ldots, v_L \rbrace \subseteq \mathcal{V}(\mathcal{G})$ such that an edge exists in \(\mathcal{E}(\mathcal{G})\)  from vertex \(v_i\)
to vertex \(v_{i+1}\), for $i \in \lbrace 1,  \ldots, L-1 \rbrace$.
\end{definition}

\begin{definition}[Graph Connectivity]
A graph \(\mathcal{G}\) is said to be connected if a path exists between each pair of vertex in $\mathcal{V}(\mathcal{G})$. Given any pair of vertices \(v^{m}\), \(v^{n}\), a directed graph is considered strongly connected if a directed path can be defined from \(v^{m}\) to \(v^{n}\) and vice versa.
\end{definition}

\begin{definition}[Minimum Cost Path]
Given a graph \(\mathcal{G}\), the minimum cost path between vertex \(v_{i}\) and vertex \(v_{i+1}\) is determined by the edge weights \(\omega_{i,i+1}\).
\end{definition}

\begin{definition}[Location]
A location $p$ corresponds to a specific pose a vehicle can assume within the environment, expressed as \(q = [x, y, \theta]^ \top \in SE(2)\), where $x, y$ represent the linear positions along the x--y axes and $\theta$ the orientation.
\end{definition}

\begin{figure}
     \centering
\captionsetup{justification=justified}
     \begin{subfigure}[b]{0.45\linewidth}
         \centering
         \includegraphics[height= 2.5 cm]{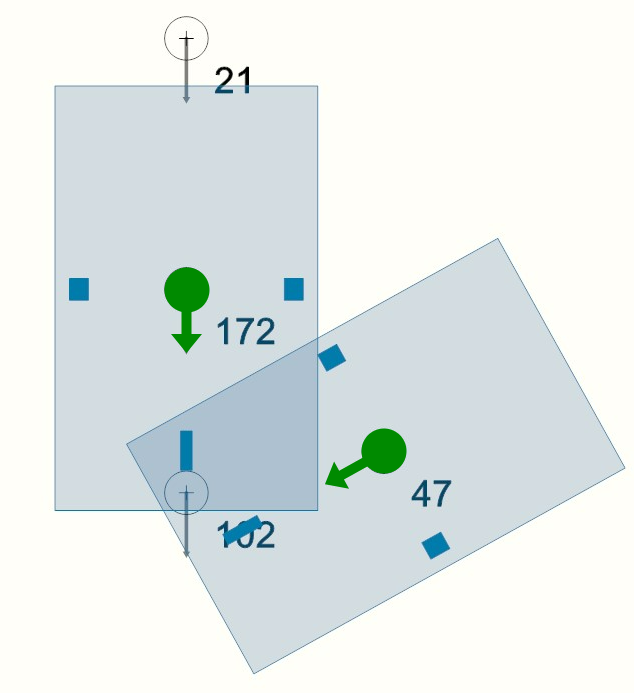}
         \caption{Location-Location collision}
         \label{fig:collision l-l}
     \end{subfigure}
     \hfill
     \begin{subfigure}[b]{0.45\linewidth}
         \centering
         \includegraphics[height= 2.5 cm]{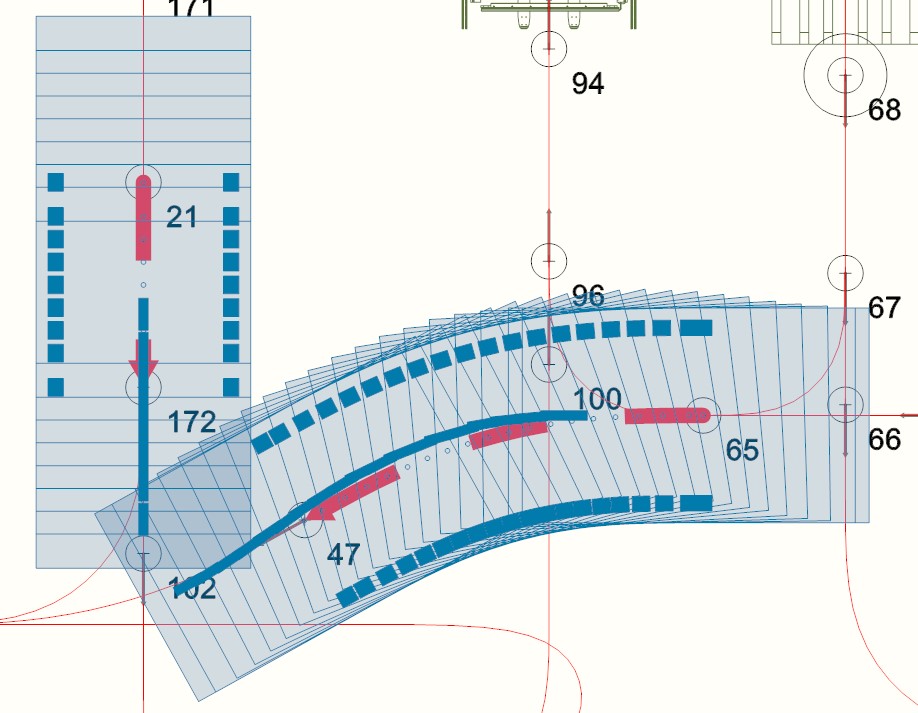}
         \caption{Segment-Segment collision}
         \label{fig:collision s-s}
     \end{subfigure}
        \hfill
          \begin{subfigure}[b]{0.45\linewidth}
         \centering
         \includegraphics[width=0.9\textwidth]{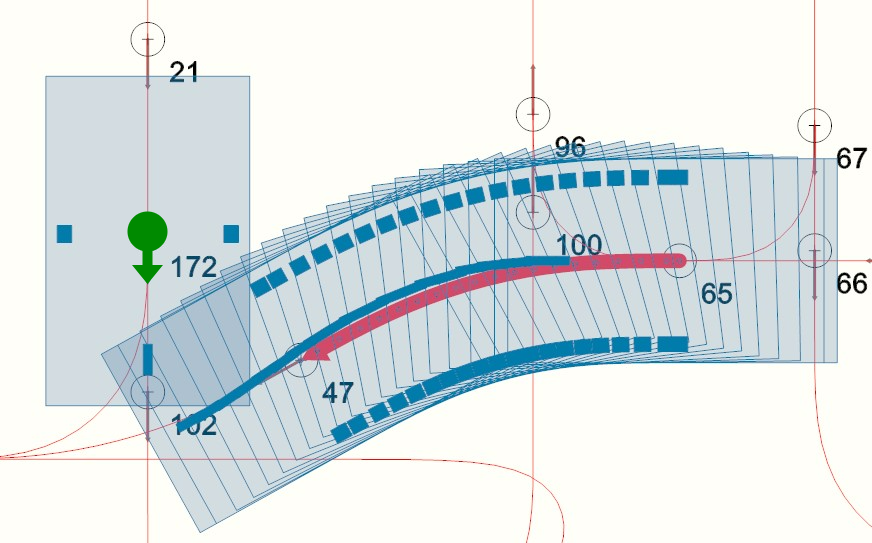}
         \caption{Segment-Location collision}
         \label{fig:collision s-l}
     \end{subfigure}

    \caption{The figure depicts the three collision cases on the roadmap. Locations are marked by green circles, with their centers corresponding to the \( x, y \) coordinates. Each location features a green arrow indicating the robot's orientation \( \theta \). Segments are depicted as red arrow curves. Light blue rectangles represent the vehicle's footprints, highlighting their spatial occupancy at both the locations and along the segments.}
    \label{fig:collision}
\end{figure}

\begin{definition}[Segment]
A segment $r$ is a continuous curve connecting two locations. 
\end{definition}

\begin{definition}[Roadmap]
A roadmap consists of a set of $N^V$ locations $p_u$, with $u \in \{1,  \dots, N^V\}$,  and a set of $N^E$ segments $r_h$ connecting them, with $h \in \{1,  \dots, N^E\}$.
\end{definition}

\AB{We then provide the definitions associated with the AGVs themselves, characterizing their geometric footprint and the interactions that may give rise to collisions.}

\AB{
\begin{definition}[Footprint]
The footprint $\phi$ of an AGV is a convex polygon representing the physical
boundary of the vehicle in the workspace. For any pose $q = (x,y,\theta)$,
the projection of the footprint is obtained by placing the polygon at $(x,y)$ and rotating it
by $\theta$, which defines the region of space occupied by the vehicle at that pose.
\end{definition}
}

\begin{definition}[Collision] 
\label{collision} 
Two elements of the roadmap are considered to be in collision if they cannot be occupied by two AGVs simultaneously due to overlap of their footprints, either at locations or along segments. The following cases can be identified (see Fig. \ref{fig:collision}):
\begin{itemize}  
    \item Location-Location: Two locations are in collision if the footprints of two AGVs, occupying them simultaneously, overlap;  
    \item Segment-Segment: Two segments are in collision if the footprints of two AGVs simultaneously traversing them overlap at least at one point along their respective segments;
    \item Segment-Location: A segment and a location are in collision if the footprint of an AGV traversing the segment overlaps at least at one point along the segment with the footprint of another AGV occupying the location.
\end{itemize}    
\end{definition}

\subsection{Assumptions}
\label{ass}

The environment of the industrial plant must satisfy the following conditions:

\begin{assumption}[Predefined Roadmap] 
\label{preRoad}
The roadmap is given a priori.
\end{assumption}

\begin{assumption}[Battery Charger]
\label{BatteryCharger}
The number of battery chargers $N^{B}$ must be at least equal to the number of AGVs $N^{A}$, i.e., $N^{B} \geq N^{A}$.
Battery chargers are positioned in locations such that, when occupied by AGVs, they are not in collision with any segments and locations traversed by other AGVs. 
\end{assumption}

Regarding the AGVs, the following simplifying assumptions are taken into account:

\begin{assumption}[Movement and Waiting Constraints]
\label{MovWait}
When operating on the roadmap, the AGVs can only wait at locations \AB{or} must move along segments.
\end{assumption}

\begin{assumption}[Task Lists]
\label{TaskList}The Task Lists, i.e., the lists of tasks to be executed for each AGV, are provided by the Task Manager.
\end{assumption}

\begin{assumption}[Fixed Path] 
\label{FixedPath}The AGVs follow a fixed path, calculated during task assignment, to perform a specific task.
\end{assumption}

\section{Traffic Management Architecture}
\label{Traffic Management System Architecture}

The proposed strategy adopts a centralized structure to address the \AB{traffic management} of AGV fleets. \AB{The centralized design} is selected for its ability to provide a \AB{global} perspective, enabling the coordination system to maintain a complete and current understanding of the entire framework. This holistic overview facilitates the implementation of optimization algorithms for coordination and deadlock resolution, and ensures consistency in decision-making, which is particularly beneficial in environments characterized by high-traffic density or frequent task alterations \citep{9599484,9657192}.

\AB{Figure~\ref{fig:2} illustrates the overall workflow of the proposed strategy. A central control unit, referred to as the Traffic Manager (TM), 
periodically receives updated information on the state of each AGV, including \AB{localization data and status indicators}. The TM also acquires the Task Lists provided by the Task Manager, which specify the tasks currently assigned to each vehicle, such as picking up or dropping off goods, navigating toward designated locations, or reaching a charging station when the battery level is low.}

\begin{figure}[t]
\centering
\captionsetup{justification=justified}
\includegraphics[width=\linewidth]{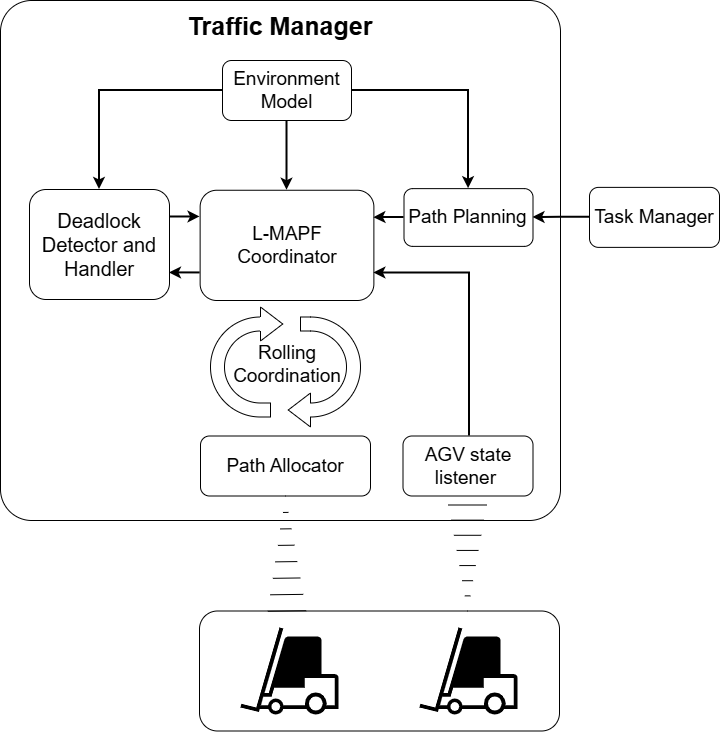}
\caption{\AB{Overview diagram of the proposed traffic management system.}}
\label{fig:2}
\vspace{-0.4cm}
\end{figure}

\AB{The TM first constructs and stores the environment model (see Section~\ref{EnvironmentModel}) and computes the fixed paths required for new task assignments (see Section~\ref{PathandTrajectoryDefinition}). 
The L-MAPF coordinator (see Section~\ref{lampfcoordinator}) then periodically computes instances consisting of collision-free trajectories that ensure coherent multi-agent coordination while adapting to task updates and uncertainties. Periodic replanning provides robustness and responsiveness that one-shot planning approaches cannot guarantee.}
\AB{Since AGVs operate in continuous-time and their actual motion may diverge from nominal timings, the trajectories produced by the coordinator cannot be executed directly. The path allocator (see Section~\ref{enabler}) therefore serves as the execution layer of the architecture: it enforces mutually exclusive allocations of roadmap elements and issues safe, real-time movement permissions consistent with the coordinated plan.}
\AB{Blocking situations may nonetheless occur in dense industrial layouts. Thus, the TM incorporates a dedicated deadlock detection and resolution module, which continuously monitors AGV interactions and replans the paths of the involved vehicles when required (see Section~\ref{DeadlockDetectorandHandler}).}

\section{Environment Model}
\label{EnvironmentModel}

We \ABB{adopt} a hierarchical structure to model the environment composed of two distinct layers: the roadmap layer and the topological layer. The roadmap layer is the representation of the roadmap and is employed for calculating the paths of the AGVs, defining their movements and velocities as they operate in the environment, and modeling their spatial occupancy in order to avoid collisions. 

\begin{remark}[Heterogeneous AGVs]
    \label{heterV}
    The fleet operating on the roadmap layer consists of \(N^{A}\) heterogeneous AGVs, each classified within a specific AGV class. Let \(\mathcal{C}\) be the set containing all the $N^{C}$ AGV classes that operate in the environment. A generic AGV class \(C_k \in \mathcal{C}\) with $k \in \lbrace 1,  \ldots, N^{C} \rbrace$ refers to the categorization of AGVs based on the kinematic constraints expressed by:
\AB{\begin{equation}
\label{vc}
 f_k(q,\dot{q})=0
\end{equation}
}and the convex polygonal footprint \(\phi_k\), which is applied to the pose that each AGV assumes in the environment.
\end{remark}

The topological layer divides the environment into distinct partitions defined as sectors, including temporary storage zones, pick-up and drop-off areas, open spaces, parking zones, and corridors. 
This segmentation enables the use of customized strategies for AGV coordination based on the specific type of sector being addressed. \AB{Particular attention in the present study is given to the corridors, as they play a crucial role in the innovative strategies integrated into the proposed traffic management system.}
 
A visualization of the environment model is depicted in Fig.~\ref{fig1}, where the industrial plant is partitioned into various rectangular sectors of the topological layer, each containing segments and locations of the roadmap layer.

\begin{figure}[t]
\centering
\captionsetup{justification=justified}
\includegraphics[width=\linewidth]{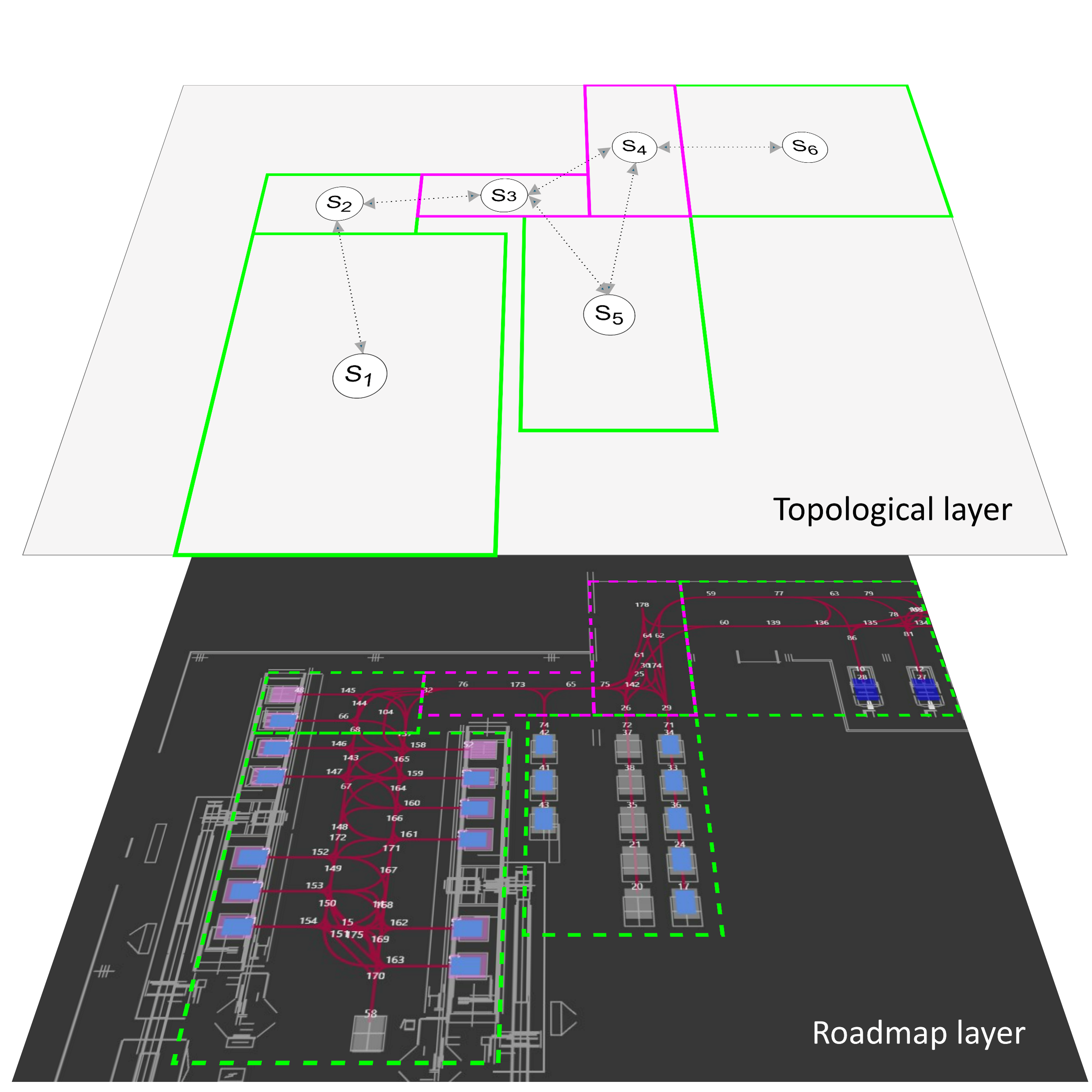}
\caption{Environment model with roadmap layer and topological layer. The former is composed of segments depicted by red lines, connecting white numbered locations. The latter is characterized by green and magenta \AB{
sectors, each labeled as \(S_i\) with $i \in \{1, \dots, 6\}$}. Specifically, the magenta \AB{sectors} delineate corridors within the environment.}
\vspace{-0.3cm}
\label{fig1}
\end{figure}

\subsection{Roadmap Layer}

The predefined roadmap is modeled as a directed weighted graph. Specifically, the roadmap layer is a graph $\mathcal{G}^R$ composed of the vertex set \(\mathcal{V}^R(\mathcal{G}^R)\) representing the set of specific locations in the environment, and of the edge set \(\mathcal{E}^R(\mathcal{G}^R)\), which represents the set of connections between neighboring connected locations, i.e., the roadmap segments.
\begin{remark}[Heterogeneous Segments]
    \label{heterR}
    Roadmap segments are heterogeneous in both shape and length.
\end{remark}
Let \(p^m\), \(p^n\) be two locations in the environment, represented by vertices \(v^m,v^n \in \mathcal{V}^R(\mathcal{G}^R)\). Then, the edge \(e^{m,n}=(v^m,v^n)\) exists in \(\mathcal{E}^R(\mathcal{G}^R)\),
if a segment exists in the roadmap that directly connects \(p^m\) and \(p^n\). For each edge in \(\mathcal{E}^R(\mathcal{G}^R)\), an edge weight \(\omega^{m,n}\) is set to the average time required to traverse the road segment from \(p^m\) to \(p^n\). 
Hence, the roadmap layer represents $N^{V}$ locations $p_{u}$, with $u \in \lbrace 1, \ldots, N^{V} \rbrace$, and $N^{E}$ roadmap segments $r_{h}$ with $h \in \lbrace 1,  \ldots, N^{E} \rbrace$. 

Each vertex \(v^m\) and edge \(e^{m,n}\) of the roadmap layer is designated to a specific AGV class, allowing only vehicles belonging to that class to navigate through them. In particular, each AGV class \(C_k \in \mathcal{C}\) with $k \in \lbrace 1,  \ldots, N^{C} \rbrace$ is associated with a strongly connected subgraph \(\mathcal{G}^{C_k}\) of $\mathcal{G}^R$, with the vertex set \(\mathcal{V}(\mathcal{G}^{C_k}) \subseteq \mathcal{V}(\mathcal{G}^R)\) and the edge set \(\mathcal{E}(\mathcal{G}^{C_k}) \subseteq \mathcal{E}(\mathcal{G}^R)\) representing the vertices and the edges assigned to \(C_k\). Consequently, the graph $\mathcal{G}^R$ is divided in $N^{C}$ disjoint subgraphs \(\mathcal{G}^{C_k}\).

The aforementioned assignment enables edges and vertices to be modeled with kinematic constraints and spatial occupancy specific to each AGV class. Hence, the operability of heterogeneous vehicles in the roadmap layer can be ensured by taking into account the kinematic constraints in (\ref{vc}) and the footprint \(\phi_k\) of each AGV class \(C_k\).

To guarantee compatibility with the AGVs' kinematics and to minimize mechanical wear, the segments represented by each edge set $\mathcal{E}(\mathcal{G}^{C_k})$ must possess adequate smoothness and maintain $G^1$ continuity \citep{peters2002geometric}. 
\AB{The latter ensures that adjacent segments share a common tangent direction at their junctions, preventing abrupt changes in curvature or vehicle orientation. Each segment is defined as a NURBS curve whose $x$ and $y$ coordinates satisfy the geometric and curvature constraints in (\ref{vc}) \citep{app132413210}, with tangent vectors at endpoints matching the vehicle orientation $\theta$ at the corresponding locations. As a result, transitions between consecutive segments remain smooth in both position and orientation, enabling kinematically feasible motion across the roadmap, with $\theta$ varying continuously along each curve according to the tangent direction of the parametric functions.}


\begin{remark} [Rotation Segments]
Under Assumption \ref{preRoad}, and assuming that \(C_k\) classifies vehicles capable of rotating in place, \(\mathcal{E}(\mathcal{G}^{C_k})\) can include edges representing ``rotation segments”. Such segments model stationary rotation maneuvers where the vehicle reorients itself, optimizing its navigation efficiency in constrained spaces. 
\AB{Mathematically, a rotation segment can be defined as a function where the orientation \( \theta \) varies continuously between an initial angle \( \theta^{I} \) and a final angle \( \theta^{F} \), while the linear positions $x,y$ remain constant, such that:}
\AB{
\begin{equation}
    \theta(s) = \theta^{I} + s(\theta^{F} - \theta^{I}), \quad s \in [0,1]
\end{equation}
}where $s$ is a normalized scalar parameter ranging from 0 to 1, with $s = 0$ corresponding to the initial location $p^I$, and $s = 1$, corresponding to the final location $p^F$.
\end{remark}

\begin{algorithm} [t!]
\caption{Collision Sets Calculator}
\label{alg12}
\KwIn{ $\mathcal{G}^R$, $\mathcal{C}$}
\KwOut{$\{\mathcal{Y}^m \, | \, g^m \in \mathcal{V}^R(\mathcal{G}^R) \cup \mathcal{E}^R(\mathcal{G}^R)\}$}

\ForEach{$g^m \in \mathcal{V}^R(\mathcal{G}^R) \cup \mathcal{E}^R(\mathcal{G}^R)$ } 
{
$\mathcal{Y}^m \gets \{\}$ \label{csinit}\\

\ForEach{$g^n \in \mathcal{V}^R(\mathcal{G}^R) \cup \mathcal{E}^R(\mathcal{G}^R)$}
{

$\phi^m \gets GetFootprint(g^m,\mathcal{C})$ \label{ab}\\
$\phi^n \gets GetFootprint(g^n,\mathcal{C})$ \label{foot} \\
\If{\textnormal{\textbf{\AB{Collision Detection}}}$(g^m, g^n, \phi^m, \phi^n )$ \label{yz}} 
{ $\mathcal{Y}^m \gets \mathcal{Y}^m \cup \{g^n\}$ \label{inser} \\

}
}
\label{assignment} 
}
$\textbf{return }\{\mathcal{Y}^m \, | \, g^m \in \mathcal{V}^R(\mathcal{G}^R) \cup \mathcal{E}^R(\mathcal{G}^R)\}$ 
\end{algorithm}

In addition, to ensure operational safety of multiple AGVs on the roadmap layer, it is necessary to account for potential collisions between heterogeneous vehicles as they navigate through shared spaces. 

\AB{Conventional cell-based collision models \citep{5547976} become inefficient in the narrow and irregular industrial layouts considered in this work. Achieving sufficient spatial resolution would require a large number of occupancy cells per roadmap element, causing collision checks to scale as $O(N^{OC})$ set-overlap operations, where $N^{OC}$ is the number of occupancy cells associated with the roadmap element being tested. Online bounding-box intersection tests \citep{zhou2000collision} offer higher geometric precision but still incur non-negligible computational cost when performed repeatedly at runtime, especially for large heterogeneous AGVs moving along curved NURBS segments in confined spaces. To obtain both geometric accuracy and real-time performance, we adopt \ABB{a precomputed} collision-set representation\ABB{, conceptually related to the precomputation strategy of CTCs in \citet{9811344}, but restricted to spatial interactions only.}}

Accordingly, each generic element \( g^m \in \mathcal{V}^R(\mathcal{G}^R) \cup \mathcal{E}^R(\mathcal{G}^R) \) of the roadmap layer is associated with its collision set $\mathcal{D}^{m}$. Specifically, $\mathcal{D}^{m}$ contains each generic element \( g^n \in \mathcal{V}^R(\mathcal{G}^R) \cup \mathcal{E}^R(\mathcal{G}^R) \), where $g^n \neq g^m$, such that \( g^m \) is in collision with \( g^n \), according to Definition~\ref{collision}. \AB{In practical terms, the collision set $\mathcal{D}^m$ of $g^m$ contains
all, and only, the elements that are in collision with $g^m$.} 

\AB{Collision sets are evaluated offline, producing compact precomputed sets that allow constant-time $O(1)$ lookup during real-time operation. This approach increases geometric accuracy in confined spaces while reducing computational overhead under high traffic conditions.} \AB{Specifically,} the collision sets are computed using Algorithm~\ref{alg12} - Collision Sets Calculator, which systematically detects collisions between all elements of \( \mathcal{G}^R \).

For each element \( g^m \in \mathcal{V}^R(\mathcal{G}^R) \cup \mathcal{E}^R(\mathcal{G}^R)\), a candidate set \( \mathcal{Y}^m \) is initialized as an empty set (Algorithm~\ref{alg12}, line~\ref{csinit}). Algorithm~\ref{alg12} then iterates over each generic element \( g^n \in \mathcal{V}^R(\mathcal{G}^R) \cup \mathcal{E}^R(\mathcal{G}^R)\), checking for collisions between \( g^m \) and \( g^n \). The process involves: (i) retrieving the footprints \(\phi^m\) and \(\phi^n\) associated with \( g^m \) and \( g^n \) from the set of AGV classes \( \mathcal{C}\)  (Algorithm~\ref{alg12}, lines~\ref{ab}-\ref{foot}); (ii) applying Algorithm~\ref{sat} - Collision Detection to determine if \( g^m \) and \( g^n \) are in collision, according to Definition~\ref{collision} 
(Algorithm~\ref{alg12}, line~\ref{yz}); (iii) adding \(g^n\) to \(\mathcal{Y}^m\) if a collision is detected (Algorithm~\ref{alg12}, line~\ref{inser}).

\begin{algorithm} [t!]
\caption{Collision Detection}
\label{sat}
\KwIn{$g^m, g^n$, $\phi^m$, $\phi^n$}
\KwOut{Boolean}

\eIf{$g^m \textnormal{ is vertex}$ \label{aa}}
{ $\mathcal{J}^m \gets \{p^m\} $ }
{ $\mathcal{J}^m \gets SamplePoses(r^m)$ }

\eIf{$g^n \textnormal{ is vertex}$}
{ $\mathcal{J}^n \gets \{p^n\}$ }
{ $\mathcal{J}^n \gets SamplePoses(r^n)$ \label{zz}}

\ForEach{{$q^{b} \in \mathcal{J}^m$} \label{startp}}
{
\ForEach{$q^{d} \in \mathcal{J}^n$}
{
$\Phi^{m} \gets TransformFootprint(\phi^m, q^{b})$\\
$\Phi^{n} \gets TransformFootprint(\phi^n, q^{d})$

\If{$\textnormal{SAT}(\Phi^{m},\Phi^{n}) $}
{ $\textbf{return } \texttt{}{\textnormal{true}}$ }
}
}
$\textbf{return } \texttt{}{\textnormal{false}}$ \label{end}
\end{algorithm}

Collision detection is performed by Algorithm~\ref{sat}, which takes as input \( g^m \) and \( g^n \), along with their respective footprints \( \phi^m \) and \( \phi^n \). Following the collision cases described in Definition~\ref{collision}, Algorithm~\ref{sat} first identifies the sets of poses \( \mathcal{J}^m \) and \( \mathcal{J}^n \) based on whether \( g^m \) and \( g^n \) represent edges or vertices in \( \mathcal{G}^R \). If \( g^m \) is an edge, \( \mathcal{J}^m \) consists of sampled poses along the represented segment \( r^m \).  If \( g^m \) is a vertex, \( \mathcal{J}^m \) includes the single pose \( q^{m} \) at the location \( p^m\). The same process applies to \( g^n\).
For each couple \( (q^{b}, q^{d}) \), with \(q^{b} \in \mathcal{J}^m \) and \(q^{d} \in \mathcal{J}^n \), Algorithm~\ref{sat} checks for overlaps between the projections $\Phi^{m}, \Phi^{n}$ of \( \phi^m \) and \( \phi^n \) using the Separating Axis Theorem (SAT) \citep{10.5555/1121584}. SAT states that two convex shapes do not collide if there is an axis along which their projections do not overlap. If an overlap is detected for at least one couple \( (q^{b}, q^{d}) \), the algorithm considers \( g^m \) and \( g^n \) to be in spatial collision, and it returns \texttt{true}. Conversely, if no overlap is found, the algorithm returns \texttt{false} (Algorithm~\ref{sat}, lines~\ref{startp}-\ref{end}).

In the following theorem, we state and prove that the output of Algorithm \ref{alg12} allows the computation of the collision set \( \mathcal{D}^m \hspace{0.8mm} \forall g^m \in \mathcal{V}^R(\mathcal{G}^R)\cup \mathcal{E}^R(\mathcal{G}^R)\).

\begin{theorem}
Let Assumptions \ref{preRoad} and \ref{MovWait} hold. Let \( g^m \in \mathcal{V}^R(\mathcal{G}^R) \cup \mathcal{E}^R(\mathcal{G}^R) \) be a generic element of the roadmap layer, and let \( \mathcal{D}^m\) be its collision set containing each generic element \( g^n \in \mathcal{V}^R(\mathcal{G}^R) \cup \mathcal{E}^R(\mathcal{G}^R) \), such that \( g^m \) is in collision with \( g^n \), according to Definition~\ref{collision}.
By applying Algorithm~\ref{alg12}, we obtain the set \(\mathcal{Y}^m= \mathcal{D}^m \hspace{0.8mm} \forall g^m \in \mathcal{V}^R(\mathcal{G}^R) \cup \mathcal{E}^R(\mathcal{G}^R) \).
\end{theorem}

\begin{proof}
We prove the theorem by contradiction.
Assume that Algorithm~\ref{alg12} fails to correctly compute 
$\mathcal{Y}^m$ for at least one \( g^m \in \mathcal{V}^R(\mathcal{G}^R) \cup \mathcal{E}^R(\mathcal{G}^R) \), i.e., \( \mathcal{Y}^m \neq \mathcal{D}^m \). Two situations can arise:
\begin{itemize}
    \item False Negative: Suppose \( g^n \notin \mathcal{Y}^m \) even though \( g^m \) and \( g^n \) are in collision, according to Definition~\ref{collision}. This implies that Algorithm~\ref{sat} returns \texttt{false}, despite an actual overlap of the projections $\Phi^{m}, \Phi^{n}$ in at least a couple of poses $(q^{b},q^{d})$ with \(q^{b} \in \mathcal{J}^m \) and \(q^{d}\in \mathcal{J}^n \). However, this condition leads to a contradiction as Algorithm~\ref{sat} is designed to detect such overlaps whenever they arise. 
    \item False Positive: Suppose \( g^n \in \mathcal{Y}^m \) even though \( g^m \) and \( g^n \) are not in collision. This implies that Algorithm~\ref{sat} returns \texttt{true} indicating an overlap of projections $\Phi^{m}, \Phi^{n}$ in at least a couple of poses $(q^{b},q^{d})$ with \(q^{b} \in \mathcal{J}^m \) and \(q^{d} \in \mathcal{J}^n \), even though no such overlap exists. However, this condition leads to a contradiction because Algorithm~\ref{sat} guarantees the detection of only actual overlaps.
\end{itemize}

Thus, Algorithm~\ref{alg12} avoids both false negatives and false positives. Therefore, for every \( g^m \), \( g^n \in \mathcal{Y}^m \) if and only if \( g^n \in \mathcal{D}^m \), i.e., \( \mathcal{Y}^m = \mathcal{D}^m \).  

\AB{
In simple terms, according to Definition~\ref{collision}, $\mathcal{D}^m$ is the
true collision set of $g^m$, containing all elements that intersect the
footprint of the vehicle associated with $g^m$.
Conversely, the set $\mathcal{Y}^m$ is the one computed by Algorithm~\ref{alg12}.
Therefore, the proof shows that the algorithm neither adds nor misses elements, so
$\mathcal{Y}^m$ coincides with $\mathcal{D}^m$.
}

\end{proof}

\begin{remark} [\AB{Sampling} Accuracy]
Under Assumptions \ref{preRoad} and \ref{MovWait}, the \AB{sampling} process in Algorithm \ref{sat} must sample poses on segments finely enough to accurately capture the set \( \{\mathcal{D}^m \mid  g^m \in \mathcal{V}^R(\mathcal{G}^R) \cup \mathcal{E}^R(\mathcal{G}^R)\} \). 
\end{remark}

\subsection{Topological Layer}
\label{tl}

The upper level consists of a topological map, \AB{serving} as an abstract representation of the environment in which the AGVs operate. \AB{The map defines the spatial relationships among different areas of the facility} and the paths accessible to the AGVs. Hence, it provides a topological representation of the industrial plant as a set of $N^{S}$ sectors, e.g., corridors.

Specifically, the set of interconnected sectors defines a directed and connected unweighted graph \(\mathcal{G}^{IS}\), referred to as the topological layer. Each vertex $S \in \mathcal{V}^{IS}(\mathcal{G}^{IS})$ represents a sector, which is associated with the subgraph \(\mathcal{G}^S\) of \(\mathcal{G}^R\), whose vertex set and edge set are given by \(\mathcal{V}(\mathcal{G}^S) \subseteq \mathcal{V}(\mathcal{G}^R)\) and \(\mathcal{E}(\mathcal{G}^S) \subseteq \mathcal{E}(\mathcal{G}^R)\), respectively. \(\mathcal{V}(\mathcal{G}^S)\) contains all the vertices of \(\mathcal{G}^R\) representing locations within the sector, while \(\mathcal{E}(\mathcal{G}^S)\) contains all the edges of \(\mathcal{G}^R\) that connect two vertices in \(\mathcal{V}(\mathcal{G}^S)\).
The edge set \(\mathcal{E}^{IS}(\mathcal{G}^{IS})\) represents the set of connections between neighboring connected sectors. Let \(S^b, S^d \in \mathcal{V}^{IS}(\mathcal{G}^{IS})\) be two sectors, and let \(v^m \in \mathcal{G}^{S^b}\) and \(v^n \in \mathcal{G}^{S^d}\) be two vertices contained in their respective subgraphs.
Then, an edge \((S^b,S^d)\) that connects \(S^b\) and \(S^d\) exists in \(\mathcal{E}^{IS}(\mathcal{G}^{IS})\), if an edge exists in \(\mathcal{G}^R\) that connects the vertex $v^m$ to the vertex $v^n$.

Hence, the topological layer is divided in $N^{S}$ sectors \(S_s\), with $s \in \lbrace1, \dots, N^{S}\rbrace$.

\color{black}

\section{Path Generation and Trajectory Definition}
\label{PathandTrajectoryDefinition}
\color{black}



\AB{This section introduces the elements required to define the paths and trajectories employed in the proposed traffic management architecture. Since both the spatial path and the associated trajectory depend on the start and goal vertices specified by each task, the discussion begins by describing how tasks are represented within the system. Although the TM does not perform task assignment, it must operate on the tasks supplied by the Task Manager, and their structure must therefore be clearly stated. Task assignment itself falls beyond the scope of this work; interested readers may refer to \citet{sabattini2015mission} and \citet{MAPFtaskassignment} for previous studies on the topic.}

\AB{Following the specification of the task structure, the section describes the computation of the fixed path derived from each task, in accordance with Assumption~\ref{FixedPath}, and subsequently introduces its trajectory counterpart, which incorporates the temporal evolution of the AGV along the predefined path.}

The system consists of \( N^{A} \) heterogeneous AGVs (see Remark~\ref{heterV}), each capable of performing tasks or remaining idle at the battery charger. According to Assumption~\ref{TaskList}, each AGV is assigned a Task List detailing the tasks to be completed, which can be categorized as:

\begin{itemize}
    \item \textit{Operational tasks}, which involve the execution of activities such as the pick-up and delivery of pallets within the operational environment.
    \item \textit{Return-to-battery-charger tasks}, which require the AGV to navigate to a designated charging station for battery recharging.
\end{itemize}

When a generic $a$-th AGV completes its current task or remains idle at the charger, the TM checks its Task List for pending tasks. If a new task is available, it is removed from the list and assigned to the AGV. 
\AB{Each task is associated with a starting vertex \( v^{s,a} \), corresponding to the AGV’s current location during task assignment, and a goal vertex \( v^{g,a} \), which must be reached to complete the current task.}
\AB{For \textit{operational tasks}, the Task List also includes the next task, with an associated next goal vertex \( v^{g'_,a} \), which remains stored but not yet execute. Note that if the task is a \textit{return-to-battery-charger}, then \( v^{g,a} = v^{g'_,a} \).}

\AB{The consideration of future tasks is crucial.}
\AB{Prior work on L-MAPF has shown that incorporating information on upcoming tasks improves coordination performance \citep{li2021lifelong, Grenouilleau_Hoeve_Hooker_2019}.
A similar principle applies in the proposed TM: by accounting for the next task, the TM avoids assuming that each $a$-th AGV remains indefinitely at its current goal vertex $v^{g,a}$, leading to coordination choices that eventually produce deadlocks.}


While this strategy is valid for \textit{return-to-battery-charger tasks} due to Assumption~\ref{BatteryCharger}, it is unsuitable for \textit{operational tasks} because of Assumption~\ref{FixedPath}. Such limitation can lead to two distinct deadlock scenarios:

\begin{itemize}
    \item  Let $a$-th and $b$-th be two AGVs assigned with a task, such that the goal vertex of the \( a \)-th AGV is \( v^{g,a} \), and the \( b \)-th AGV must traverse edges that conflict with \( v^{g,a} \) to complete its task. When the \( a \)-th AGV reaches its goal \( v^{g,a} \), the TM may incorrectly assume that it will remain stationary indefinitely. Consequently, the TM may coordinate the \( b \)-th AGV in a way that causes both AGVs to block each other during the execution of the \( a \)-th AGV’s subsequent task. 
    \item  If both the \( a \)-th and \( b \)-th AGVs are assigned the same goal vertex, \( v^{g,a}  = v^{g,b}  \), the first AGV to reach \( v^{g,a} \) may be blocked by the arrival of the second AGV. 
\end{itemize}

Hence, proactively coordinating AGVs even after they have reached their current goal vertex enables the TM to anticipate and resolve potential space-time collisions \AB{arising when the next task is assigned.} This approach ensures smoother task transitions and minimizes system downtime, effectively preventing the deadlock scenarios outlined above.

To enhance clarity and facilitate a better understanding of the TM coordination process, it is helpful to classify AGVs based on their current task assignment status.
\begin{remark} [Classification of AGVs as Active and Free]
\label{activefree}
    AGVs assigned to a task are referred to as active AGVs and belong to the set $\mathcal{A}$, whereas AGVs not assigned to a task are called free AGVs. When a task is assigned to a free AGV, such AGV is added to $\mathcal{A}$. Conversely, when an AGV returns to the battery charger, i.e., when it is no longer assigned any tasks, it is removed from $\mathcal{A}$.
\end{remark}

According to Assumption \ref{FixedPath}, each $a$-th active AGV is associated with the fixed path $\pi^a$. 
The path $\pi^a$ is determined during the task assignment process by extending the main path $\pi^{M,a}$ calculated from \(v^{s,a}\) to \(v^{g,a}\) of length $L^{M,a}$ with the path $\pi^{N,a}$ from \(v^{g,a}\) to \(v^{g'_,a}\) of length $L^{N,a}$. Both paths are calculated using a generic graph search algorithm \citep{EDELKAMP201247} on the subgraph \(\mathcal{G}^{C_k}\) of \(\mathcal{G}^R\). Specifically, $\pi^a$ is defined on \(\mathcal{G}^{C_k}\) as an ordered sequence of vertices: 

\begin{equation}
\begin{aligned}
    \pi^a =\lbrace v_{1}^{a}, v_{2}^{a}, \ldots ,v^a_{L^{M,a}}, \ldots  ,v_{L^{a}}^{a} \rbrace,\\ \quad v_{i}^{a} \in \mathcal{V}(\mathcal{G}^{C_k}), \forall i = 1,\dots,L^{a}
\end{aligned}
\end{equation}

\noindent where \(v_{1}^{a}=v^{s,a}\), \(v_{L^{M,a}}^{a}= v^{g,a}\), \(v_{L^{a}}^{a}=v^{g'_,a}\), and $L^{a}$ is the total path length. Specifically, $L^{a}$ is given by:
\[L^{a} = L^{M,a} + L^{N,a} - 1\]
with \(-1\) accounting for the shared vertex \(v^{g,a}\) between \(\pi^{M,a}\) and \(\pi^{N,a}\), included only once in the path \(\pi^{a}\).

\begin{remark}[Return-to-Battery-Charger Path]
    If the assigned task is a return-to-battery-charger task, i.e., $v^{g,a} = v^{g'_,a}$, then $\pi^{N,a} = \{v^{g,a}\}$, $ L^{N\ABB{,a}} = 1$, and $\pi^{a} = \pi^{M,a}$.
\end{remark}

Each pair of consecutive vertices \(v_{i}^{a}\) and \(v_{i+1}^{a}\) is connected by an edge $e_{i,i+1}^{a} \in\mathcal{E}(\mathcal{G}^{C_k})$, for $i \in \lbrace 1,  \ldots, L^{a}-1 \rbrace$. The path $\pi^{a}$ minimizes the total traversal cost:

\begin{equation}    \Omega^{a}=\sum_{i=1}^{L^{a}-1} \omega_{i,i+1}^{a}
\end{equation}
where $\omega_{i,i+1}^{a}$ is the cost of traversing the edge $e_{i,i+1}^{a}$.
Edge costs can include dynamic weight adjustments based on complex factors like traffic status~\citep{10132864} or traffic predictions~\citep{9171550} to optimize flow.
However, since this work focuses on coordination strategies for smooth fleet operation in high-density, non-standardized environments rather than path planning methodologies, edge costs are simplified to the average travel times of the represented segments, excluding traffic density distribution and other dynamic factors \citep{9943032}.

To prevent collisions and complete its task, according to Assumption~\ref{MovWait}, each $a$-th active AGV may either move along an edge $e_{i,i+1}^{a}$ connecting consecutive vertices $v_i^a, v_{i+1}^a \in \pi^a$, or wait at a vertex $v_i^a$ of its predetermined path $\pi^{a}$.
A move action executed along edge $e_{i,i+1}^{a}$ is represented by the tuple 
\begin{equation}
 \{v_{i}^{a}, v_{i+1}^{a}, \tau^{a}, D_{i,i+1}^{a}\}, 
\end{equation}
where $\tau^{a}$ is the starting timestep and $D_{i,i+1}^{a}$ is the traversal duration, expressed in discrete timesteps and equal to the edge cost $\omega_{i,i+1}^{a}$.
Conversely, a wait action is represented by 
\begin{equation}
 \{v_i^a, v_i^a, \tau^{a}, D_{i,i}^{a}\},  
\end{equation}
where $D_{i,i}^{a}$ corresponds to one timestep.
\AB{A trajectory $\gamma^{a}$ for the $a$-th AGV is therefore defined as an ordered sequence of $N^{\gamma,a}$ actions
\begin{equation}
  \{v_{i(k)}^{a}, v_{j(k)}^{a}, \tau_k^{a}, D_{i(k),j(k)}^{a}\},
\quad k \in \{1,\dots,N^{\gamma, a}\}, 
\end{equation}
with $j(k) = i(k)$ (wait) or $j(k)=i(k)+1$ (move).}

\color{black}

\section{L-MAPF Coordinator}
\label{lampfcoordinator}
\color{black}



\AB{The L-MAPF coordinator is responsible for generating space-time consistent trajectories for all active AGVs, ensuring collision-free motion while enforcing the fixed paths assigned to each vehicle and respecting the structural constraints of the industrial layout.
Unlike one-shot planning strategies, the L-MAPF coordinator operates in a rolling horizon fashion: instances, each consisting of a set of trajectories, are periodically recomputed so that the system can continuously adapt to new task assignments, evolving AGV states, and execution uncertainties.}

\AB{A rigorous formulation of the coordination problem solved at each instance requires a precise specification of all quantities involved in trajectory generation.
Relevant elements include the current target position associated with each AGV, the fixed path along which coordination is performed, the roadmap subgraphs used during planning, and the static and dynamic obstacles that represent environmental and traffic constraints.}

\AB{Section~\ref{coordinputs} introduces and organizes all components required as inputs to the L-MAPF coordination framework for computing an instance.
Section~\ref{coordproblem} then provides the formal formulation of the coordination problem solved at each instance.
Section~\ref{CBS} concludes the discussion by describing the Anytime Bounded Horizon CBS (ABH-CBS) algorithm, which computes the corresponding set of coordinated trajectories.}

\color{black}
\subsection{Coordination Problem Inputs}
\label{coordinputs}
\color{black}
\AB{To define an instance, several inputs must be specified, as summarized in Table~\ref{tabellainput} and detailed in the sequel.}

\AB{The coordination process operates within a base horizon of 
$\delta^{'}$ timesteps (Row~I), which may be extended in corridor sectors when required for safe conflict resolution. Outside corridors, the horizon can be incrementally expanded by $\delta^{''}$ timesteps (Row~II) until the timeout  $t^{\sigma}$ is reached (Row~III).}

\AB{Each $a$-th active AGV is associated with a target vertex $v_{z}^{a}$ with $z \in \{1,\dots, L^{M,a}\}$, provided by the path allocator (see Section~\ref{enabler}). The target vertex represents the discrete location where the newly computed trajectory must start, corresponding to the endpoint of the last edge assigned to the 
$a$-th AGV. The set of all target vertices, \AB{$\mathcal{V}^z =\{ v_{z}^{a} | a\in \mathcal{A}\}$} (Row~IV) is therefore provided as input. 
}

\AB{Let $v_{l}^{a} \in \pi^{a}$ denote the currently occupied vertex of the $a$-th active AGV, with $l \in \{1,\dots, L^{M,a}\}$ and $l \leq z$.  We define the allocated trajectory $u^{a}$ as the ordered sequence of move actions
\begin{equation}
   u^{a}
=
\left\{
  \{ v_{i}^{a}, v_{i+1}^{a}, \tau^{a}, D_{i,i+1}^{a} \}
  \;\middle|\;
  i \in \{l,\dots,z-1\}
\right\},
\end{equation}
which connects the currently occupied vertex $v_{l}^{a}$ to the target vertex $v_{z}^{a}$.  
The resulting sequence represents the portion of $\pi^{a}$ that the AGV must traverse without interruption.} \ABB{The set of all allocated trajectories is defined as $\mathcal{U}=\{ 
u^{a} | a\in \mathcal{A}\}$.}

\AB{The target timestep $\tau_{z}^{a}$, belonging to the set $\mathcal{T}^z$ (Row~V), denotes the discrete time at which the AGV is expected to reach $v_{z}^{a}$.  
The corresponding value is obtained from $u^{a}$ as:
\begin{equation}
    \tau_{z}^{a} = \tau_{z-1}^{a} + D_{z-1,z}^{a}.
\end{equation}}

Hence, at the target timestep \(\tau_{z}^{a}\), the L-MAPF coordinator calculates a trajectory \(\gamma^{a}\) for each \(a\)-th active AGV from the current target vertex \(v_{z}^{a}\) to the next goal vertex \(v_{g'}^{a}\), contained in the set \(\mathcal{V}^{g'}\) (Row~VI).

\begin{remark}
[Target Timestep with Empty Allocated Queue]
\label{EmptyAllocation}
If the allocated queue $w^{a}$ of the $a$-th active AGV is empty, i.e., $u^{a} = \{\}$, then the target timestep of the $a$-th active AGV is \( \tau_{z}^{a} = 0\).
\end{remark}

\begin{table}[t]
    \centering
    \caption{L-MAPF Coordinator Inputs.}
\begin{tabular}{@{}cccc@{}}

    \toprule
 &\textbf{Input} & \textbf{Description} \\
  \midrule

    I&$ \delta^{'}$ & Base Time Horizon \\
    \hline
    II&$\delta^{''}$ & Horizon Increment \\
    \hline
    III&$t^{\sigma}$ & Timeout \\ 
    \hline
    IV&\(\mathcal{V}^{z} =\{ v_{z}^{a} | a\in \mathcal{A}\}\) & Target Vertices \\
    \hline
    V&\(\mathcal{T}^{z} =\{ \tau_{z}^{a} | a\in \mathcal{A}\}\) & Target Timesteps \\
    \hline
    VI&\(\mathcal{V}^{g'} =\{ v^{g'_,a} | a\in \mathcal{A}\}\)& Next Goal Vertices \\
    \hline
    VII&\(\mathcal{K}^{\text{tot}} =\{ \mathcal{K}^{a} | a\in \mathcal{A}\}\) & Path Subgraphs \\
    \hline
    VIII&\(\mathcal{O}^s\) & Static Obstacles \\
    \hline
    IX&\(\mathcal{O}^d =\{ \mathcal{U}^{\setminus a} | a\in \mathcal{A}\}\) & Dynamic Obstacles \\
    \hline
    X&$\mathcal{S} =\{\xi^{a}| a\in \mathcal{A}\}$ & Extended Corridors \\
 \bottomrule
    
\end{tabular}
    \label{tabellainput}
    \vspace{-0.1cm}
\end{table}

According to Assumption~\ref{FixedPath}, each active $a$-th AGV must be coordinated along its fixed path $\pi^{a}$. 
\AB{Thus, coordination is performed on the subgraph} $\mathcal{K}^{a}$, defined as \( \mathcal{K}^{a} =(\mathcal{V}^{a}, \mathcal{E}^{a})\) of \(\mathcal{G}^{C_k} = (\mathcal{V}^{C_k}, \mathcal{E}^{C_k})\), where \(\mathcal{V}^{a} \subseteq \mathcal{V}^{C_k}\) and \(\mathcal{E}^{a} \subseteq \mathcal{E}^{C_k}\). \( \mathcal{K}^{a}\) contains the vertices $\lbrace v_{1}^{a}, v_{2}^{a}, \ldots, v_{L}^{a} \rbrace$ and edges $\lbrace e_{1,2}^{a}, e_{2,3}^{a}, \ldots, e_{L^{a}-1,L^{a}}^{a} \rbrace$. 
Hence, we define the set \(\mathcal{K}^{\text{tot}} =\{ \mathcal{K}^{a} | a\in \mathcal{A}\}\) (Row VII) containing all path subgraphs of active AGVs.

\begin{remark}[Goal Vertex Traversal Constraint]
    Each trajectory $\gamma^{a}$ is constrained to traverse the goal vertex \(v^{g,a}\). This is automatically ensured by computing $\gamma^{a}$ on the subgraph $\mathcal{K}^{a}$, which contains only vertices and edges associated with the fixed path $\pi^{a}$.
\end{remark}

To calculate collision-free trajectories, the L-MAPF coordinator must account for obstacles, which can be categorized as static or dynamic. On the one hand, static obstacles consist of unintended obstructions, e.g.,  shelving units and pallets, that block specific edges of the roadmap layer, potentially causing collisions with active AGVs. \AB{To prevent the computation of colliding trajectories}, each static obstacle \(o_k\), where $k \in \{1,\dots,N^{O}\}$ with $N^{O}$ the total number of obstructions, is associated with a subset of obstructed edges \(\mathcal{O}_k \subseteq \mathcal{E}^R\). Consequently, the static obstacles are represented by the set \(\mathcal{O}^s \subseteq \mathcal{E}^R\) (Row VIII), defined as follows:
\begin{equation}
    \mathcal{O}^s = \bigcup_{k=1}^{N^{O}} \mathcal{O}_k.
\end{equation}

\AB{On the other hand, dynamic obstacles represent the time-varying occupation of roadmap edges generated by the portions of the path allocated for other AGVs to traverse. Since the allocated trajectory $u^{a}$ represents the segment of $\pi^{a}$ that the $a$-th AGV must traverse without interruption, it must be treated as a dynamic obstacle when planning the trajectories of the others AGVs. Consequently, the L-MAPF coordinator takes as input the set \(\mathcal{O}^d =\{ \mathcal{U}^{\setminus a} | a\in \mathcal{A}\}\) (Row IX), which contains the dynamic obstacles for each $a$-th active AGV, defined as
\begin{equation}
        \mathcal{U}^{\setminus a} = \mathcal{U}\setminus u^{a}.
\end{equation}
\vspace{-0.5cm}
}

\begin{figure}
    \centering
\captionsetup{justification=justified}\includegraphics[width=.7\linewidth]{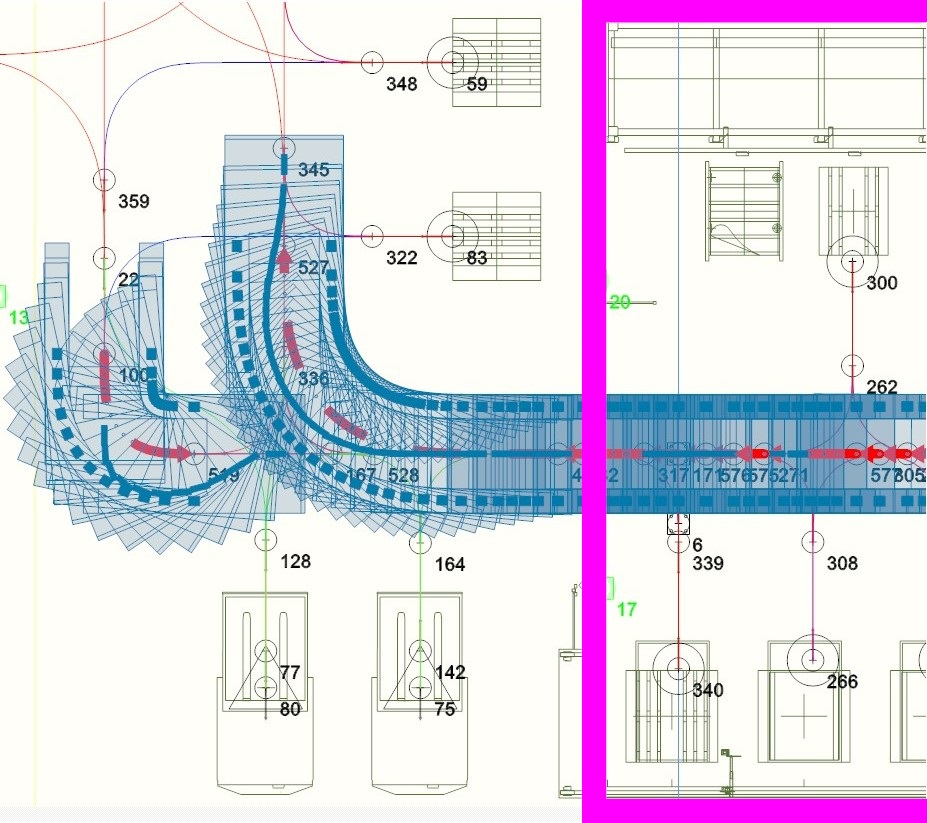}
    \caption{
    Example of a corridor extension with the corridor sector delimited by magenta borders, where one AGV intends to enter while another needs to exit.}
    \label{fig:enter-label}
\end{figure}

To ensure effective coordination, the trajectory $\gamma^a$ for each
$a$-th AGV must also account for potential challenges posed by shared resources, such as narrow bidirectional corridors. Each corridor, defined in the topological layer as a sector $S \in \mathcal{V}^{IS}(\mathcal{G}^{IS})$ (see Section \ref{tl}), requires complete conflict resolution to prevent deadlocks. While our system operates within a bounded horizon, this may not be adequate for long, narrow bidirectional corridors. By applying our Extended Horizon in Corridor (EHC) strategy proposed in \citet{bonetti2024agv}, we can effectively manage AGV coordination within any corridor $S$.



    


However, since the segments of the roadmap layer vary in length and shape (see Remark~\ref{heterR}), deadlocks could potentially occur outside the boundaries of \AB{corridors} (see Fig.~\ref{fig:enter-label}). Thus, following the Corridor eXtension (CX) procedure described in our previous work \citep{bonetti2024agv}, we define the set $\mathcal{S} =\{\xi^{a}| a\in \mathcal{A}\}$ (Row X), which includes the extended corridor $\xi^a \subseteq \ \pi^{a}$ for each $a$-th active AGV. In particular, \(\xi^{a}\) includes all the vertices of $\pi^{a}$, both within and adjacent to the corridor sectors where the EHC strategy is applied.

Each time a new path is calculated during task assignment or an active AGV returns to a battery charger (see Remark~\ref{activefree}), $\mathcal{S}$ is updated by comparing the paths of all active AGVs as outlined in Algorithm~\ref{extendedCorridors} - Extended Corridors Update.

\begin{algorithm} [t!]
\caption{Extended Corridors Update}
\label{extendedCorridors}
\footnotesize
\KwIn{$\mathcal{P}, \mathcal{G}^{IS}$}
\KwOut{$\mathcal{S}$}

$\mathcal{S} \gets \{\xi^{a} \gets \{\} \mid a \in \mathcal{A}\}$ \label{initializecorr} \\

\ForEach{$\pi^{a}, \pi^{b} \in \mathcal{P}, \pi^{a} \neq \pi^{b}$\label{startcorr}}
{
\ForEach{\AB{$S \in \mathcal{V}^{IS}( \mathcal{G}^{IS})$}}
{
    \If{\AB{$S \textnormal{ is Corridor}$}}
    {
        \AB{$\mathcal{G}^{S} \leftarrow GetCorridorSubgraph(S)$}\\
    
        $\xi^{a} \gets \xi^{a} \cup ExtendedCorridor(\pi^{a},\pi^{b},\mathcal{G}^S)$ \\
        
        $\xi^{b} \gets \xi^{b} \cup ExtendedCorridor(\pi^{b},\pi^{a}, \mathcal{G}^S)$ \label{endcorr}\\
    }
}
}

\Return{$\mathcal{S}$}
\end{algorithm}

Algorithm~\ref{extendedCorridors} takes as input the topological layer $\mathcal{G}^S$ and the set $\mathcal{P} = \{\pi^{a} \mid a \in \mathcal{A}\}$, which includes the path of each \( a \)-th active AGV, and returns the updated set $\mathcal{S}$.
The process begins by initializing $\mathcal{S}$ with an empty extended corridor for each $a$-th AGV (Algorithm~\ref{extendedCorridors}, line~\ref{initializecorr}). Subsequently, the algorithm compares the paths of each pair of AGVs, $\pi^{a}$ and $\pi^{b}$, applying the CX strategy \AB{to each corridor sector} to fill the respective extended corridors $\xi^{a}$ and $\xi^{b}$ (Algorithm~\ref{extendedCorridors}, lines~\ref{startcorr}-\ref{endcorr}). Through an exhaustive analysis of all possible path combinations, Algorithm~\ref{extendedCorridors} calculates the updated set $\mathcal{S}$, which is provided as input for the L-MAPF coordinator (Row~X).

\color{black}
\subsection{Coordination Problem Definition}
\label{coordproblem}

Having defined all required inputs in Section \ref{coordinputs}, we now formally state the coordination problem solved by the L-MAPF coordinator for each instance.



\begin{problem}[Coordination Problem]\label{def:coordination}

Given:
\begin{itemize}
    \item the set of active AGVs $\mathcal{A}$;
    \item for each $a \in \mathcal{A}$, the fixed path $\pi^{a}$;
    \item for each $a \in \mathcal{A}$, the target vertex $v_{z}^{a}$ and the target timestep $\tau_{z}^{a}$ with $z \in \{1,\dots, L^{M,a}\}$;
    \item for each $a \in \mathcal{A}$, the goal vertex $v^{g'_,a}$;
    \item the set of static obstacles $\mathcal{O}^{s}$;
    \item the set of dynamic obstacles $\mathcal{O}^{d}$;
\end{itemize}
the coordination problem consists in computing a set of collision-free trajectories
\[
\mathcal{H}^\star = \{ \gamma^{a,\star} \mid a \in \mathcal{A} \},
\]
such that:
\begin{itemize}
    \item each $\gamma^{a,\star}$ begins at $v_{z}^{a}$ at time $\tau_{z}^{a}$;
    \item each $\gamma^{a,\star}$ reaches $v^{g'_,a}$ in finite time, if a solution exists;
    \item each $\gamma^{a,\star}$ adheres to the corresponding fixed path $\pi^{a}$;
    \item all trajectories are mutually conflict-free;
    \item all trajectories avoid static and dynamic obstacles.
\end{itemize}
Among all feasible sets of trajectories, the 
L-MAPF coordinator tries to minimize a global performance metric 
$J$, which corresponds to the Sum of Costs (SoC). The chosen metric captures the overall travel effort of the fleet and promotes solutions that reduce unnecessary delays and congestion. In particular, SoC is computed as the sum of the travel times of the individual trajectories, defined as the completion time of their last action:
\begin{equation}
    J = \sum_{\gamma^{a,\star} \in \mathcal{H}^{\star}} \left( \tau_{k}^{a} + D_{i(k),\, j(k)}^{a} \right), 
\qquad k = N^{\gamma, a}.
\end{equation}
\end{problem}

\AB{Finding an optimal solution to Problem~\ref{def:coordination} has been shown to be NP-hard \citep{yu2013structure}. The problem, in fact, admits a reduction to a MAPF instance defined on a suitably constructed reduced graph, where the combinatorial growth generated by temporal interactions among multiple agents remains fully preserved. The resulting complexity renders optimality infeasible in realistic industrial settings, where AGVs operate in dense and constrained environments and must react in real-time to traffic variations, task updates, and execution uncertainties.}

\AB{The adopted strategy therefore focuses on computing locally optimal solutions, preserving coordination quality while keeping the computational burden tractable. To achieve such a balance, the L-MAPF coordinator employs the ABH-CBS algorithm, which builds upon the RHCR framework~\citep{li2021lifelong}. Every $\eta$ timesteps, the algorithm periodically solves bounded-horizon instances and progressively extends the planning window through an anytime mechanism.}

\AB{In operational terms, at each instance, ABH-CBS first solves the coordination problem restricted to the base time horizon $\delta'$. The algorithm then applies selective horizon extensions within the extended corridors $\mathcal{S}$, ensuring that potential conflicts in the considered regions are fully resolved. Once a feasible solution is identified, the planning window is enlarged by the horizon increment $\delta''$, and the bounded-horizon instance is refined, repeatedly expanding the horizon until the global timeout $t^{\sigma}$ is reached. As the horizon grows, the sequence of locally optimal solutions may eventually converge to an optimal and complete solution when the window becomes sufficiently large to cover the full arrival of all AGVs at their respective goals; under such conditions, ABH-CBS solution coincides with the one of standard CBS.}

\color{black}

\subsection{Anytime Bounded Horizon CBS}
\label{CBS}
The proposed ABH-CBS utilizes a hybrid approach consisting of two levels \citep{sharon2015conflict}: a low-level perspective that deals with individual AGVs and a high-level perspective that views the system holistically. On the one hand, the low-level is responsible for planning trajectories for each $a$-th active AGV that meet the constraints assigned by the high-level and that avoid static and dynamic obstacles. On the other hand, the high-level involves searching for space-time collision among AGV trajectories, defining constraints, and allocating them to AGVs for the low-level planning.


The low-level planner is based on the space-time A* algorithm \citep{silver2005cooperative}, modified to operate effectively on a roadmap consisting of heterogeneous segments and to align with the assumptions and obstacle representation defined in this work. Specifically, the low-level planner (as detailed in the Appendix A) calculates a trajectory $\gamma^a$ for the $a$-th active AGV from $v_{z}^{a}$ to $v^{g'_,a}$ on the subgraph $\mathcal{K}^a$, while adhering to constraints imposed by the high-level planner and avoiding both dynamic and static obstacles.

The high-level planner receives the inputs reported in Table~\ref{tabellainput} (see Section \ref{coordproblem}) and utilizes a best-first approach to explore a binary tree named Constraint Tree (CT) \citep{sharon2015conflict}, as shown in Fig. \ref{fig:CTtree}. In contrast to our previous work \citep{bonetti2024agv}, the planner has been modified to incorporate anytime functionality. This enhancement enables the computation of conflict-free trajectories within a bounded horizon, which are progressively refined toward optimality as computation time permits.

\begin{figure}
    \centering   \captionsetup{justification=justified}\includegraphics[width=0.4\linewidth]{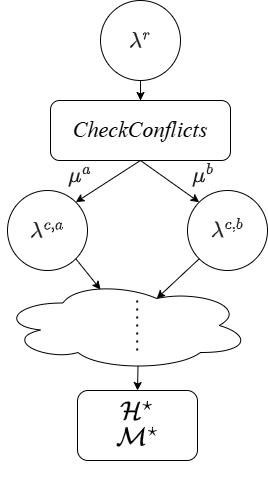}
    \caption{\AB{Constraint Tree expansion of the ABH-CBS algorithm, detailing the process leading to the computation of the outputs $\mathcal{H}^\star$ and $\mathcal{M}^\star$.}}
    \label{fig:CTtree}
\end{figure}

Each generic node within the CT comprises four key elements:
\begin{itemize}
\label{not}
    \item \AB{\textit{Node Constraint}}: this consists of a set \(\mathcal{M}\) of constraints \(\mu_k\), with $k\in\{1,2,\dots\}$, accumulated throughout the expansion of the CT. Each element \(\mu_k\) restricts the movement of the generic $a$-th active AGV between two neighbour vertices, i.e., edge $e_{i,i+1}^a$, or forbids wait action on vertex $v_{i}^a$, during a specific timestep interval in the low-level planner. In particular, \(\mu_k\) is expressed as a tuple \((a, v_{i}^a, v_{j}^a, [\tau^{I}, ..., \tau^{F}], b)_k\), with \(j=i \vee i+1\) for the \(a\)-th AGV. Here, $[\tau^I, \ldots, \tau^F]$ represents the constrained timestep interval, spanning from the initial timestep $\tau^I$ to the final timestep $\tau^F$, imposed by the $b$-th AGV.
    \item \AB{\textit{Node Solution}}: this includes a set of AGVs' timestep-minimal trajectories \AB{\(\mathcal{H}={\{\gamma^a | a \in \mathcal{A}}\}\)}, calculated by the low level planner and satisfying the constraints defined in \(\mathcal{M}\).
    \item \AB{\textit{Node Cost}}: this represents the total cost of the node, calculated as $\sum_{\gamma^{a} \in \mathcal{H}} |\gamma^a|$, \AB{where $|\gamma^a|$ denotes the cost of each trajectory in $\mathcal{H}$. Since $|\gamma^{a}|$ corresponds to the trajectory travel time, the node cost implements the SoC, ensuring consistency with the coordination metric $J$ in  Problem~\ref{def:coordination}.}
    \item \AB{\textit{Node Horizon}}: this consists of the set \(\mathcal{N}\), which includes the distinct time horizon value \(\delta^a\) for each \(a\)-th AGV. In particular, \(\delta^a\) is set as the current value of $\delta$ in the non-corridor zones of the environment. If the AGV is inside its extended corridor $\xi_a$ at the timestep \(\delta^a\) along the timestep-minimal trajectory \(\gamma^a \in \mathcal{H}\), then \(\delta^a\) is extended until the \(a\)-th AGV reaches the first vertex \(v^{\text{out}, a} \notin \xi_a\) using the EHC strategy \citep{bonetti2024agv}.
    \end{itemize}

\begin{algorithm} [t!]
\caption{High-Level Planner}
\label{algcbs}
\footnotesize
\KwIn{$\delta^{'}, \delta^{''}, t^{\sigma},$\(\mathcal{V}^z,\)  \(\mathcal{T}^z,\) 
\(\mathcal{V}^{g'},\)
\(\mathcal{K}^{\textnormal{tot}},\)  \(\mathcal{O}^s,\)
\(\mathcal{O}^d,\) \(\mathcal{S}\)
}
\KwOut{$\mathcal{H}^\star, \mathcal{M}^{\star}$}

$t^{\iota} \gets t$ \label{init1}\\
$\delta \gets \delta^{'}$\\
$\mathcal{H}^\star \gets \emptyset$\\
$\mathcal{M}^{\star} \gets \emptyset $\label{init2}\\

$\lambda^{r} \gets GenerateRoot(\mathcal{V}^z,\mathcal{T}^z, \mathcal{V}^{g'},\mathcal{O}^{s} , \mathcal{O}^{d},\mathcal{K}^{\textnormal{tot}},\mathcal{S},\delta)$\label{root1}\\
$\text{OPEN\_H} \gets \{\lambda^{r}\}$ \label{root2}\\

\While{$\textnormal{OPEN\_H} \neq \emptyset \ \mathbf{and} \ t - t^{\iota} < t^\sigma$}
{
$\lambda^{p} \gets \textnormal{lowest \textit{Cost} node from OPEN\_H}$\label{popout1}\\

$\textnormal{OPEN\_H} \gets \textnormal{OPEN\_H} \setminus \lambda^{p}\label{popout2}$

\ForEach{$\delta^a \in \mathcal{N}(\lambda^{p})$ \label{horizonup1}}
{
\If{$\delta^a < \delta$}
{ $\delta^a \gets \delta$\\
$\delta^a \gets EHC(\delta^a,\mathcal{S}, \mathcal{H}(\lambda^{p}))$ \label{horizonup2}}

}
\footnotesize
$[\mu^{a}, \mu^{b}] \gets CheckConflicts(\mathcal{H}(\lambda^{p}), \mathcal{N}(\lambda^{p}),\mathcal{K}^{\textnormal{tot}})$ \label{cc}\\ 

\eIf{$\mu^{a} \neq \varnothing \ \textbf{and}\ \mu^{b} \neq \varnothing$}
{
$\lambda^{c,a} \gets GenerateChild(\lambda^{p},\mathcal{V}^z,\mathcal{T}^z, \mathcal{V}^{g'},\mathcal{O}^{s} , \mathcal{O}^{d},\mathcal{K}^{\textnormal{tot}},\mathcal{S}, \delta, \mu^{a})$\\ \label{gc1}
$\lambda^{c,b} \gets GenerateChild(\lambda^{p},\mathcal{V}^z,\mathcal{T}^z, \mathcal{V}^{g'},\mathcal{O}^{s} , \mathcal{O}^{d},\mathcal{K}^{\textnormal{tot}},\mathcal{S}, \delta, \mu^{b})$\\ \label{gc2}
\If{$\lambda^{c,a} \textnormal{ is valid}$ \label{add1}}
{
$\textnormal{OPEN\_H} \gets \textnormal{OPEN\_H}  \cup  \{\lambda^{c,a}\}$
}
\If{$\lambda^{c,b} \textnormal{ is valid}$}
{
$\textnormal{OPEN\_H} \gets \textnormal{OPEN\_H} \cup  \{\lambda^{c,b}\} \label{add2}$
}
}
{
$\mathcal{H}^\star \gets \mathcal{H}(\lambda^{p})$ \label{sol1}\\
$\mathcal{M}^{\star} \gets \mathcal{M}(\lambda^{p})$\\
$\delta \gets \delta + \delta^{''}$  \label{sol2}\\
$\textnormal{OPEN\_H} \gets \textnormal{OPEN\_H} \cup \{\lambda^{p}\}\label{insss}$
}

}

\Return{ $\mathcal{H}^\star, \mathcal{M}^{\star}$}
\normalsize

\end{algorithm}

Algorithm~\ref{algcbs} - High-Level Planner begins with an initialization phase, which involves recording the current time $t$ in $t^\iota$, setting the time horizon $\delta$ with the base time horizon $\delta^{'}$, and initializing the algorithm outputs $\mathcal{H}^\star$ and $\mathcal{M}^{\star}$ as empty sets (Algorithm~\ref{algcbs}, lines~\ref{init1}-\ref{init2}). Subsequently, Algorithm~\ref{algcbs} computes the root node $\lambda^{r}$ using the \textit{GenerateRoot} function and inserts it in a list, in the sequel referred to as the OPEN\_H list (Algorithm~\ref{algcbs}, lines~\ref{root1}-\ref{root2}). $\lambda^{r}$ is characterized by an initial solution $\mathcal{H}$ planned by the low-level with an empty set of constraints $\mathcal{M} = \emptyset$. At this step, the root node \textit{Cost} and \textit{Horizon} fields are also computed.

As long as the OPEN\_H list is not empty and the current computation time $t - t^{\iota}$ does not exceed the calculation timeout $t^{\sigma}$, the node $\lambda^{p}$ with the lowest cost, referred to as the parent node, is selected and removed from the OPEN\_H list (Algorithm~\ref{algcbs}, lines~\ref{popout1}-\ref{popout2}). Then, Algorithm~\ref{algcbs} performs the horizon update phase on $\lambda^{p}$, which consists of checking whether each $\delta^a \in \mathcal{N}$ of $\lambda^{p}$ is less than the current time horizon $\delta$. If this condition holds, $\delta^a$ is set equal to $\delta$ and the EHC strategy is applied to $\gamma^a \in \mathcal{H}$ of the corresponding $a$-th active AGV (Algorithm~\ref{algcbs}, lines~\ref{horizonup1}-\ref{horizonup2}). 

Algorithm~\ref{algcbs} proceeds with conflict detection among all pairs of trajectories in $\mathcal{H}$ of $\lambda^{p}$ using the \textit{CheckConflicts} function (Algorithm~\ref{algcbs}, line~\ref{cc}). Let \(a,b \in \mathcal{A}\) be two active AGVs and let $\gamma^a$ and $\gamma^b$ be the respective trajectories included in $\mathcal{H}$. The function checks for space-time collisions between each pair of actions $\{v_i^a, v_j^a, \tau^a, D_{i,j}^a\} \in \gamma^a, \{v_m^b, v_n^b, \tau^b, D_{m,n}^b\} \in\gamma^b$ using Algorithm~\ref{checkspacetimecollisions} - Check Space-Time Collision (see Appendix A). The search is conducted within a bounded horizon defined by the minimum value between $\delta^a$ and $\delta^b$, as described in \citep{bonetti2024agv}.

On the one hand, if conflicts exist, \textit{CheckConflicts} selects the first identified conflict between the $a$-th and $b$-th AGVs and generates the constraints $[\mu^a, \mu^b]$ to resolve it in the child nodes. In particular, $[\mu^a, \mu^b]$ are defined as:
 \begin{align} 
 \begin{split} 
 [ (a, v_{i}^a, v_{j}^a, [\tau^{b}, ..., \tau^{b}+D_{m,n}^b], b), \\ 
 (b, v_{m}^b, v_{n}^b, [\tau^{a}, ..., \tau^{a}+D_{i,j}^a], a)].
 \end{split}
\end{align}
Specifically, the action from vertex $v_{i}^a$ to $v_{j}^a$ in $\mu^a$ is constrained for the time interval $[\tau^{b}, ..., \tau^{b}+D_{m,n}^b]$ in which $\{v_m^b, v_n^b, \tau^b, D_{m,n}^b\} \in\gamma^b$ conflicts with it and vice versa.
Then, Algorithm \ref{algcbs} generates the child nodes \(\lambda^{c,a}\) and \(\lambda^{c,b}\) using the \textit{GenerateChild} function (Algorithm~\ref{algcbs}, lines~\ref{gc1}-\ref{gc2}). Specifically, \textit{GenerateChild} incorporates the constraints $\mu^a$ and $\mu^b$ into the respective constraint sets of \(\lambda^{c,a}\) and \(\lambda^{c,b}\) and invokes the low-level planner to replan the trajectories $\gamma^a$ and $\gamma^b$ for the $a$-th and $b$-th AGVs involved in the selected conflict. In addition, the function updates the $Cost$ and $Horizon$ fields of the child nodes accordingly. Hence, if the new trajectories $\gamma^a,\gamma^b$ are successfully computed by the low-level planner, the child nodes $\lambda^{c,a}$ and $\lambda^{c,b}$ are deemed valid, and Algorithm \ref{algcbs} adds them to the OPEN\_H list (Algorithm~\ref{algcbs}, lines~\ref{add1}-\ref{add2}).

On the other hand, if there are no conflicts, $\mathcal{H}$ of $\lambda^{p}$ is stored in $\mathcal{H}^\star$, the constraints set $\mathcal{M}$ is stored in $\mathcal{M}^{\star}$, and the time horizon $\delta$ is incremented by $\delta^{''}$ (Algorithm~\ref{algcbs}, lines~\ref{sol1}-\ref{sol2}). Then, the parent node $\lambda^{p}$ is re-inserted in the OPEN\_H list (Algorithm~\ref{algcbs}, line~\ref{insss}).

The high-level ABH-CBS repeats this process until the current computation time exceeds the calculation timeout $t^{\sigma}$ or all possibilities are exhausted, i.e., the OPEN\_H list is empty. It then returns as output the last assignment of $\mathcal{H}^\star$ \AB{corresponding to the lowest-cost solution found within the explored planning horizon that minimize $J$, together with the associated constraint set $\mathcal{M}^\star$.}



\color{black}

\section{Path Allocator}
\color{black}
\label{enabler}

\AB{The path allocator module acts as the execution layer of the traffic management architecture. Although the L-MAPF coordinator periodically computes collision-free trajectories, each computation requires a non-negligible processing time and is performed at fixed recomputation intervals of $\eta$ timesteps. During each interval, AGVs operate in continuous-time, and their motion is subject to execution uncertainties that may introduce deviations from the predicted timing.}

\begin{remark}[Execution Uncertainties]
    {During execution, AGVs are exposed to both internal and external sources of uncertainty that affect their nominal traversal times. 
    Internal factors include variations in acceleration and deceleration profiles, differences in initial conditions (e.g., entering a segment from a full stop or while already in motion), mechanical wear, and sensor or control inaccuracies. 
    External factors arise from the operational environment and typically include vehicle alarms, safety laser scanner activations due to temporary obstacles, human interventions, and other transient disturbances. 
    These combined effects may force an AGV to slow down or stop unexpectedly, resulting in deviations from the predicted timing assumed by the coordinator.}
\end{remark}

\AB{As a result, two AGVs that are conflict-free in the computed space-time solution may still collide during execution if their actual traversal times deviate from the nominal ones. The path allocator eliminates this risk by regulating access to the roadmap and ensuring safe execution of the coordinated plan. Based on the trajectories $\mathcal{H}^\star$ provided by the L-MAPF coordinator as a reference, it assigns edges and determines the corresponding allocated trajectory and target vertex for each active AGV, enforcing mutual exclusion over all roadmap elements involved in potential conflicts so that their collision sets remain disjoint. The resulting target vertices and allocated trajectories are then supplied as inputs for the subsequent instance calculation for the L-MAPF coordinator. }


\begin{figure}[t]
\centering
\captionsetup{justification=justified}
\includegraphics[width=0.9\linewidth]{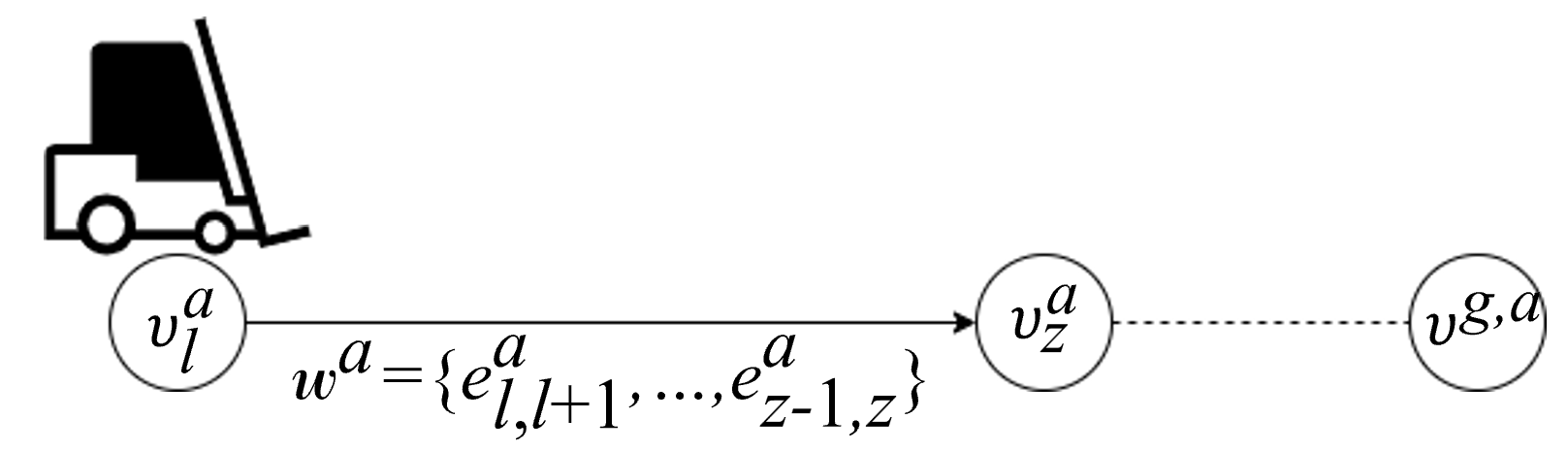}
\caption{Allocation process for the $a$-th AGV, illustrating the allocated queue $w^a$ from the currently occupied vertex $v_l^a$ to the target vertex $v_z^a$, depicted with a solid black arrow. The remaining portion of the path $\pi^a$ from $v_z^a$ to $v^{g,a}$ is represented by a dashed line.}
\label{fig:alloc}
\vspace{-0.3cm}
\end{figure}

As depicted in Fig.~\ref{fig:alloc}, the allocated edges for each $a$-th active AGV are contained in the allocated queue $w^{a}$, defined as an ordered sequence of edges $\{e_{l,l+1}^{a}, \dots, e_{z-1,z}^{a}\}$, with $l,z \in \{1,\dots, L^{M,a}\}$ and $l\leq z$. In particular, $w^{a}$ connects the currently occupied vertex $v_{l}^{a} \in \pi^{a}$ to the target vertex $v_{z}^{a} \in \pi^{a}$ (see Section~\ref{coordinputs}). From $w^{a}$, the allocated trajectory $u^{a}$ is determined. Hence, each $a$-th active AGV is associated with an allocated queue $w^{a}$, an allocated trajectory $u^{a}$, and a target vertex $v_{z}^{a}$.

\begin{remark}[Empty Allocated Queue and Trajectory]
The allocated queue $w^{a}$ and the trajectory $u^{a}$ of the $a$-th active AGV are empty if $l = z$, i.e., the AGV is occupying its current target vertex. \AB{In case the AGV is assigned a new task, i.e., $l=1$, $w^{a}$ and $u^{a}$ are initialized as empty.}
\end{remark}

The sets of allocated queues $\mathcal{W}=\{ 
w^{a} | a\in \mathcal{A}\}$, allocated trajectories $\mathcal{U}=\{ 
u^{a} | a\in \mathcal{A}\}$, and target vertices $\mathcal{V}^z=\{ v_{z}^{a} | a\in \mathcal{A}\}$ are updated each time a new instance is computed, as outlined in Algorithm~\ref{algAlloc} - Allocated Queues Updates, ensuring consistency and efficient edge allocation.


From the current L-MAPF coordinator instance, Algorithm~\ref{algAlloc} takes as inputs the sets $\mathcal{W}$, $\mathcal{U}$, $\mathcal{V}^z$, $\mathcal{H}^\star$, and a user-defined threshold parameter $\epsilon$.
The parameter $\epsilon$ specifies that edge allocation must be performed within
a time horizon of $\epsilon$ timesteps. Ignoring $\epsilon$ can extend the allocation unnecessarily, reducing coordination efficiency and limiting the L-MAPF coordinator's ability to adjust trajectories based on updated information within the next $\epsilon$ timesteps.

\begin{algorithm} [t!]
\caption{Allocated Queues Updates}
\label{algAlloc}
\KwIn{$\epsilon,$ $\mathcal{W}, $
$\mathcal{U}$,
$\mathcal{H}^\star,$
$\mathcal{V}^z$}
\KwOut{$\mathcal{W},
\mathcal{U}, 
\mathcal{V}^z $}

\ForEach{$\gamma^{a,\star} \in \mathcal{H}^\star$ } 
{
$k \gets 0$ \\
\While{$k < \textnormal{length}(\gamma^{a,\star}$)}
{
$\{v_{i}^{a},v_{j}^{a}, \tau^{a}, D_{i,j}^{a}\}
\gets \gamma^{a,\star}[k]$ \\

\If{$\tau^{a} > \epsilon$}
{\label{allhor}
    \textnormal{break}\\
}
\If{$v_{i}^{a} = v_{j}^{a}$}
{\label{checkwait}
    \textnormal{break}\\
}

\If{$v_{i}^{a} = v_{g}^{a}$}
{\label{checkgoal}
    \textnormal{break}\\
}

$e_{i,i+1}^{a} \gets GetEdge(v_{i}^{a},v_{i+1}^{a})$ \label{getedge}\\
$\mathcal{D}_{i,i+1}^{a} \gets GetCollisionSet(e_{i,i+1}^{a})$\label{getcoll}\\
$CollisionDetected \gets \textnormal{false}$ \label{colli} \\
\ForEach{$w^{b} \in \mathcal{W}\setminus w^{a}$}
{
\ForEach{$e_{m,m+1}^{b} \in w^{b}$}
{
\If{$e_{m,m+1}^{b} \in \mathcal{D}_{i,i+1}^{a}$}
{
$CollisionDetected \gets \textnormal{true}$\\
}
}
}
\If{$CollisionDetected =\textnormal{true}$}
{
$\textnormal{break}$ \label{colle}\\
}
$\textnormal{Append} \ e_{i,i+1}^{a} \ \textnormal{to} \ w^{a}$ \label{append}\\
$\textnormal{Append} \ \{v_{i}^{a}, v_{i+1}^{a}, \tau^{a}, D_{i,i+1}^{a}\} \ \textnormal{to} \ u^{a}$\\
$v_{z}^{a} \gets v_{i+1}^{a}$ \label{update}\\
$k \gets k+1$\\
}
}
$\mathbf{return} \ \mathcal{W},\mathcal{U},\mathcal{V}^z$
\end{algorithm}

Algorithm~\ref{algAlloc} iterates through each trajectory $\gamma^{a,\star} \in \mathcal{H}^\star$, processing all move and wait actions. For each action $\{v_{i}^{a}, v_{j}^{a}, \tau^{a}, D_{i,j}^{a}\}$, the algorithm performs the following checks:

\begin{itemize}
    \item If \(\tau^{a} > \epsilon\), the action occurs beyond the time horizon defined by \(\epsilon\) (Algorithm~\ref{algAlloc}, line~\ref{allhor}).
    \item If \(v_{i}^{a} = v_{j}^{a}\), a wait action occurs at vertex \(v_{i}^{a}\), preventing the \(a\)-th AGV from proceeding (Algorithm~\ref{algAlloc}, line~\ref{checkwait}).
    \item If \(v_{i}^{a} = v^{g,a}\), the path allocator has already allocated edges up to the task goal vertex \(v^{g,a}\) (Algorithm~\ref{algAlloc}, line~\ref{checkgoal}).
\end{itemize}

If any of these conditions are satisfied, the algorithm terminates the analysis of $\gamma^{a,\star}$. Otherwise, Algorithm~\ref{algAlloc} retrieves the edge $e_{i,i+1}^{a}$, corresponding to the move action from $v_{i}^{a}$ to $v_{i+1}^{a}$, and its collision set $\mathcal{D}_{i,i+1}^{a}$ (Algorithm~\ref{algAlloc}, lines~\ref{getedge}-\ref{getcoll}). 

Next, the algorithm compares $e_{i,i+1}^{a}$ with edges in any other allocated queue $w^{b} \in \mathcal{W} \setminus w^{a}$. If a collision is detected for at least one edge $e_{m,m+1}^{b}$ in $w^{b}$ (i.e., $e_{m,m+1}^{b} \in \mathcal{D}_{i,i+1}^{a}$), the analysis of $\gamma^{a,\star}$ is stopped (Algorithm~\ref{algAlloc}, lines~\ref{colli}-\ref{colle}). In contrast, if no collision is detected, the edge $e_{i,i+1}^{a}$ is added to $w^{a}$, $\{v_{i}^{a},v_{i+1}^{a}, \tau^{a}, D_{i,i+1}^{a}\}$ is added to $u^{a}$, and the target vertex $v_{z}^{a}$ is updated (Algorithm~\ref{algAlloc}, lines~\ref{append}-\ref{update}). 

Algorithm \ref{algAlloc} then proceeds to analyse the next action of $\gamma^{a,\star}$. Once all trajectories in $\mathcal{H}^\star$ have been completely processed, the algorithm returns the updated sets of allocated steps $\mathcal{W}$, allocated trajectories $\mathcal{U}$, and target vertices $\mathcal{V}^z$.

In accordance with the VDA5050 standard, each $a$-th active AGV stores its currently allocated queue $w^{a}$, communicated by the TM, in a local queue. The allocation gives the AGV the right-of-way in traversing the edges contained in $w^a$ without stopping (see Assumption~\ref{MovWait}). While moving along the allocated edges, each $a$-th AGV communicates its state to the TM. Upon reaching vertex $v_{l}^{a}$ from $v_{l-1}^{a}$, the path allocator deallocates the edge $e_{l-1,l}^{a}$, removing it from $w^{a}$, and allows other AGVs to allocate edges within the collision set $\mathcal{D}_{l-1,l}^{a}$. In addition, the path allocator removes the corresponding move action from the allocated trajectory $u^{a}$ and updates $\tau^{a}$ for all subsequent actions accordingly.

In case $w^{a}$ runs empty (i.e., $v_{l}^{a} = v_{z}^{a}$), the $a$-th AGV stops at the current target vertex $v_{z}^{a}$ and starts to brake beforehand.


\color{black}
\section{Deadlock Detector and Handler}
\color{black}
\label{DeadlockDetectorandHandler}

The implementation of a deadlock detection and resolution strategy is essential for minimizing AGV operational interruption. In our TM, deadlocks typically arise under two primary conditions: mainly (1) due to unpredictable events such as task updates and also (2) when the bounded horizon $\delta$ is insufficient for the ABH-CBS algorithm to resolve space-time collisions in non-corridor sectors. Deadlocks are categorized based on the nature of the precedence relationships among the involved AGVs \citep{10132864}:
\begin{itemize}
    \item \textit{Cyclic Deadlocks}, where the precedence relationships between AGVs form a circular dependency.
    \item \textit{Acyclic Deadlocks}, where the precedence relationship forms a directed acyclic graph structure.
\end{itemize}

As shown in Fig. \ref{fig:deadlock}, in this work, acyclic deadlocks typically exhibit a nested structure, where AGVs give priority to other AGVs, that are already involved in deadlock scenarios.

Hence, we propose the deadlock detector and handler modules, based on our previous work \citep{10.1007/978-3-031-76428-8_50}, that detect deadlocks and resolve them by updating the paths of the involved AGVs.

On the one hand, the deadlock detector entails analyzing the output of the current \AB{L-MAPF coordinator} instance to determine the existence of deadlocks and identifying the involved active AGVs. From the constraint set $\mathcal{M}^\star$, the deadlock detector builds the precedence graph $\mathcal{G}^P$, which models the actual precedence relationships between AGVs. The vertex set $\mathcal{V}^P(\mathcal{G}^P) = \mathcal{A}$ represents active AGVs, while the edge set $\mathcal{E}^P(\mathcal{G}^P)$ denotes their precedence relationships. Specifically, an edge $(a,b)$ is inserted in the precedence graph if the $a$-th active AGV is currently occupying its target vertex $v_{z}^a$, i.e., $\tau_{z}^a = 0$ (see Remark~\ref{EmptyAllocation}),
and a constraint $\mu^a =(a, v_{i}^a, v_{i+1}^a, [\tau^{I}, ..., \tau^{F}],b) \in \mathcal{M}^\star$ exists to resolve a space-time collision with the $b$-th AGV, such that:
\begin{itemize}
    \item $v_{i}^a = v_{z}^a$, meaning that the constraint refers to an action starting from the current target vertex $v_{z}^a$ of the $a$-th AGV.
    \item $\tau^{I} \leq D_{i,i+1}^a$, meaning that the $a$-th AGV cannot perform the move action from the target vertex to the next vertex $v_{z+1}^a$.
\end{itemize}

\begin{figure}
    \centering
    \captionsetup{justification=justified}\includegraphics[width=\linewidth]{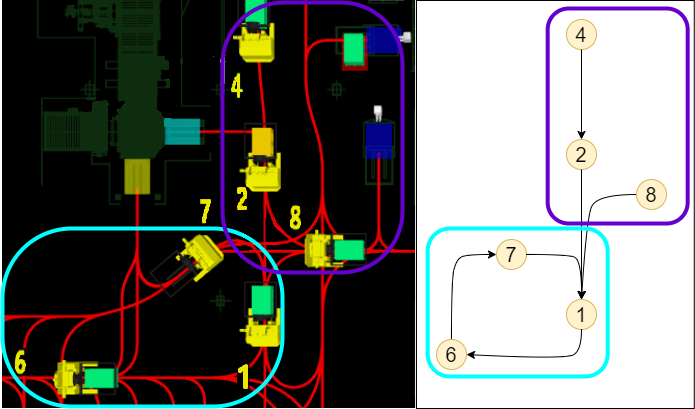}
    \caption{The figure illustrates two distinct types of deadlock scenarios. AGVs 1, 6, and 7 are involved in a circular deadlock, shown within a light blue circle, while AGVs 8, 2, and 4 are engaged in a nested acyclic deadlock, highlighted by a purple circle.}
    \label{fig:deadlock}
    \vspace{-0.3cm}
\end{figure}

After iterating over all $\mathcal{M}^\star$, the deadlock detector obtains the complete precedence graph $\mathcal{G}^P$ relative to the current \AB{L-MAPF coordinator} instance. Then, as described in \citet{10.1007/978-3-031-76428-8_50}, the deadlock detector identifies cyclic deadlocks employing the widely used Depth-First Search (DFS) algorithm \citep{cormen01introduction}, a common technique for detecting deadlocks in operating systems. The DFS algorithm begins from an initial node in the graph and recursively explores each branch as deeply as possible before backtracking to previously visited nodes. This traversal method enables DFS to efficiently navigate through graph structures while maintaining a record of visited nodes, typically in a dedicated list \citep{venkatesh1998deadlock}. A cycle is detected when the algorithm encounters a node that has already been visited within the same recursive path, indicating the presence of a deadlock. In our case, the DFS algorithm specifically identifies AGVs involved in precedence cycles within $\mathcal{G}^P$. By merging all the AGVs contained in precedence cycles, the deadlock detector obtains the set $\mathcal{B}$ containing the AGVs that require deadlock resolution due to cyclic deadlocks.

Then, differently from \citet{10.1007/978-3-031-76428-8_50}, the deadlock detector also addresses acyclic nested deadlocks using Algorithm~\ref{AcyclicDeadlockDetection} - Acyclic Deadlock Detection. The procedure iterates over the edge set $\mathcal{E}^P(\mathcal{G}^P)$, checking if there exists an edge $(a, b)$ such that $a \notin \mathcal{B}$ and $b \in \mathcal{B}$ (Algorithm~\ref{AcyclicDeadlockDetection}, lines~\ref{forab}-\ref{checkinsideB}). Upon detecting an edge $(a,b)$ that meets these conditions, node $a$ is added to the set $\mathcal{B}$, and Algorithm~\ref{AcyclicDeadlockDetection} is invoked recursively (Algorithm~\ref{AcyclicDeadlockDetection}, line~\ref{invok}). The process ultimately returns the complete set $\mathcal{B}$ of AGVs involved in deadlocks that require path replanning.

\AB{The idea of representing inter-AGV dependencies through a graph shares a conceptual similarity with the Action Dependency Graph (ADG) and the \AB{Temporal Plan Graph (TPG) used in Simple Temporal Networks (STN) approaches~\citep{Deadlock1,Deadlock2}.} In our case, however, the precedence graph employed for deadlock detection captures only the immediate precedence relations between AGVs in the current L-MAPF coordinator instance, and therefore does not model the full temporal evolution of dependencies characteristic of STN-based formulations.
}

On the other hand, once the set $\mathcal{B}$ has been defined, the deadlock handler employs the CBS algorithm \citep{sharon2015conflict}, which is both optimal and complete, to solve the deadlocks. \AB{Completeness is crucial in deadlock resolution, as it guarantees that if a feasible set of conflict-free trajectories exists for the involved AGVs, the algorithm will eventually find it, ensuring system recovery from any solvable blocking condition.} In particular, the CBS algorithm calculates a new path for each AGV in $\mathcal{B}$ to ensure that ABH-CBS can perform the AGV coordination without any blockages. 

Although CBS can be computationally intensive \citep{gordon2021revisiting}, this is not a significant concern for deadlock resolution in our AGV system. The number of AGVs involved in deadlocks is typically small, and computational constraints are less critical since the AGVs remain stationary during the resolution process.

Specifically, our CBS employs the Multi-Label A* \citep{grenouilleau2019multi} as the low-level planner to calculate the paths from $v_z^{a}$ to $v^{g,a}$ and from $v^{g,a}$ to $v^{g'_,a}$. This approach aligns with the path definition outlined in Section \ref{PathandTrajectoryDefinition}. Once a new path for each $a$-th AGV in $\mathcal{B}$ has been computed, the deadlock handler updates the respective paths in $\mathcal{P}$. Subsequently, the sets $\mathcal{K}^{\textnormal{tot}}$ and $\mathcal{S}$ are adjusted  accordingly to ensure that the next instance of \AB{the L-MAPF coordinator} is correctly initialized (see Section \ref{coordinputs}). 

\begin{algorithm} [t!]
\caption{Acyclic Deadlock Detection}
\label{AcyclicDeadlockDetection}
\KwIn{$\mathcal{B},\mathcal{G}^{P}$}
\KwOut{$\mathcal{B}$}

\ForEach{$(a,b) \in \mathcal{E}^{P}(\mathcal{G}^P)$\label{forab}}
{
\If{$a \notin \mathcal{B} \ \mathbf{and} \ b \in \mathcal{B}$\label{checkinsideB}}
{
$\mathcal{B} \gets \mathcal{B} \cup \{a\}$\\
$\mathcal{B} \gets \textbf{Algorithm \ref{AcyclicDeadlockDetection}}(\mathcal{B},\mathcal{G}^{P})\label{invok}$\\
\Return{$\mathcal{B}$}
}
}
\Return{$\mathcal{B}$}
\end{algorithm}

\AB{Although the proposed deadlock detector and handler enables automatic recovery from blocking situations, failures may still occur in rare cases where the roadmap layer does not offer the segments required for a feasible reconfiguration of the involved AGVs. Such situations typically arise in highly constrained areas of the plant, where the directed and predefined graph (see Assumption~\ref{preRoad}) provides no admissible detour or backtracking maneuver within the physical layout. When the resolution phase fails under these conditions, the system escalates the issue, requiring human operators to temporarily switch the affected AGVs to manual mode and reposition them to restore operability.
}

\section{Case Study}
\label{set}
This section describes the experimental setup (see Section~\ref{setup}) and the results (see Section \ref{results}) of the proposed traffic management system. In particular, our algorithm is developed in collaboration with \emph{Gruppo TecnoFerrari S.p.A.}, an Italian company specializing in end-of-line machinery, with a focus on warehouse and plant logistics using AGVs. 

The solution is implemented in C\#, integrated into the TecnoFerrari Supervisor software for AGV systems, which includes the Task Manager module. All experiments are conducted on a notebook equipped with an Intel Core i7-1165G7 processor (2.80 GHz clock speed) and 16 GB of DDR4-3200 RAM, and are validated in a real industrial environment in collaboration with the company (see Fig.~\ref{veicoloTecnoferrari}).
The video attached to the supplementary material (see Extension 1) shows some of the tests conducted in the automated factory.

\subsection{System Setup}
\label{setup}

\AB{The effectiveness and adaptability of the traffic management system are evaluated using three realistic layouts that simulate a palletizing, storage, and pallet‐wrapping plant located at the end of a production line. The corresponding views in the TecnoFerrari Supervisor software are shown in Fig.~\ref{Scenario1}, Fig.~\ref{Scenario2}, and Fig.~\ref{Scenario3}. The roadmaps and plant layouts, provided by \emph{Gruppo TecnoFerrari S.p.A.}, are used to generate the roadmap layer, ensuring compatibility with different AGV classes and their operational constraints. Based on this topological information, the topological layer is constructed by partitioning each plant into sectors and identifying corridor locations.}

\begin{figure}[t]
\centering
\captionsetup{justification=justified}
\includegraphics[width=\linewidth]{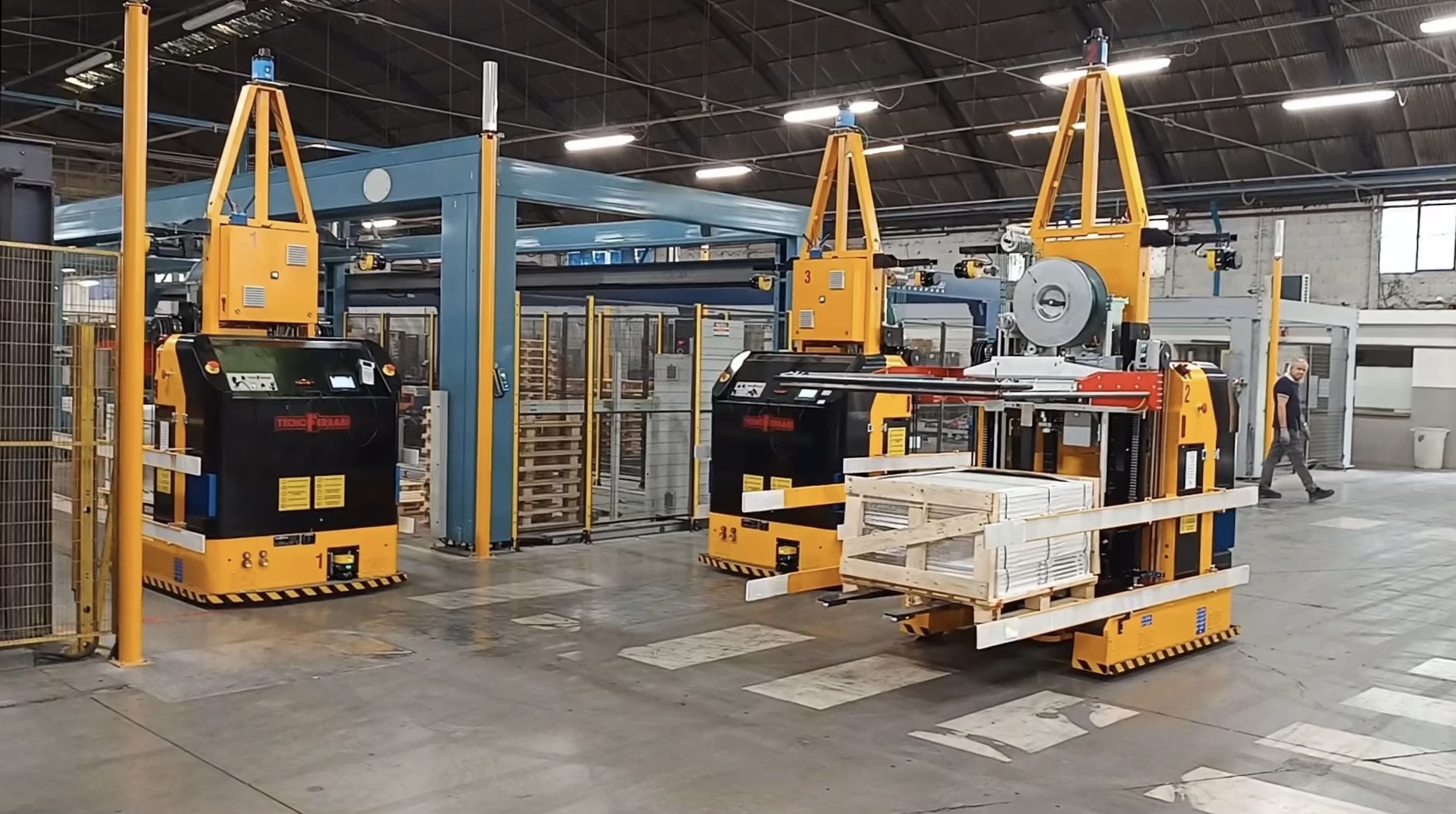}
\caption{A photo taken during the experiments conducted in the automated factory. The AGVs are coordinating in narrow and non-standardized settings.}
\label{veicoloTecnoferrari}
\end{figure} 

\AB{The three layouts represent small-, medium-, and large-scale industrial environments relevant to palletizing, storage, and pallet-wrapping plants.
The small environment is a non-standardized layout characterized by narrow dead-end bidirectional corridors and a tightly interconnected routing structure that restricts the AGV’s feasible path options. The structural constraints of the small layout inherently increase the likelihood of congestion and deadlock formation, making this environment representative of challenging coordination conditions for large AGVs.}
\AB{The medium environment preserves the non-standardized structure but is slightly larger, with wider operational areas and a broader set of interconnected sectors. Traffic density increases due to the presence of a larger fleet, while irregular geometry, asymmetric intersections, and dead-end bidirectional corridors continue to produce complex coordination conditions. This layout also includes heterogeneous AGVs, showcasing the capability of the proposed solution to handle multiple AGV classes operating simultaneously.}
\AB{The large environment corresponds to a more regular facility where heterogeneous AGVs move only along unidirectional tracks, resulting in a simpler navigation structure compared to the bidirectional corridor configurations of the smaller layouts. Although not central to the study, this layout is included to demonstrate the robustness and general applicability of the proposed traffic management system when coordinating a larger fleet in a nearly standardized layout.}

\AB{To quantitatively characterize the topological and geometric differences between the three layouts, two evaluation metrics are computed from the graphs associated with the corresponding roadmaps.}

\begin{figure}[t]
\centering
\captionsetup{justification=justified}
\includegraphics[width=\linewidth]{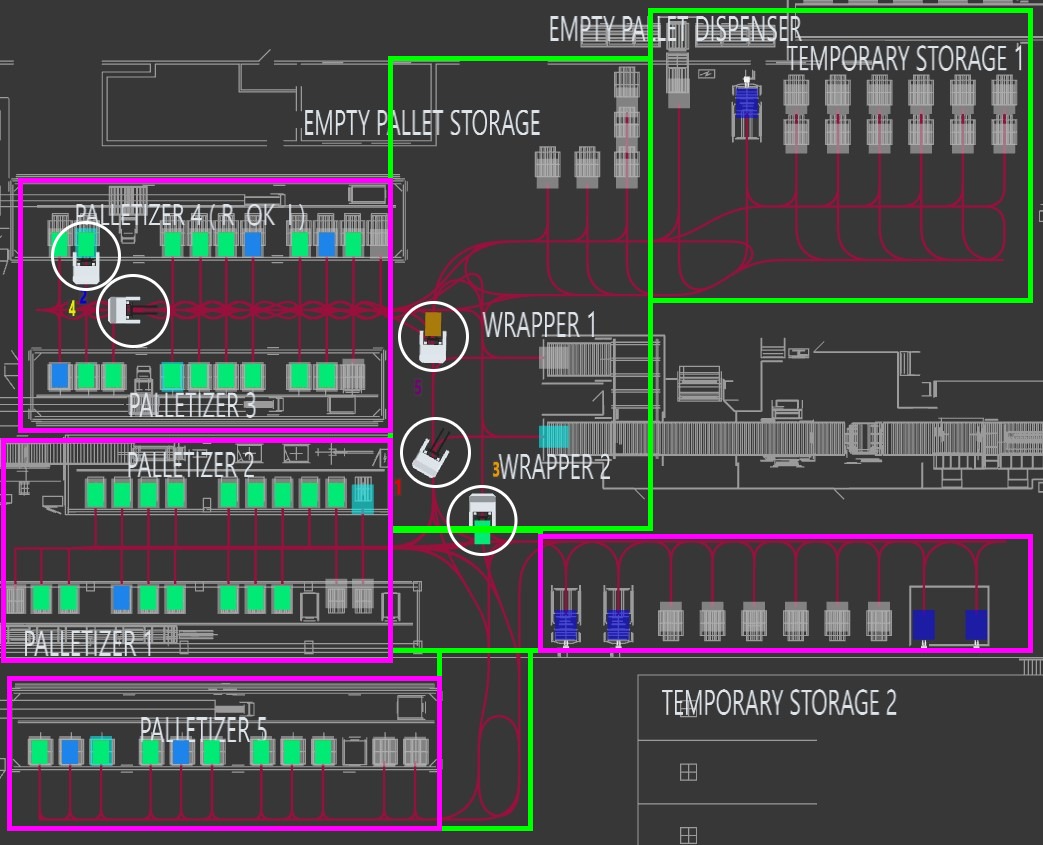}
\caption{2D reconstruction of the plant in the first \AB{layout}, with green areas delineating non-corridor sectors and magenta areas representing dead-end narrow corridors, forming a narrow, non-standardized environment, where AGVs of class $C_1$, marked by a white circle, are in motion.}
\label{Scenario1}
\end{figure}

\begin{figure}[t]
\centering
\captionsetup{justification=justified}
\includegraphics[width=\linewidth]{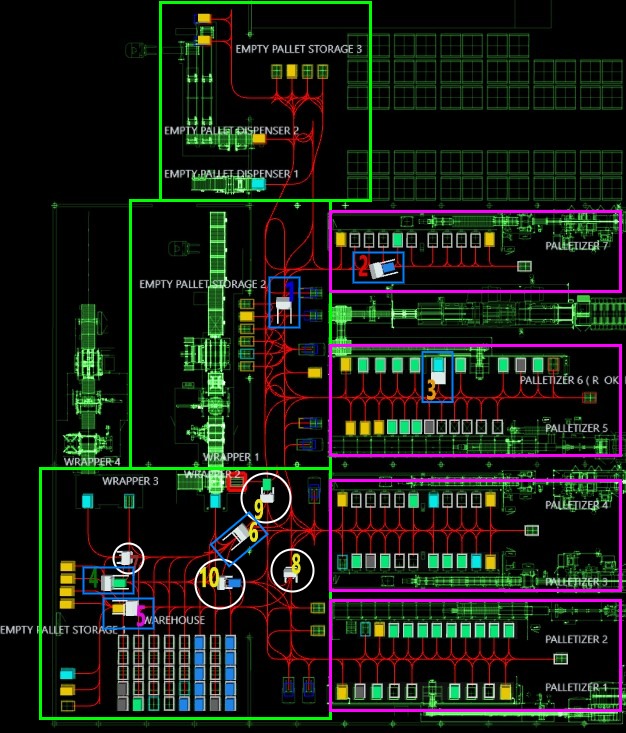}
\caption{2D reconstruction of the plant in the second \AB{layout}, with green areas delineating non-corridor sectors and magenta areas representing dead-end corridors, forming a medium-sized, high-traffic density, and non-standardized environment, where AGVs of class $C_1$ and class $C_2$, marked by a white circle and a blue square, respectively, are in motion.}
\label{Scenario2}
\end{figure}

\AB{The first metric is the \textit{Origin-Destination Betweenness Centrality} (O-D BC), a topological indicator from network science~\citep{chen2023betweenness} that measures how frequently an edge lies on the shortest paths connecting any pair of origin-destination vertices. For two generic starting $v^{s}$ and goal $v^{g}$ vertices representing station or battery-charger locations, let $N_{v^{s},v^{g}}^{e}$ denote the number of shortest paths from $v^{s}$ to $v^{g}$ that traverse edge $e$, and let $N_{v^{s},v^{g}}$ be the total number of shortest paths between them. The O-D BC of edge $e$ is defined as:}
\[
\AB{\text{O-D BC}(e) = \sum_{(v^{s}, v^{g})} \frac{N_{v^{s},v^{g}}^{e}}{N_{v^{s},v^{g}}}.}
\]
\AB{This metric is evaluated over the roadmap layer to quantify the structural relevance of each edge.}

\AB{Higher O-D BC values identify edges that play a critical role in traffic flow and are more likely to act as bottlenecks during AGV operation. For each layout's roadmap layer, the mean, variance, and standard deviation of the normalized O-D BC values are computed to characterize the distribution of centrality and to highlight regions where increased flow intensity may occur.}


\AB{The second metric is the \textit{Graph Density Index} (GDI) \citep{menniti2013}, defined for directed graphs without loops as:
\[
\mathrm{GDI} = \frac{|\mathcal{E}^{R}|}{|\mathcal{V}^{R}|(|\mathcal{V}^{R}|-1)}
\]
where $|\mathcal{E}^{R}|$ denotes the number of directed edges and $|\mathcal{V}^{R}|$ the number of vertices in the roadmap layer.
GDI quantifies the overall density of connections in the roadmap graph. In our environments, higher GDI values indicate a larger number of directed connections, primarily due to the presence of bidirectional corridors. The increased connectivity results in layouts that are more demanding from a coordination perspective.}

\AB{The indices are evaluated for the three layouts and reported in Table~\ref{tab:metrics}, where the first column identifies the layout, and the remaining columns report the BC and GDI indices.}


The first \AB{layout}, depicted in Fig. \ref{Scenario1}, represents a narrow, non-standardized environment where AGVs transport pallets from the palletizers to the wrapping machine, redirecting them to temporary storage \AB{when the wrapping machine is occupied or unavailable}. Empty pallets are supplied via a dispenser, enabling the AGVs to restock the palletizers \AB{whenever required}. The facility, with horizontal dimensions of \(60 \times 40\)~m$^2$, is divided into 8 sectors, including 4 corridors. The homogeneous AGVs, belonging to class $C_1$, have a rectangular footprint of \(2.9 \times 1.6\)~m$^2$ and are designed to navigate the confined layout efficiently, with a maximum speed of $1 \ \text{m/s}$. 
\AB{Effective task coordination is essential to maintain a continuous workflow and to limit delays in such a constrained environment.}
\AB{Table~\ref{tab:metrics} reports an O-D BC mean value of 0.281, a BC variance of 0.015, and a BC standard deviation of 0.122 for Layout~1. The values indicate that many origin-destination shortest paths tend to converge on a limited set of edges, producing a highly concentrated flow pattern and increasing the likelihood of bottleneck formation. The GDI value of 0.010 reflects a relatively high level of connectivity, largely attributable to the presence of bidirectional corridors, which introduce a greater number of edges despite the overall constrained layout geometry.}


\begin{figure*}[t]
\centering
\captionsetup{justification=justified}
\includegraphics[width=0.9\textwidth]{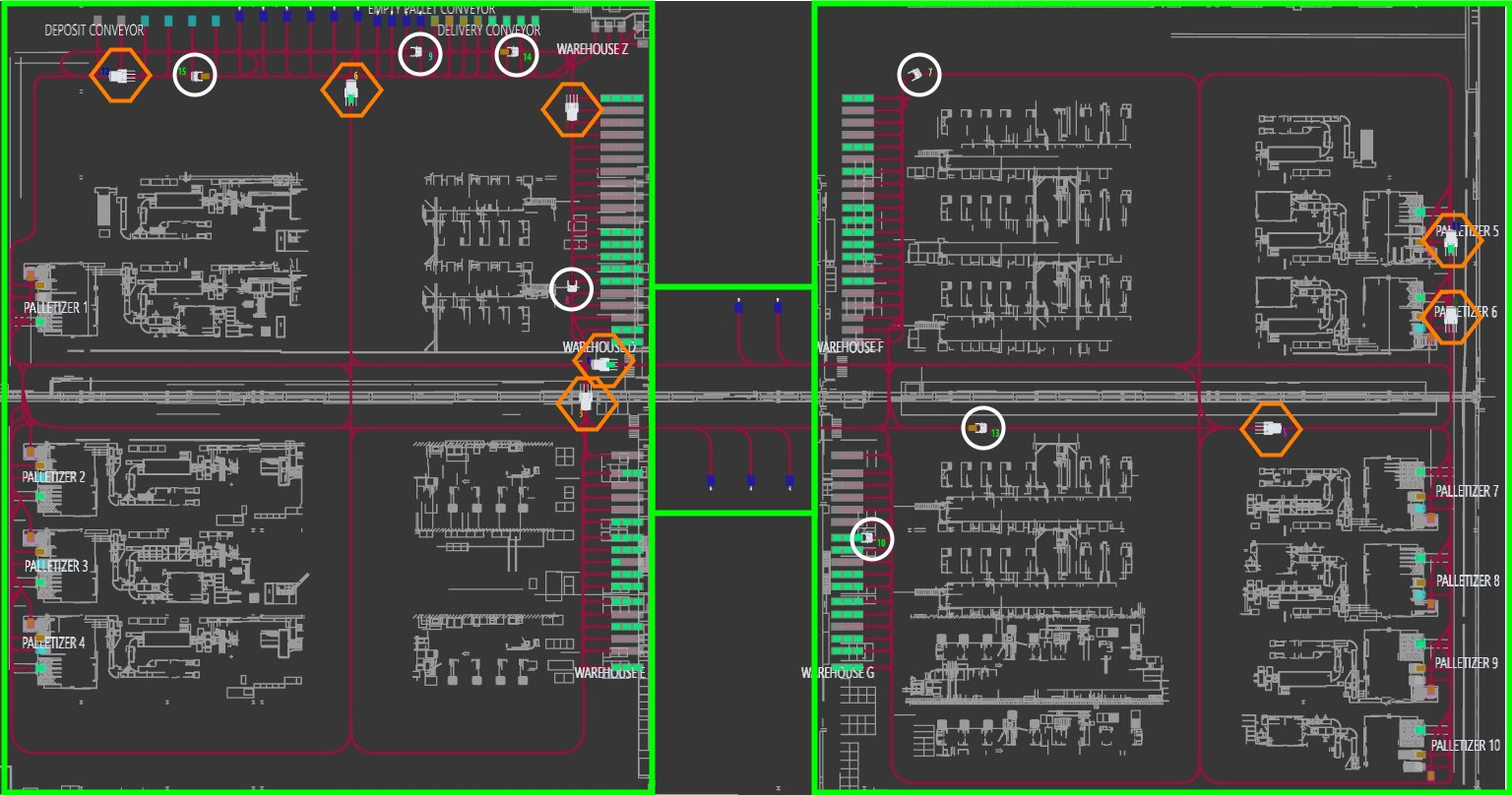}
\caption{2D reconstruction of the plant in the third \AB{layout}, representing a large environment divided by green areas that delineate non-corridor sectors, where AGVs of class $C_2$ and $C_3$, marked by a blue square and an orange hexagon, respectively, are in motion.}
\label{Scenario3}
\end{figure*}

\begin{table}[t!]
\caption{\AB{Layout evaluation metrics.}}
\centering
\begin{tabular}{c ccc c}
\toprule
\textbf{Layout} & \multicolumn{3}{c}{\textbf{O-D BC}} & \textbf{GDI} \\
\cmidrule(lr){2-4}
 & Mean & Var & Std Dev & \\
\midrule
1 & 0.281 & 0.015 & 0.122 & 0.010\\
\midrule
2 & 0.187 & 0.012 & 0.118 & 0.005\\
\midrule
3 & 0.131 & 0.016 & 0.134 & 0.002 \\ 
\bottomrule
\end{tabular}
\label{tab:metrics}
\vspace{-0.3cm}
\end{table}

The second \AB{layout} in Fig. \ref{Scenario2} models a medium-sized, non-standardized environment to assess the system’s scalability under high-traffic density.
The facility of horizontal size \(50 \times 70\) m$^2$ is also divided into 7 sectors, including 4 corridors. The AGV fleet consists of heterogeneous vehicles categorized into two classes: class $C_1$, and class $C_2$,  with a rectangular footprint of \(3.2 \times 1.8\) m\(^2\), both with a maximum speed of $1 \ \text{m/s}$. The AGVs are tasked with transporting pallets from the palletizers to the wrapping machine. In cases where the wrapping machine is occupied or unavailable, the AGVs transport pallets to a designated warehouse area. Furthermore, empty pallets are supplied from two dispensers, enabling the AGVs to restock the palletizers when required or deposit them in the empty pallet storages.
\AB{Table~\ref{tab:metrics} reports an O-D BC mean value of 0.187, a variance of 0.012, and a standard deviation of 0.118 for Layout~2. The values indicate a less concentrated centrality distribution compared to the small layout, reflecting the presence of wider areas and a routing structure that is less tightly interconnected, which allows shortest paths to distribute more evenly across the network.
The GDI value of 0.005 suggests a moderate level of connectivity, lower than in Layout~1 due to the larger graph size but still aligned with the non-standardized structure and the bidirectional corridor pattern of the environment.
}


The third layout in Fig. \ref{Scenario3} evaluates the system’s performance in a large environment of horizontal size \(210 \times 115\) m$^2$ featuring no corridor sectors. 
\AB{Although large and unconstrained facilities are not the primary focus of the study, this scenario offers insight into system behavior and performance under more regular and less restrictive operating conditions.} \AB{In particular, the layout includes AGVs belonging to class $C_3$ with a rectangular footprint of \(4.3 \times 1.5\) m\(^2\) and a maximum speed of $0.8 \text{m/s}$, operating alongside AGVs of class $C_1$.} The AGVs transport pallets from the palletizers or the deposit conveyor to the warehouses or delivery conveyor, and also provide palletizers with empty pallets from the empty pallet conveyor.
\AB{Table~\ref{tab:metrics} reports an O-D BC mean value of 0.131, a variance of 0.016, and a standard deviation of 0.134 for Layout~3. 
The values indicate a generally lower centrality level compared to the smaller layouts, as shown by the reduced mean, while the relatively high standard deviation reveals that certain edges are traversed by substantially more shortest paths than the rest of the roadmap.
The GDI value of 0.002 reflects the low graph density resulting from the use of unidirectional tracks, which provide fewer edges than the bidirectional corridors found in the smaller layouts.}



\begin{table}[t]
\centering
\caption{Parameters set during the testing phase for all three \AB{layouts}.}
\label{tab3}
\relsize{-1}
\begin{tabular}{ccccc}

    \toprule
    \textbf{Parameter} & \textbf{Description} & \textbf{\AB{Layout} 1} & \textbf{\AB{Layout} 2} & \textbf{\AB{Layout} 3}\\
    \midrule
    $\tau$ & Timestep  & 1 [s]& 1 [s]& 1 [s] \\
    \hline
    $\eta$ & Replanning  & 1 $\tau$& 1 $\tau$& 1 $\tau$ \\
     &Time& & & \\
    \hline
    $\delta'$ &Base Time& 30 $\tau$& 30 $\tau$& 30 $\tau$ \\
    &Horizon& & & \\
    \hline
    $\delta''$ & Horizon  & 10 $\tau$& 10 $\tau$& 10 $\tau$ \\
     &Increment& & & \\
    \hline
    $t^{\sigma}$ &Timeout& 250 [ms]& 500 [ms]& 500 [ms] \\
     \hline
    $\epsilon$ & Allocation &  6 $\tau$&  6 $\tau$&  10 $\tau$ \\
     &Horizon& & & \\

     \bottomrule
\end{tabular}
\vspace{-0.5cm}
\end{table}

\AB{Table~\ref{tab3} summarizes the parameters adopted across the three layouts. The timestep $\tau$ is set to 1 s (Row I), a choice aligned with common practice in AGV traffic management systems \citep{10132864}. A one-second discretization matches typical AGV speed, offers adequate temporal resolution for conflict detection and avoidance, and limits unnecessary computational overhead.}
\AB{The replanning time $\eta$ is set to 1 $\tau$ (Row II), enabling the L-MAPF coordinator to react to execution uncertainties at every timestep. This configuration corresponds to a common update frequency used in industrial AGV control systems \citep{9171550}.}
\AB{The base time horizon $\delta'$ is set to 30 $\tau$ (Row III), while the horizon increment $\delta''$ is set to 10 $\tau$ (Row IV). The initial horizon provides sufficient look-ahead for capturing short-term interactions in congested areas as shown in our previous work \citep{bonetti2024agv}, whereas the incremental extension balances computational effort and solution refinement within the anytime behavior of the L-MAPF coordinator. 
The timeout $t^{\sigma}$ values (Row V) are chosen according to the operational requirements of the three facilities and the computational constraints imposed by the TecnoFerrari Supervisor software, which executes concurrently with the traffic management system.
The allocation horizon $\epsilon$ (Row VI) is selected to match the scale and structural characteristics of each layout. The third layout adopts a larger value than the first two (Columns III-IV), reflecting the longer traversal distances and extended queues associated with its wider geometry.}
\AB{Among all parameters, $\epsilon$ is the only one not fixed by industrial conventions, hardware limits, or real-time operational requirements. It directly modulates the interaction between the path allocator and the L-MAPF coordinator by determining the portion of each AGV’s path reserved during execution, a feature introduced by the proposed system and central to the balance between workspace allocation and execution flexibility. A detailed assessment of the impact of $\epsilon$ is presented in the sensitivity analysis reported in Section~\ref{sensa}.}



\subsection{Results}
\label{results}

\AB{To evaluate the effectiveness of the proposed approach, we define a structured set of scenarios for the three layouts described above. 
For each layout, scalability is assessed through a first group of scenarios (1.A-1.C, 2.A-2.D, and 3.A-3.D), where the fleet size is progressively increased to determine the configuration that maximizes throughput. Selecting throughput-maximizing configurations aligns with the primary operational goal of AGV systems, namely completing tasks at the highest possible rate under realistic operating conditions, and identifies the fleet sizes that exploit each layout’s capacity most effectively before congestion becomes dominant.}

\AB{Based on the scalability results, the scenarios producing the highest throughput for each layout (1.B, 2.B, and 3.B) are selected as reference cases for three additional analyses. The first analysis, a sensitivity study, evaluates the impact of the allocation horizon $\epsilon$, a concept introduced in the present work, on system performance. The second analysis, an ablation study, evaluates the contribution of anytime and EHC strategies included in the proposed L-MAPF coordinator. The final analysis, a comparative evaluation, benchmarks the proposed solution against three alternative coordination methods: a traditional rule-based traffic management system commonly used in logistics operations, a state-of-the-art industrial solution, and a priority-based L-MAPF method.}



\AB{To assess the performance of the algorithm across the scenarios, we consider the following Key Performance Indicators (KPIs): } 


    \begin{itemize}
     \item \textbf{Throughput:} the average number of tasks completed per hour by the AGVs. A higher throughput value indicates improved traffic management system efficacy, as it reflects a greater number of tasks being executed within the same time span.
    \item \textbf{Average Flow Time:} the average time required to accomplish a task in the system, from allocation to completion. A lower average flow time indicates a more efficient system, as tasks are completed more quickly, reducing delays and improving overall responsiveness.
    \item \textbf{Management Efficiency:} the traffic management efficiency is defined as \( \frac{T^{M}}{T^{M} + T^{W}} \), where \( T^{M} \) represents the time AGVs spend actively moving, while \( T^{W} \) is the time spent waiting. This metric quantifies the proportion of operational time dedicated to AGV movement. A higher efficiency value indicates a more effective coordination strategy, minimizing idle time and maximizing productive AGV usage.
    \item  \textbf{Number of Deadlocks:} the number of deadlocks generated during testing phase. A lower number of deadlocks indicates a more robust coordination strategy, ensuring smoother traffic flow and minimizing disruptions in the system.
    \item \textbf{Deadlocks Detected \& Resolved Percentage:} the percentage of deadlocks effectively detected and resolved by the deadlock detector and handler modules. A higher percentage indicates a more effective resolution strategy, ensuring minimal system disruptions and maintaining continuous AGV operations.
    \item \textbf{Average Deadlock Resolution Time: } the average computational time required to solve deadlocks. A lower resolution time indicates a more efficient deadlock handling mechanism, enabling quicker recovery and minimizing delays in AGV operations.
    \item \textbf{Average Time Window (TW)}: the average value of $\delta$ reached during ABH-CBS calculation. It represents the computational effort required to solve an instance. Lower values of Average TW indicate a higher computational complexity.
    \item \textbf{Valid Solution Percentage:} the percentage of ABH-CBS instances for which a solution is found before the computational time expires. A higher percentage reflects a more reliable and efficient algorithm, ensuring timely coordination.
\end{itemize} 
\AB{In the ablation study, the KPIs associated with features specific to ABH-CBS (Average TW and Valid Solutions Percentage) are not reported, since the ablated variants do not implement anytime refinement or timeout-based termination. The same applies to the comparative analysis, where the baseline coordination methods do not include anytime expansion or a timeout stopping rule, making these metrics inapplicable. Deadlock-related KPIs are also reported only for approaches that explicitly incorporate deadlock detection and resolution mechanisms.}

\AB{The data required to evaluate the KPIs is collected during execution within the TecnoFerrari Supervisor software, where AGVs operate under continuous task assignment. Each scenario runs for approximately ten hours, providing an observation period long enough to capture representative operating conditions and to obtain statistically meaningful performance indicators under sustained workload.}

\AB{A dedicated convergence analysis is reported at the end of the section. The study focuses on Layout 1 with the fleet size of scenario 1.B, a configuration that maximizes throughput while keeping the computational load sufficiently moderate to allow systematic data collection. The selected operating point offers a representative and tractable setting for examining the anytime behavior of ABH-CBS and for quantifying its progression toward the full-horizon CBS solution. The analysis also monitors the evolution of solution quality as a function of cumulative elapsed time, providing a complementary perspective on the convergence dynamics of the solver.}



\begin{table*}[t]
\centering
\caption{\AB{
KPI values obtained for the three layouts in the scalability analysis, as the number of AGVs $N^{A}$ varies.
}}
\label{tab4}
\relsize{-2.4}
\begin{tabular}{cccccccccc}

    \toprule
    \textbf{Scenario} & $\bm{N^A}$ & \textbf{Throughput} & \textbf{Average} & \textbf{Management} & \textbf{No.} & \textbf{Deadlocks } & \textbf{Average Deadlock } & \textbf{Average} & \textbf{Valid}\\
    && \textbf{[h$^{-1}$]} & \textbf{Flow Time} & \textbf{Efficiency} & \textbf{Deadlocks} & \textbf{Detected \& Resolved } & \textbf{Resolution Time } & \textbf{TW} & \textbf{Solutions}\\
    &&  & \textbf{[s]} &  & & \textbf{Percentage [\%]} & \textbf{[ms]} & [$\boldsymbol{\tau}$] & \textbf{ Percentage [\%]}\\
         \midrule
\textbf{1.A} & 3 & 137 & 77 & 0.93 & 1 & 100 & 144 & 151.37 & 100 \\
\textbf{1.B} & 5 & 208 & 85 & 0.84 & 3 & 100 & 378 & 112.80 & 100 \\
\textbf{1.C} & 7 & 189 & 127 & 0.56 & 8 & 75 & 617 & 70.00 & 97.98 \\
\hline
\textbf{2.A} & 5 & 178 & 100 & 0.94 & 0 & -- & -- & 132.64 & 100 \\
\textbf{2.B} & 8 & 246 & 115 & 0.82 & 1 & 100 & 1765 & 80.93 & 99.82 \\
\textbf{2.C} & 10 & 283 & 125 & 0.75 & 4 & 75 & 1828 & 60.97 & 98.22 \\
\textbf{2.D}  & 12 & 242 & 178 & 0.53 & 13 & 85 & 3256 & 43.64 & 28.65 \\
\hline
\textbf{3.A} & 10 & 195 & 181 & 0.93 & 0 & -- & -- & 130.74 & 100 \\
\textbf{3.B} & 15 & 288 & 183 & 0.92 & 0 & -- & -- & 100.38 & 99.97 \\
\textbf{3.C}& 20 & 341 & 204 & 0.82 & 1 & 100 & 269 & 72.16 & 87.42 \\
\textbf{3.D}& 25 & 327 & 265 & 0.63 & 4 & 100 & 304 & 47.41 & 43.18 \\
\bottomrule

\end{tabular}

\end{table*}

\color{black}

\subsubsection{Scalability Analysis}
To assess scalability, performance is evaluated across the three layouts while progressively increasing the fleet size. The complete set of KPIs is reported in Table~\ref{tab4}. 

Across all layouts, throughput rises as additional AGVs are introduced, confirming that the proposed traffic management system can effectively exploit larger fleets even under growing traffic density. Management efficiency, however, steadily decreases: with more vehicles operating simultaneously, a larger share of time is spent waiting rather than moving. The resulting increase in flow time indicates the higher coordination effort imposed by denser operating conditions.

Each layout eventually reaches a saturation point where the negative effects of congestion outweigh the benefits of additional AGVs. In Layouts 1 and 2, the threshold appears in 1.C and 2.D: waiting times escalate, local congestion becomes persistent, and throughput no longer increases. This behavior results from the restricted routing options available in the two layouts and from the growth in coordination complexity induced by the increasing number of inter-agent interactions. Layout 3 follows the same qualitative trend, but the saturation point appears at a larger operational scale due to the wider geometry and the presence of unidirectional lanes. Throughput increases up to scenario 3.C; in 3.D, however, the pronounced rise in flow time and the reduction in management efficiency indicate that both the structural capacity of the layout and the computational load on the coordinator have reached their limits. Beyond that point, adding more AGVs no longer enhances productivity and instead intensifies congestion.

Deadlock occurrence increases consistently with fleet size across all layouts, reflecting the limited routing flexibility typical of non-standardized industrial environments. The deadlock detection and resolution module handles the vast majority of events, and no undetected deadlocks are observed. Only the most overloaded scenarios (1.C and 2.C-2.D) show a reduction in the fraction of resolved cases (down to 75-85\%), largely because the involved AGVs lack any feasible motion to resolve the blockage. 
Resolution times increase with fleet size, as a higher number of AGVs makes multi-vehicle deadlocks more likely, which naturally requires longer coordination and resolution effort.

For ABH-CBS, the average TW decreases as the fleet becomes larger. An increase in the number of agents raises both the number of potential conflicts and the amount of constraint checking, directing most of the available computation toward resolving imminent interactions rather than extending the base horizon $\delta'$. The effect is particularly evident in 2.D and 3.D, where TW drops to values close to $\delta'$. A similar pattern emerges in the valid solutions percentage: in the most congested scenarios, the solver often fails to reach $\delta'$ before the timeout, leading to the lowest recorded rates (28.65\% and 43.18\%). In all other scenarios, TW remains sufficiently large to sustain reliable conflict resolution, and the valid solutions percentage stays consistently high even under significant traffic density, confirming the robustness of the proposed approach when operating under real-time constraints.

Across the scalability analysis, 1.B, 2.C, and 3.C emerge as the throughput-maximizing configurations for Layouts 1, 2, and 3 respectively. Each configuration represents the fleet size that exploits the corresponding layout capacity most effectively while remaining below the congestion threshold reached in more populated settings. The three scenarios act as the reference operating points for the sensitivity, ablation, and comparative evaluations.

\subsubsection{Sensitivity Analysis}
\label{sensa}
\begin{table*}[t]
\centering
\caption{\AB{KPI values obtained for the three layouts in the sensitivity analysis, with the fleet size set to $N^{A} = 5$ for scenarios 1.B.a--1.B.b--1.B.c, $N^{A} = 10$ for scenarios 2.C.a--2.C.b--2.C.c, and $N^{A} = 20$ for scenarios 3.C.a--3.C.b--3.C.c.}}
\label{tab7}
\relsize{-2}
\begin{tabular}{cccccccccc}

    \toprule
    \textbf{Scenario} & $\epsilon$ & \textbf{Throughput} & \textbf{Average} & \textbf{Management} & \textbf{No.} & \textbf{Deadlocks } & \textbf{Average Deadlock } & \textbf{Average} & \textbf{Valid}\\
    && \textbf{[h$^{-1}$]} & \textbf{Flow Time} & \textbf{Efficiency} & \textbf{Deadlocks} & \textbf{Detected \& Resolved } & \textbf{Resolution Time } & \textbf{TW} & \textbf{Solutions}\\
    &&  & \textbf{[s]} &  & & \textbf{Percentage [\%]} & \textbf{[ms]} & [$\boldsymbol{\tau}$] & \textbf{ Percentage [\%]}\\
         \midrule

 \textbf{1.B.a} & 4 & 207 & 85 & 0.84 & 3 & 100 & 543 & 111.25 & 100\\
 \textbf{1.B.b} & 6 & 208 & 85 & 0.84 & 3 & 100 & 378 & 112.80 & 100 \\
 \textbf{1.B.c} & 10 & 199 & 89 & 0.80 & 4 & 100 & 235 & 116.62 & 100 \\
\hline
\textbf{2.C.a} & 4 & 283 & 124 & 0.75 & 5 & 80 & 1277 & 57.35 & 96.50\\
\textbf{2.C.b} & 6 & 283 & 125 & 0.75 & 4 & 75 & 1828 & 60.97 & 98.22 \\
\textbf{2.C.c} & 10& 276 & 128 & 0.73 & 3 & 100 & 1442 & 
66.06 & 99.34\\
\hline
\textbf{3.C.a} & 4 & 341 & 204 & 0.82 & 1 & 100 & 157 & 68.51 & 86.08\\
\textbf{3.C.b} & 6 & 340 & 205 & 0.82 & 0 & - & - & 69.13 & 86.41\\
\textbf{3.C.c}& 10 & 341 & 204 & 0.82 & 1 & 100 & 269 & 72.16 & 87.42 \\
 \bottomrule

\end{tabular}

\end{table*}

A sensitivity analysis is conducted to assess how the allocation horizon $\epsilon$ influences the performance of the proposed traffic management system. Three representative values are examined to capture different operating regimes: a short horizon ($\epsilon=4$), a medium horizon ($\epsilon=6$), and a long horizon ($\epsilon=10$). The selected values allow us to evaluate how the parameter used by the path allocator to reserve portions of the roadmap affects coordination performance. The results across all layouts are reported in Table~\ref{tab7}.

For Layouts 1 and 2, results show that $\epsilon=4$ and $\epsilon=6$ lead to nearly identical values of throughput, average flow time, and management efficiency. The only marked difference concerns ABH-CBS: with $\epsilon=6$, the solver attains a slightly larger average TW, improving predictive performance without altering the KPIs. As a result, $\epsilon=6$ emerges as the most balanced configuration in the two layouts. The improvement is driven by the longer reserved portion of the path, which shifts the starting timestep of the space-time search forward and enables deeper exploration of the temporal dimension within the same timeout.

Increasing the allocation horizon to $\epsilon=10$ in Layouts 1 and 2 leads to a slight but consistent degradation in throughput, average flow time, and management efficiency, even though the average TW becomes larger. The decline arises from the longer portion of path allocated to each AGV: an overly extended reservation of segments reduces flexibility and amplifies the effect of execution uncertainties. When an AGV slows down because of an unexpected obstacle, the roadmap element belonging to the collision set of the allocated segment remains unavailable to other AGVs for an unexpectedly long time. Since the allocator enforces non-overlapping allocations, the remaining AGVs must wait until the segment is released, which negatively affects overall coordination performance when execution uncertainties occurs. Across both layouts, the number of deadlocks remains comparable for all values of $\epsilon$, indicating that the allocation horizon primarily affects solution quality rather than the structural likelihood of deadlock formation.

In Layout 3, all tested values of $\epsilon$ lead to nearly identical operational performance, reflecting the standardized geometry and predominantly unidirectional tracks. In this setting, selecting $\epsilon=10$ does not reduce throughput or flow time and even provides a computational advantage, as ABH-CBS attains a larger average TW without compromising coordination quality. Deadlock occurrences remain similarly low for all configurations, indicating that the allocation horizon has limited influence when vehicle circulation is largely unconstrained.


\subsubsection{Ablation Study}
\begin{table*}[t]
\centering
\caption{\AB{KPI values obtained for the three layouts in the ablation study, with the fleet size set to $N^{A} = 5$ for scenarios 1.B.b--1.B.d--1.B.e, $N^{A} = 10$ for scenarios 2.C.b--2.C.d--2.C.e, and $N^{A} = 20$ for scenarios 3.C.c--3.C.d--3.C.e.}}
\label{tab6}
\relsize{-2}
\begin{tabular}{cccccccc}

    \toprule
    \textbf{Scenario} & \textbf{Approach} & \textbf{Throughput} & \textbf{Average Flow} & \textbf{Management} & \textbf{No.}  & \textbf{Deadlocks Detected \&} & \textbf{Average Deadlock }\\
     &  & \textbf{[h$^{-1}$]} & \textbf{Time [s]} & \textbf{Efficiency} & \textbf{Deadlocks} &  \textbf{Resolved Percentage [\%]} & \textbf{Resolution Time }\\
         \midrule
 \textbf{1.B.b}  & ABH-CBS with EHC & 208 & 85 & 0.84 & 3 & 100 & 378\\
 \textbf{1.B.d} & BH-CBS with EHC & 201 & 88 &  0.81 & 5 & 100 & 286\\
 \textbf{1.B.e}  &  BH-CBS & 184 & 96 & 0.74 & 17 & 94 & 547 \\
\hline
\textbf{2.C.b}  & ABH-CBS with EHC & 283 & 125 & 0.75 & 4 & 75 & 1828\\
\textbf{2.C.d}  & BH-CBS with EHC & 272  & 130 & 0.72 & 7 &  86 & 1901\\
\textbf{2.C.e} & BH-CBS & 241 & 148 & 0.63 & 26 & 85 & 2362\\
\hline
\textbf{3.C.c}  & ABH-CBS with EHC & 341 & 204 & 0.82 & 1 & 100 & 269\\
\textbf{3.C.d} & BH-CBS with EHC & 338 & 206 & 0.81 & 2 & 100 & 76\\
\textbf{3.C.e} & BH-CBS & 339 & 206 & 0.81 & 2 & 100 & 257\\
 \bottomrule

\end{tabular}

\end{table*}

To isolate the contribution of the key coordination mechanisms embedded in our L-MAPF coordinator, an ablation study is conducted across the three layouts using the throughput-maximizing scenarios, i.e., 1.B, 2.C, and 3.C.
The full version of our coordinator, denoted as ABH-CBS, augments the standard Bounded-Horizon CBS (BH-CBS) with two additional mechanisms: (i) an anytime expansion of the time window, which incrementally increases $\delta$ starting from the base value $\delta'$; and (ii) the EHC strategy, which extends the horizon $\delta^{a}$ for each $a$-th AGV as long as it operates inside narrow corridor sectors.

To quantify the impact of both mechanisms, we compare the ABH-CBS against two simplified variants. The first variant, BH-CBS + EHC, disables the anytime mechanism; in this configuration the time window cannot grow beyond the base horizon $\delta'$ in non-corridor sectors, while the EHC module remains active. As a result, the coordinator behaves as a BH-CBS outside corridors and retains full conflict-resolution capability inside corridor sectors.
The second variant, BH-CBS, removes both the anytime expansion and the EHC mechanism forcing the
coordinator to operate in all sectors with the base bounded horizon $\delta'$. This configuration represents the minimal baseline and corresponds to a classical BH-CBS without any form of horizon refinement or corridor extension.

The results in Table~\ref{tab6} indicate that the effect of the two ablations depends strongly on the layout. In Layout~1 and Layout~2, which are narrow, non-standardized, and constrained by limited maneuvering space, disabling the anytime strategy already leads to a clear performance reduction, with throughput decreasing by 3\% and 4\% in 1.B.d and 2.C.d respectively. Without horizon expansion, the L-MAPF coordinator cannot anticipate conflicts beyond $\delta'$ in non-corridor sectors, so conflicts that ABH-CBS would detect earlier are identified only when the AGVs are already close to, or partially committed to, a congested region. The delay in conflict identification increases the likelihood of deadlocks, extends waiting times, and leads to noticeable degradation in throughput, average flow time, and management efficiency. In Layout~3 the impact is negligible, because the standardized geometry and the absence of narrow corridors allow the coordinator to predict conflicts effectively even when operating with the base horizon $\delta'$.

When both the anytime strategy and the EHC mechanism are disabled, the degradation becomes substantially more severe in layouts containing narrow bidirectional corridors. With only the base horizon available, the L-MAPF coordinator cannot capture the temporal evolution of conflicts inside corridor sectors. As a result, AGVs often enter the same corridor before the impending conflict becomes detectable, leading to a large number of corridor deadlocks. The deadlock detector and handler resolve all such cases, but recovery is costly: once a deadlock forms inside a narrow corridor, the involved AGVs must stop, wait for the conflict to be processed, and perform the maneuvers required to restore feasibility before resuming their tasks. The resulting increase in travel distance and waiting time causes marked reductions in throughput and management efficiency.
In Layout~1, where several narrow bidirectional corridors connect the palletizers to the rest of the plant, the absence of EHC causes corridor deadlocks to form systematically, generating repeated interruptions and costly recovery maneuvers. Compared to 1.B.b, throughput decreases by about 12\% in 1.B.e. The effect is even more pronounced in Layout~2, where higher traffic density increases both the likelihood and the complexity of corridor conflicts, further degrading system throughput in 2.C.e of approximately 15\% compared to 2.C.b. By contrast, Layout~3 is essentially unaffected: the environment contains no narrow bidirectional corridors and traffic flows are predominantly unidirectional, so conflicts remain easy to anticipate and handle even with the base horizon $\delta'$, and the performance of the ablated configuration remains comparable to that of the full ABH-CBS coordinator.


\subsubsection{Comparative Evaluation}

\begin{table*}[t]
\centering
\caption{\AB{KPI values obtained for the three layouts in the comparative evaluation across all tested traffic-management approaches.}}
\label{tab5}
\relsize{-2}
\begin{tabular}{cccccccc}

    \toprule
    \textbf{Scenario} & \textbf{Approach} & \textbf{Throughput} & \textbf{ Average Flow} & \textbf{Management} & \textbf{No.} & \textbf{Deadlock Detector} & \textbf{Deadlocks Detected \&}\\
     & & \textbf{[h$^{-1}$]} & \textbf{Time [s]} & \textbf{Efficiency} & \textbf{Deadlocks} & \textbf{\& Handler} &  \textbf{Resolved Percentage [\%]} \\
         \midrule
\textbf{1.B.b} & \textit{ours} &208 & 85 & 0.84 & 3 & \text{\ding{51}} & 100\\
\textbf{1.B.f} & \textit{company} & 188 & 94 & 0.76 & 4 & \text{\ding{55}}& - \\
 \textbf{1.B.g} &  \citet{10132864} & 194 & 92 & 0.78 & 4 & \text{\ding{51}} & 50 \\
 \textbf{1.B.h}  &  L-MAPF coordinator with PBS & 196 & 91 & 0.79 & 7 & \text{\ding{55}} & - \\ 
\hline
\textbf{2.C.b} & \textit{ours} & 283 & 125 & 0.75 & 4 & \text{\ding{51}} & 75 \\
 \textbf{2.C.f}& \textit{company} & 255 & 139 & 0.67  & 5 & \text{\ding{55}} & - \\
\textbf{2.C.g} &  \citet{10132864} & 264 & 132 & 0.71 & 8 & \text{\ding{51}} & 37 \\
\textbf{2.C.h}   & L-MAPF coordinator with PBS & 258  & 138 & 0.68 & 11 & \text{\ding{55}} & - \\ 
\hline
\textbf{3.C.c} & \textit{ours} & 341 & 204 & 0.82 & 1 & \text{\ding{51}} & 100 \\
\textbf{3.C.f} & \textit{company} & 338 & 206 & 0.81 & 1 & \text{\ding{55}} & -\\ 
\textbf{2.C.g} &  \citet{10132864} & 337 & 208 & 0.80 & 1 & \text{\ding{51}} & 100 \\
\textbf{3.C.h}   & L-MAPF coordinator with PBS & 338  & 207 & 0.81  & 2 & \text{\ding{55}} & -  \\ 
 \bottomrule

\end{tabular}

\end{table*}

To assess the effectiveness of the proposed approach, we evaluate its performance against three representative baselines covering commonly adopted traffic management strategies in industrial AGV systems. The baselines consist of: (i) the rule-based TM deployed by \emph{Gruppo TecnoFerrari S.p.A.}; (ii) the state-of-the-art TM, i.e., the solution proposed in~\citet{10132864}; and (iii) the priority-based TM, obtained by replacing ABH-CBS with PBS \citep{ma2019searching} within the L-MAPF coordinator. The comparison is conducted on 1.B, 2.C, and 3.C, as identified in the scalability analysis.
Search-based coordination methods surveyed in Section~\ref{relwork} are not included as baselines. Their idealized assumptions, i.e., unit-time actions, single-intersection layouts with symmetric unidirectional lanes, and simplified geometric models, do not match the industrial layouts considered in this work, where AGVs follow fixed paths with heterogeneous traversal times.

\AB{The rule-based TM is included as a benchmark representative of conventional industrial coordination practice. Coordination is governed by a reactive First-Come-First-Served (FCFS) reservation strategy over roadmap segments~\citep{rule_based_traditional}, where each AGV requests access to the next segments of its path and proceeds only if the segments are not in collision with elements of the roadmap layer already reserved for other AGVs. The FCFS mechanism is augmented with handcrafted rule-based control zones that encode localized access constraints for structurally critical areas such as narrow corridors, machine entries, or intersections. Within these zones, standard FCFS reservations may be overridden by deterministic right-of-way rules, stopping conditions, or direction-dependent permissions to ensure safe traversal and avoid blocking situations.
The approach reflects a family of rule-driven coordination strategies designed for deterministic execution and straightforward deployment rather than predictive reasoning or multi-agent optimization. Its inclusion in the study provides a reference point for assessing the impact of replacing prescriptive rules with search-based coordination in non-standardized and spatially constrained environments.}


The state-of-the-art TM introduced in~\citet{10132864} constitutes, to the best of our knowledge, the only existing solution in the literature that can be deployed in non-standardized real industrial environments with high traffic density, where AGVs operate on a predefined roadmap composed of segments with heterogeneous traversal times.
The method integrates a traffic-aware hierarchical path planner with a coordination module based on a time-expanded graph that regulates agent interactions through predictive collision avoidance and precedence assignment. Since the original formulation does not support heterogeneous AGVs, an adaptation is implemented to align with the constraints of the evaluated layouts. In particular, path generation is restricted to roadmap elements admissible for each vehicle class, and replanning is invoked exclusively during task allocation and deadlock resolution to comply with the fixed-path assumption adopted throughout the experiments.

The priority-based TM isolates the contribution of the search component within the L-MAPF coordinator of our TM by replacing ABH-CBS with PBS, a two-level priority-based planner that resolves conflicts exclusively through consistent priority assignments. At the high level, PBS evaluates alternative priority assignments; at the low level, agents are planned sequentially according to the resulting priority structure, with lower-priority agents adjusting their trajectories to remain compatible with higher-priority ones. The implementation complies with Assumption~\ref{FixedPath} and adopts the same timestep, replanning time, and allocation-horizon parameters reported in Table~\ref{tab3}, ensuring consistent evaluation conditions. Constraint reasoning is not supported by PBS, and a deadlock arises whenever no priority assignment can produce feasible trajectories. Deadlocks typically emerge in swap interactions or in narrow areas of the roadmap where path compatibility is structurally restricted. Countermeasures are introduced by assigning distinct goal nodes to active AGVs and by enforcing different priority relations for the first and second task composing each path. Separate priority relations increase the chance of obtaining a feasible priority structure but do not overcome the intrinsic limitations of priority-driven planning. Deadlocks occurring under PBS remain unresolved, because the deadlock detector and handler in the proposed traffic management system rely on CBS-style constraint reasoning, which cannot be translated into the priority-only structure of PBS. As a consequence, the resolution computed by the deadlock detector and handler cannot be exploited by an L-MAPF coordinator that relies solely on priority relations. The limitation is not addressed by designing an alternative strategy for deadlock resolution, since incorporating a separate module tailored specifically to PBS would introduce behaviors and assumptions that diverge from the intended scope of the comparison. Our objective is to evaluate PBS as a pure priority-driven baseline within the L-MAPF coordinator, rather than proposing an alternative framework for traffic management.

To enable a fair comparison across the evaluated traffic management strategies, intervals during which AGVs remain blocked in undetected or unresolved deadlocks are excluded from the computation of throughput, management efficiency, and average flow time. Two of the three baselines do not incorporate deadlock detection or resolution, and including blockage intervals would impose penalties unrelated to the underlying coordination logic. Reported KPIs therefore reflect effective operating periods.


\AB{The comparative analysis shows clear advantages of the proposed TM in each scenario, as reported in Table~\ref{tab5}.}

In Layout 1, which features narrow corridors and strong spatial constraints, search-based conflict resolution provides a clear advantage, generating improvements in throughput of approximately 10\% over the rule-based TM in 1.B.f, 7\% over the state-of-the-art TM in 1.B.g, and 6\% over the priority-based TM in 1.B.h. The results indicate that the proposed traffic management system outperforms both traditional and state-of-the-art rule-based coordination strategies, as rule-driven approaches typically impose rigid precedence relations that become inefficient in such constrained environments. In contrast, the proposed TM resolves conflicts with greater flexibility and prevents the formation of bottlenecks. The performance gap with the priority-based TM can be largely attributed to the specific geometry of Layout 1. Around the wrapping area, multiple AGVs must converge into a confined region, and fixed priority chains often become overly restrictive, producing extended blocking sequences and persistent congestion.
Deadlock behavior reinforces the same trend: every deadlock occurring in 1.B.b is detected and resolved, whereas the rule-based TM and the priority-based TM leave all encountered deadlocks unresolved, and the state-of-the-art TM resolves only two of the three detected cases. The combination of effective conflict resolution and reliable deadlock recovery proves critical in a constrained environment where small disturbances rapidly propagate through interconnected corridors.

In Layout 2, the increase in traffic density and the presence of heterogeneous AGV classes amplify the differences among the evaluated coordination strategies. The rule-based TM in 2.C.f and the priority-based TM in 2.C.h tend to accumulate substantial idle time, especially when vehicles converge on shared areas and intersections. In these conditions, neither handcrafted rules nor priority relations are able to regulate high-density AGV traffic, leading to persistent waiting times and reduced system reactivity. The state-of-the-art TM in 2.C.g alleviates part of this behavior thanks to its predictive components and simple negotiation rules, although deadlocks and long waiting periods still emerge frequently under dense operating conditions. The proposed TM maintains a clearly superior operating profile in this layout. Through ABH-CBS, which incorporates the anytime and EHC strategies, bottlenecks are minimized, and the deadlock-handling mechanism resolves nearly all blocking situations. As a result, throughput increases by roughly 11\% compared to the rule-based TM, 10\% compared to the priority-based TM, and by about 7\% compared to the state-of-the-art TM, while average flow time remains consistently lower. All the baseline strategies leave several AGVs in long-term blocking conditions, underscoring the relevance of explicit and effective deadlock management and coordinated trajectory adaptation in medium-sized facilities.


In the large and regular geometry of Layout 3, where unidirectional tracks dominate and interaction density remains low, all coordination strategies achieve similar throughput in 3.C.f, 3.C.g and 3.C.h, reflecting the intrinsically limited complexity of the environment. Even in this favorable setting, the proposed TM preserves a slight performance edge. Flow time remains marginally lower, and the system resolves the only deadlock encountered during the evaluation.
The rule-based TM and the priority-based TM both experience blocking situations that cannot be recovered, generating one and two unresolved deadlocks respectively. The state-of-the-art TM successfully resolves its single deadlock but still exhibits a modest increase in flow time compared to the proposed solution. 

\color{black}

\begin{figure}[t]
\centering
\captionsetup{justification=justified}
\includegraphics[width=0.48\textwidth]{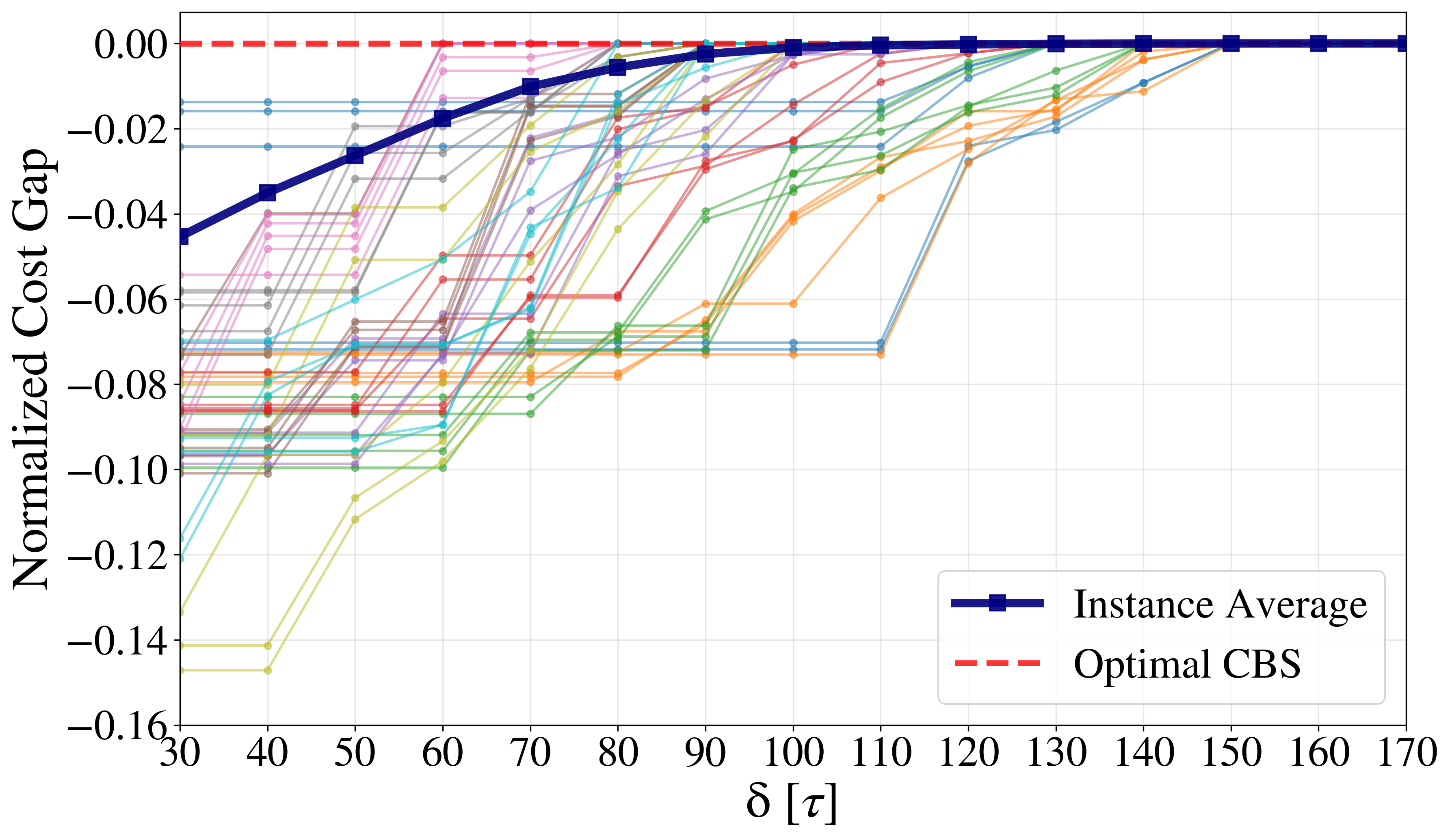}
\caption{\AB{Convergence of ABH-CBS toward the full-horizon CBS solution. Thin colored curves correspond to 50 randomly selected instances, while the thick blue curve reports the average normalized cost gap over all 3982 instances of scenario 1.B.}}
\label{Scenario3conv}
\end{figure}

\color{black}

\subsubsection{Convergence Analysis}
\label{ConvergenceAnalysis}

To empirically assess the convergence properties of ABH-CBS toward the optimal and complete CBS solution, the evolution of the solution cost is evaluated as the time horizon $\delta$ increases. ABH-CBS functions as a bounded-horizon planner whose window is selectively extended within corridor sectors through the EHC mechanism. As $\delta$ grows, the planner refines the partial solution in non-corridor regions; once $\delta$ becomes sufficiently large to cover the full execution horizon of all AGVs, ABH-CBS coincides with the full-horizon formulation of CBS and returns the optimal complete solution.


To examine this behavior in a realistic operational setting, approximately one hour of continuous operation of scenario 1.B is analyzed, yielding 3982 ABH-CBS planning instances. In each instance, ABH-CBS is executed without a timeout, allowing the bounded horizon to expand arbitrarily. During the update of the ABH-CBS solution $\mathcal{H}^{\star}$ in Algorithm~\ref{algcbs}, we record the node cost $\sum_{\gamma^{a} \in \mathcal{H}^{\star}} |\gamma^{a}|(\delta)$ for every attained value of $\delta$. Each cost value is then compared with the CBS reference cost $\sum_{\gamma^{a} \in \mathcal{H}^{\text{CBS}}} |\gamma^a|$, defined as the cost of the CBS solution $ \mathcal{H}^{\text{CBS}}$ computed with the same input configuration of the corresponding instance. The comparison is expressed through a \textit{Normalized Cost Gap (NCG)}, defined as the relative difference between the ABH-CBS cost of the collision-free trajectory at a given $\delta$ and the full CBS cost:

\[
\mathrm{NCG}(\delta) =
\frac{\sum_{\gamma^{a} \in \mathcal{H}^{\star}} |\gamma^{a}|(\delta) - \sum_{\gamma^{a} \in \mathcal{H}^{\text{CBS}}} |\gamma^a|}
     {\sum_{\gamma^{a} \in \mathcal{H}^{\text{CBS}}} |\gamma^a|}.
\]

Figure~\ref{Scenario3conv} reports the convergence curves of the normalized cost gap toward the optimal and complete CBS solution as the time horizon~$\delta$ increases. Thin colored lines correspond to 50 randomly selected instances, while the thick blue line shows the average behavior over all 3982 instances. A clear pattern emerges: as $\delta$ grows, the normalized cost gap increases monotonically and eventually approaches zero, demonstrating convergence of ABH-CBS toward the optimal CBS solution as the time window expands. The bounded-horizon formulation generally provides costs lower than the full CBS cost because, once the window is exceeded, additional conflicts are no longer detected and no further delays are inserted into the partial plan. The observed convergence reflects the design principle underlying ABH-CBS: when operating under real-time constraints, the bounded-horizon formulation focuses computation on resolving imminent conflicts and delivers high-quality solutions within the imposed timeout; as the horizon grows, e.g., in offline or relaxed-time conditions, the search progressively recovers the full CBS structure and converges to the optimal, complete solution.

\begin{figure}[t]
\centering
\captionsetup{justification=justified}
\includegraphics[width=0.48\textwidth]{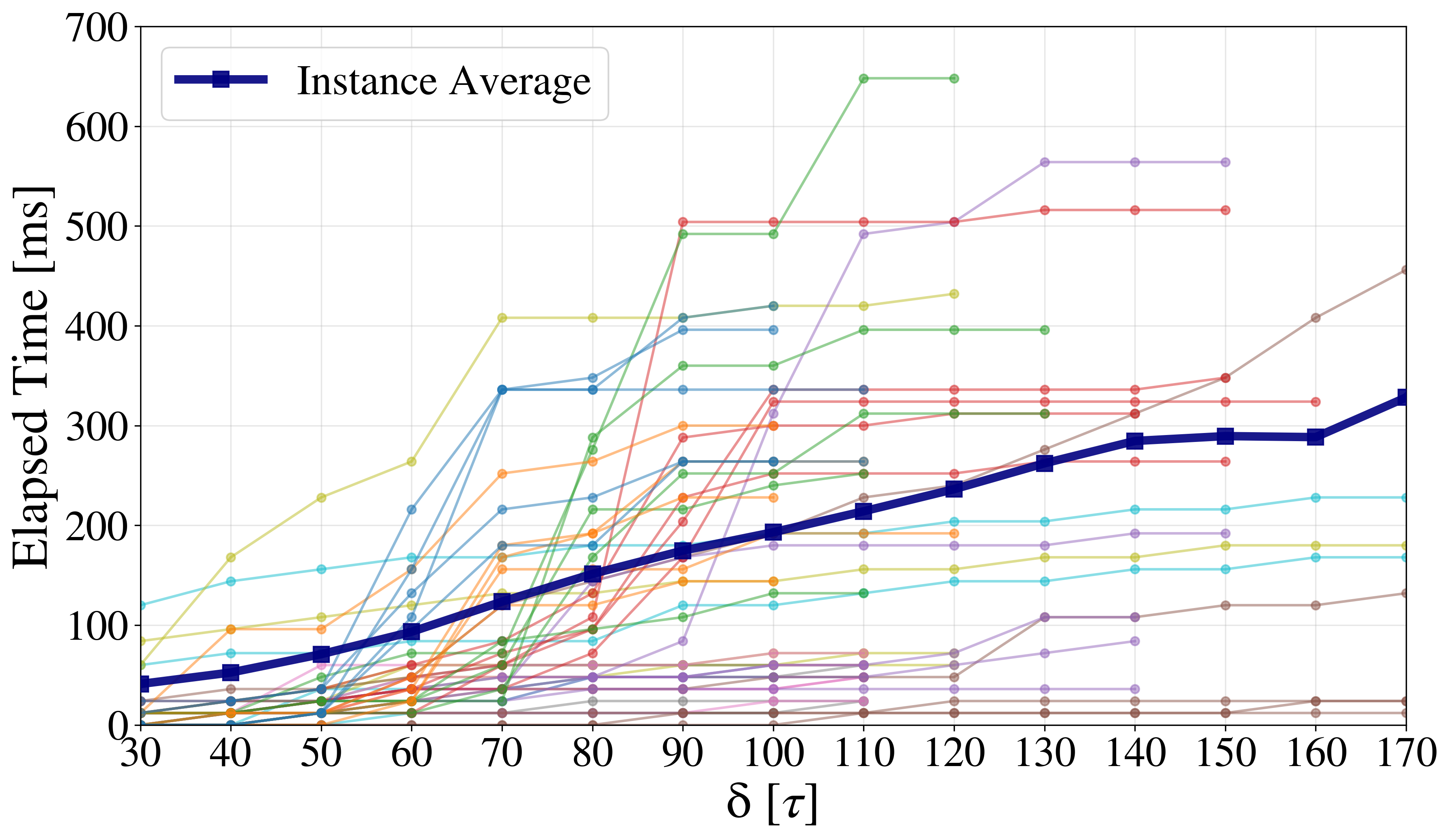}
\caption{\AB{Execution time of ABH-CBS in relation to the bounded-horizon $\delta$. Thin colored lines correspond to 50 randomly selected instances. The thick blue curve reports the average elapsed time computed over all 3982 instances in scenario 1.B for each value of the bounded-horizon parameter.}}
\label{elapsedtime1b}
\end{figure} 

In addition to evaluating convergence in terms of solution cost, the computational profile of ABH-CBS is examined by monitoring execution time as the time horizon $\delta$ increases. For each planning instance, a timer is reset at the start of the ABH-CBS run, and the elapsed time is recorded whenever $\mathcal{H}^{\star}$ is updated during the progressive expansion of the bounded horizon in Algorithm~\ref{algcbs}. The procedure yields, for every attained value of $\delta$, a corresponding execution time that reflects the cumulative computation required to refine the partial plan up to that horizon.

Figure~\ref{elapsedtime1b} reports the resulting execution-time curves. Thin colored lines correspond to 50 randomly selected instances and show the variability of the solver’s computation time as a function of the time horizon $\delta$. The thick blue curve presents the average elapsed time computed over all 3982 instances for each value of the time horizon.

For all reported instances, ABH-CBS computes an initial solution for the base time horizon $\delta'$ within $t^{\sigma} = 250$ ms. Consequently, every run reaches at least the base horizon inside the allowed timeout, which directly accounts for the 100\% valid-solution rate observed in the scalability analysis for scenario~1.B. The average elapsed time grows approximately linearly with the time horizon $\delta$, reflecting the incremental computation required to refine the plan as the bounded window expands. At the same time, execution times vary substantially across instances: some configurations are resolved within a few milliseconds, whereas others exceed $600,\mathrm{ms}$ due to higher conflict density and a correspondingly larger branching effort. The spread in computation times illustrates the dependence of the computational load on the underlying multi-agent interaction pattern, while the overall distribution confirms that ABH-CBS remains compatible with real-time operation.



\subsection{Findings}

The results obtained from the above analyses offer a unified picture of how the proposed system behaves across different operating conditions.

The proposed traffic management system effectively coordinates a heterogeneous AGV fleet on non-standardized industrial roadmaps with narrow bidirectional corridors and reduced maneuvering space, while also delivering consistent performance in more regular and spacious layouts.

Throughput grows with fleet size until each layout reaches the saturation point imposed by geometry and interaction density. Beyond that limit, additional AGVs no longer enhance productivity, although the L-MAPF coordinator continues to handle conflicts within real-time bounds except in the most congested configurations. The allocation-horizon parameter exerts limited influence on system-level performance: a broad range of values yields comparable results, with short horizons working best in irregular layouts and longer horizons remaining advantageous in more structured environments.

The ablation study indicates that both the anytime mechanism and the EHC strategy play a critical role in complex layouts. Disabling either one delays conflict detection, increases deadlock occurrences, and degrades throughput, flow time, and management efficiency. The effect is weaker in large, regular environments but remains beneficial whenever interaction density rises.

The comparative evaluation shows that the proposed system surpasses the rule-based TM, the state-of-the-art TM, and the priority-based L-MAPF variant across all
non-standardized layouts. Anticipation of conflicts, coordinated motion under high-density, and reliable deadlock detection and resolution produce consistently higher throughput and lower travel times in geometrically constrained settings.

The convergence analysis confirms that the ABH-CBS solver approaches the optimal and complete CBS solution as the time horizon expands. Execution-time measurements further show that real-time constraints are met across all instances while preserving the ability to recover the full CBS solution when additional computation is available.
\color{black}

\vspace{-0.3cm}

\section{Conclusions and Outlooks} 
\label{Concl}
In this paper, we present an innovative \ABB{architecture} for efficiently managing high-density AGV traffic in complex industrial environments. By employing a Lifelong Multi-Agent Path Finding strategy, our approach ensures robust and scalable coordination while maintaining locally optimal \ABB{and real-time} performance. Furthermore, our system guarantees safety through the path allocator module compliant with the VDA5050 standard and incorporates real-time deadlock resolution using an enhanced Conflict Based Search algorithm. 
\AB{Extensive validation in real-world layouts shows throughput improvements of up to 11\%, together with enhanced operational continuity and coordination efficiency, when compared with a conventional rule-based system, a state-of-the-art industrial solution, and an alternative L-MAPF method. Improvements are demonstrated in non-standardized settings involving large and heterogeneous AGVs operating on a roadmap, and the results indicate that the proposed traffic management system generalizes effectively to larger, less constrained environments while retaining real-time performance.}

Future work will focus on relaxing the assumption that AGVs must follow the paths established during task assignment. A first direction is the introduction of real-time adaptive path adjustments to alleviate congestion and reduce conflict frequency during execution. Another area of extension concerns online replanning to navigate around obstacles, since the current fixed-path constraint may otherwise result in prolonged waiting. \ab{Building on this perspective, deadlock handling could evolve from the reactive resolution adopted in this work toward proactive prevention, where the early identification of incipient configurations would enable anticipatory re-routing or priority reassignment. A complementary direction is the integration of task assignment and coordination within a unified decision layer, in which demand-aware allocation policies account for the expected traffic state to improve throughput and mitigate congestion.} Further research will also investigate learning-based coordination strategies capable of reproducing the behavior of the proposed system, with the aim of improving scalability and robustness in large and uncertain industrial environments.


\section{Acknowledgments}
The authors gratefully acknowledge the \emph{Gruppo TecnoFerrari S.p.A.}, and particularly Anna Bellodi, Riccardo Benedetti, and Andrea Bargi from the Software Development Office, for their support in the preparation of this paper.

\section{Declaration of Conflicting Interests}
The author(s) declared no potential conflicts of interest with respect to the research, authorship, and/or publication of this article.

\section{Funding}
The author(s) disclosed receipt of the following financial support for the research, authorship, and/or publication of this article: This work was supported by the Socially-acceptable Extended Reality Models and Systems (SERMAS) Project of the European Union’s Horizon Europe Research and Innovation Program (GA n. 101070351).

\section{ORCID iDs}
Alessandro Bonetti https://orcid.org/0009-0008-2068-6165 \\
Silvia Proia https://orcid.org/0000-0002-3801-4745 \\
Simone Guidetti https://orcid.org/0009-0002-1446-1142 \\
Lorenzo Sabattini https://orcid.org/0000-0002-2734-5549

\bibliographystyle{SageH}
\bibliography{biblio.bib}
\appendix


\section{Low-Level Planner}

\begin{algorithm} [h!]
\caption{Low-Level Planner}
\label{Low Level CBS}
\footnotesize
\KwIn{$v_z^{a}, \tau_z^a, v^{g'_,a}, \mathcal{K}^a, \mathcal{O}^{s}, \mathcal{U}^{\setminus a}, \mathcal{M}$}
\KwOut{$\gamma^a$}


$\textnormal{OPEN\_L} \gets \{(v_z
^{a}, \tau_z^a, f_z^a \gets \tau_z^a + h(v_z
^{a},v^{g'_,a}))\}$ \label{inserimento}\\
$\textnormal{CLOSED\_L} \gets \{\}$\label{closedinit}\\
$traceback \gets\{\}\label{tracebackinit}$

\While{$\textnormal{OPEN\_L} \neq \emptyset$}
{
$(v_{i}^{a}, \tau^{\rho,a}, f^{\rho,a}) \gets \text{lowest $f^{\rho,a}$ cost node from OPEN\_L}$\label{lowestcostselect}\\
\If{$v_{i}^{a} = v^{g'_,a}$\label{reconstruct1}}
{
    \Return{$GetTrajectory(traceback,v_{i}^{a}, \tau^{\rho,a})$\label{reconstruct2}}
}
$\textnormal{OPEN\_L} \gets \textnormal{OPEN\_L} \setminus (v_{i}^{a}, \tau^{\rho,a}, f^{\rho,a})$ \label{bho1}\\
$\textnormal{CLOSED\_L} \gets \textnormal{CLOSED\_L} \cup (v_{i}^{a}, \tau^{\rho,a}, f^{\rho,a})$\\
\ForEach{$v_{j}^{a} \in GetNeighbors(\mathcal{K}^a, v_{i}^{a})$ \label{bho2}}
{
$\{v_i^a, v_j^a, \tau^{\rho,a}, D_{i,j}^a\} \gets GetAction(v_i^a, v_j^a,\tau^{\rho,a},\mathcal{K}^a)$ \label{bho3}\\
$\tau^{\zeta,a} \gets \tau^{\rho,a} + D_{i,j}^{a}$\\ 
$f^{\zeta,a} \gets \tau^{\zeta,a} + h(v_j^{a},v^{g'_,a})$\label{bho4}\\

\If{$(v_{j}^{a}, \tau^{\zeta,a}, f^{\zeta,a}) \in \textnormal{CLOSED\_L}$\label{inclosed}}
{$\textnormal{continue}$}

\If{$CheckConstraints(\{v_i^a, v_j^a, \tau^{\rho,a}, D_{i,j}^a\}, \mathcal{M})$ \label{checkconstr}}
{$\textnormal{continue}$}

\If{$\textnormal{\textbf{Algorithm \ref{CheckIfSafe}}}(\{v_i^a, v_j^a, \tau^{\rho,a}, D_{i,j}^a\}, \mathcal{O}^s, \mathcal{U}^{\setminus a})$\label{checkobst}}
{$\textnormal{continue}$}

\If{ $(v_{j}^{a}, \tau^{\zeta,a}, f^{\zeta,a}) \notin \textnormal{OPEN\_L}$\label{checkifinopen}}
{
$traceback[(v_j^a, \tau^{\zeta,a})] \gets(v_i^a,\tau^{\rho,a})$ \label{lowfin1}\\
$\textnormal{OPEN\_L}\gets \textnormal{OPEN\_L} \cup (v_j^a,\tau^{\zeta,a}, f^{\zeta,a})$\label{lowfin2}\\

}
}

}
\Return{$\varnothing$}
\normalsize
\end{algorithm}


As detailed in Algorithm \ref{Low Level CBS} - Low-Level Planner, our space-time A* takes as input the current target vertex $v_z^a$, the corresponding target timestep $\tau_z^a$, the next goal vertex $v^{g'_,a}$, the path subgraph $\mathcal{K}^a$, the constraint set $\mathcal{M}$ defined by the high-level planner, and the static and dynamic obstacle sets $\mathcal{O}^s, \mathcal{U}^{\setminus a}$.
Algorithm \ref{Low Level CBS} is initialized with a start node $(v_z^{a}, \tau_z^a, f_z^a)$, where the total cost $f_z^a$ is given by the sum of $\tau_z^a$ and the heuristic value $h(v_z^a,v^{g'_,a})$. Specifically, the heuristic $h(v_i^a,v^{g'_,a})$ estimates the cost to reach $v^{g'_,a}$ from a generic vertex $v_i^a \in \pi^a$ and is defined as:
\begin{equation}
    h(v_i^a,v^{g'_,a}) = \sum_{m=i}^{L^{a}-1} D_{m,m+1}^{a}
\end{equation}
where $D_{m,m+1}^a$ represents the timestep duration of traversing the edge $e_{m,m+1}^a$. The start node is inserted into a list, referred to as the OPEN\_L list, which contains the nodes to be explored (Algorithm~\ref{Low Level CBS}, line~\ref{inserimento}). At this step, Algorithm \ref{Low Level CBS} initializes another list, referred to as the CLOSED\_L list, which is used to prevent redundant expansions by storing processed nodes (Algorithm~\ref{Low Level CBS}, line~\ref{closedinit}). In addition, a dictionary named \textit{traceback} is introduced to facilitate trajectory reconstruction upon reaching the goal vertex $v^{g'_,a}$ (Algorithm~\ref{Low Level CBS}, line~\ref{tracebackinit}).

During execution, while OPEN\_L list is not empty, Algorithm \ref{Low Level CBS} selects the current node $(v_{i}^{a}, \tau^{\rho,a}, f^{\rho,a})$ with the lowest value of $f^{\rho,a}$ from the OPEN\_L list for expansion (Algorithm~\ref{Low Level CBS}, line~\ref{lowestcostselect}). If $v_{i}^{a}$ corresponds to the goal vertex $v^{g'_,a}$, Algorithm \ref{Low Level CBS} terminates successfully invoking the \textit{GetTrajectory} function, which reconstructs the trajectory $\gamma^a$ by backtracking through the \textit{traceback} dictionary from $v^{g'_,a}$ to $v_z^a$ (Algorithm~\ref{Low Level CBS}, lines~\ref{reconstruct1}-\ref{reconstruct2}). Otherwise, Algorithm \ref{Low Level CBS} removes the current node from OPEN\_L list, adds it to the CLOSED\_L list, and retrieves the neighbour vertices $v_j^a$, with $j = i \vee i+1$, by using the \textit{GetNeighbors} function applied to the path subgraph $\mathcal{K}^a$ (Algorithm~\ref{Low Level CBS}, lines~\ref{bho1}-\ref{bho2}).

For each neighbour vertex $v_j^a$, Algorithm \ref{Low Level CBS} calculates the corresponding action $\{v_i^a, v_j^a, \tau^{\rho,a}, D_{i,j}^a\}$, the next timestep $\tau^{\zeta,a}$, and the cost $f^{\zeta,a}$ using the heuristic $h$ (Algorithm~\ref{Low Level CBS}, lines~\ref{bho3}-\ref{bho4}). Algorithm \ref{Low Level CBS} first checks whether the neighbour node $(v_{j}^{a}, \tau^{\zeta,a}, f^{\zeta,a})$ is in the CLOSED\_L list, indicating that it has already been expanded in the previous iterations (Algorithm~\ref{Low Level CBS}, line~\ref{inclosed}). If this condition is not verified, Algorithm \ref{Low Level CBS} proceeds to validate the action $\{v_i^a, v_j^a, \tau^{\rho,a}, D_{i,j}^a\}$ against the constraint set $\mathcal{M}$, using the \textit{CheckConstraints} function (Algorithm~\ref{Low Level CBS}, line~\ref{checkconstr}). If the action satisfies the constraints, Algorithm \ref{Low Level CBS} then verifies its safety by checking for collisions with static or dynamic obstacles using Algorithm \ref{CheckIfSafe} - Check Obstacles Collision (Algorithm~\ref{Low Level CBS}, line~\ref{checkobst}).

\begin{algorithm} [t!]
\caption{Check Obstacles Collision}
\label{CheckIfSafe}
\footnotesize
\KwIn{$\{v_i^a, v_j^a, \tau^a, D_{i,j}^a\}, \mathcal{O}^s, \mathcal{U}^{\setminus a}$}
\KwOut{Boolean}

\If{$v_i^a \neq v_j^a$ \label{static1}}
{
$e_{i,i+1}^{a} \gets GetEdge(v_{i}^a,v_{i+1}^a)$\\
\If{$e_{i,i+1}^a \in \mathcal{O}^s$\label{static2}}
{\Return{\textnormal{true}}}
}

\ForEach{$u^b \in \mathcal{U}^{\setminus a}$\label{dynamic1}}
{
\ForEach{$\{v_m^b, v_{m+1}^b, \tau^b, D_{m,m+1}^b\} \in u^b$}
{
\If{$\textnormal{\textbf{Algorithm \ref{checkspacetimecollisions}}}(\{v_i^a, v_j^a, \tau^a, D_{i,j}^a\},\{v_m^b, v_{m+1}^b, \tau^b, D_{m,m+1}^b\})$\label{dynamic2}}
{
\Return{\textnormal{true}}
}
}
}

\Return{\textnormal{false}}
\normalsize
\end{algorithm}

\begin{algorithm} [t!]
\caption{Check Space-Time Collision}
\label{checkspacetimecollisions}
\footnotesize
\KwIn{$\{v_i^a, v_j^a, \tau^a, D_{i,j}^a\},\{v_m^b, v_n^b, \tau^b, D_{m,n}^b\}$}
\KwOut{Boolean}
\If{$\max(\tau^{a}, \tau^{b}) \leq \min(\tau^{a} + D_{i,j}^a, \tau^{b} + D_{m,n}^b) \label{overlappingt}$}
{
\eIf{$v_i^a = v_j^a$}
{
$\mathcal{D}_{i}^a \gets GetCollisionSet(v_i^a) \label{collsionvertex}$ \\
\eIf{$v_m^b = v_n^b$}
{\label{a1}
\If{$v_m^b \in \mathcal{D}_{i}^a$}
{
\Return{\textnormal{true}}
}
}{
$e_{m,m+1}^b \gets GetEdge(v_m^b,v_{m+1}^b)$\\
\If{$e_{m,m+1}^b \in \mathcal{D}_{i}^a$}
{
\Return{\textnormal{true}}
}
\label{a2}
}

}
{
$\mathcal{D}_{i,i+1}^a \gets GetCollisionSet(v_i^a,v_{i+1}^a) \label{collsionedge}$\\
\eIf{$v_m^b = v_n^b$}
{\label{b1}
\If{$v_m^b \in \mathcal{D}_{i,i+1}^a$}
{
\Return{\textnormal{true}}
}
}{
$e_{m,m+1}^b \gets GetEdge(v_m^b,v_{m+1}^b)$\\
\If{$e_{m,m+1}^b \in \mathcal{D}_{i,i+1}^a$}
{
\Return{\textnormal{true}}
}
\label{b2}
}
}
}

\Return{\textnormal{false}}
\end{algorithm}

Specifically, Algorithm \ref{CheckIfSafe} is employed to verify whether the move or wait action (see Section \ref{PathandTrajectoryDefinition}) from the vertex $v_i^a$ to its neighbour $v_j^a$, where $j = i \vee i+1$, avoids static and dynamic obstacles. First, if the considered action is a move, Algorithm \ref{CheckIfSafe} checks if the corresponding edge $e_{i,i+1}^{a}$ belongs to the set of static obstacles $\mathcal{O}^s$ (Algorithm~\ref{CheckIfSafe}, lines~\ref{static1}-\ref{static2}). If this condition is not satisfied, the action can be safely executed without colliding with static obstructions. The check then proceeds to dynamic obstacles. In particular, if the action $\{v_i^a, v_j^a, \tau^a, D_{i,j}^a\}$ results in a space-time collision with any other allocated trajectory $u_b \in \mathcal{U}^{\setminus a}$, Algorithm \ref{CheckIfSafe} asserts that the action is unsafe (Algorithm~\ref{CheckIfSafe}, lines~\ref{dynamic1}-\ref{dynamic2}). 

The space-time collision check is performed using Algorithm \ref{checkspacetimecollisions} - Check Space-Time Collision, which evaluates whether the edges or vertices involved in two move or wait actions from different AGV trajectories are in collision with overlapping timestep intervals. Let $\{v_{i}^a, v_{j}^a, \tau^a, D_{i,j}^a\}$ be the action of the $a$-th AGV and let $\{v_m^b, v_n^b, \tau^b, D_{m,n}^b\}$ be the action of the $b$-th AGV, with $n = m \vee m + 1$. Algorithm \ref{checkspacetimecollisions} first checks if the two actions have overlapping timestep intervals (Algorithm~\ref{checkspacetimecollisions}, line~\ref{overlappingt}). Then, it retrieves the collision set $\mathcal{D}_{i}^{a}$ for the vertex $v_{i}^a$ if $\{v_{i}^a, v_{j}^a, \tau^a, D_{i,j}^a\}$ is a wait action, or $\mathcal{D}_{i,i+1}^a$ for the edge $e_{i,i+1}^a$ if it is a move action (Algorithm~\ref{checkspacetimecollisions}, line~\ref{collsionvertex} and line~\ref{collsionedge}). Next, Algorithm \ref{checkspacetimecollisions} determines whether the vertex $v_{m}^b$ or the edge $e_{m,m+1}^b$ is contained within the retrieved collision set (Algorithm~\ref{checkspacetimecollisions}, lines~\ref{a1}-\ref{a2} and lines~\ref{b1}-\ref{b2}). If this condition is met, $\{v_i^a, v_j^a, \tau^a, D_{i,j}^a\}$ and $\{v_m^b, v_n^b, \tau^b, D_{m,n}^b\}$ are in space-time collision; otherwise, they are collision-free. Hence, if the action avoids static and dynamic obstacles, Algorithm \ref{Low Level CBS} checks if the neighbour node $(v_{j}^{a}, \tau^{\zeta,a}, f^{\zeta,a})$ is absent from the OPEN\_L list (Algorithm~\ref{Low Level CBS}, line~\ref{checkifinopen}). If this condition holds, \textit{traceback} is updated and $(v_{j}^{a}, \tau^{\zeta,a}, f^{\zeta,a})$ is added to OPEN\_L list (Algorithm~\ref{Low Level CBS}, lines~\ref{lowfin1}-\ref{lowfin2}).
Algorithm \ref{Low Level CBS} ultimately returns the trajectory $\gamma^a$ for the $a$-th active AGV compliant with constraints and static and dynamic obstacles or an empty sequence if no such trajectory exists.

\color{black}
\vspace{-0.4cm}
\section{Nomenclature}

\subsection*{AGV Model}
\begin{tabular}{ll}
$N^A$ & Number of AGVs \\
$N^B$ & Number of AGV battery chargers \\
$\mathcal{A}$ & Set of active AGVs \\
$x$ & AGV x-coordinate \\
 $y$ & AGV y-coordinate \\
 $\theta$ & AGV angle \\
$q = [x,y,\theta]^\top$ & AGV pose \\
 $\mathcal{C}$ & Set of all AGV classes \\
\end{tabular}

\begin{tabular}{ll}
 $C$ & AGV class \\
 $N^{C}$ & Number of AGV classes \\
$\phi$ & Convex polygonal footprint of $C$ \\
$f(q,\dot{q})=0$ & Kinematic constraints of $C$\\
\end{tabular}

\subsection*{Environment Model}

\begin{tabular}{ll}
$\mathcal{G}^R$ & Roadmap layer graph\\
$\mathcal{V}^R$ & Set of roadmap layer vertices \\
$\mathcal{E}^R$ & Set of roadmap layer edges \\
$p$ & Location \\
$v$ & Vertex   \\
$N^{V}$ & Number of locations/vertices \\
$r$ & Segment \\
$e$ & Directed edge \\
$N^{E}$ & Number of segments/edges  \\
$\omega$ & Traversal time of edge $e$ (edge weight) \\
$D$ & Traversal timesteps duration \\
$\mathcal{G}^{C}$ & Subgraph of roadmap layer relative to $C$ \\
$g$ & Generic roadmap element (vertex or edge)  \\
$\mathcal{D}$ & Collision set of $g$ \\
$\mathcal{Y}$ & Collision set computed with Algorithm \ref{alg12} \\
$\mathcal{J}$ & Sampled poses for $g$\\
$\Phi$ & Projection of $\phi$ \\
$\mathcal{G}^{IS}$ & Topological layer graph\\
$\mathcal{V}^{IS}$ & Set of topological layer vertices\\
$\mathcal{V}^{IS}$ & Set of topological layer edges\\
$S$ & Sector \\
$N^{S}$ & Number of sectors \\
\end{tabular}

\subsection*{Paths and Trajectories}

\begin{tabular}{ll}

$v^{s}$ & Starting vertex \\
$v^{g}$ & Goal vertex \\
$v^{g'}$ & Next goal vertex \\
$\pi^{M}$ & Main path from $v_{s}$ to $v_{g}$ \\
$\pi^{N}$ & Extension path from $v_{g}$ to $v^{g'}$ \\
$\pi$ & Fixed path \\
$L^{M}$ & Length of $\pi^{M}$ \\
$L^{N}$ & Length of $\pi^{N}$\\
$L$ & Length of $\pi$ \\
$\Omega$ & Total traversal cost of $\pi$ \\
$\gamma$ & Trajectory \\
$N^{\gamma}$ & Number of actions in $\gamma$\\
\end{tabular}

\subsection*{L-MAPF Coordinator}

\begin{tabular}{ll}
$\tau$ & Timestep \\ 
$\eta$ & Replanning time \\
$\delta'$ & Base time horizon \\
$\delta''$ & Horizon increment \\
$t^\sigma$ & Timeout\\
$\mathcal{V}^z$ & Set of target vertices \\
$\mathcal{T}^{z}$ & Set of target timesteps \\
$\mathcal{V}^{g'}$ & Set of next goal vertices \\
$\mathcal{K}$ & Path subgraph of $\pi$ \\
$\mathcal{K}^{\text{tot}}$ & Set of all path subgraphs \\
$o$ & Static obstacle \\
$N^{O}$ & Number of static obstacles \\
 $\mathcal{O}$ & Set of obstructed edges by $o$  \\
 $\mathcal{O}^s$ & Set of static obstacles \\
\end{tabular}

\begin{tabular}{ll}
$\mathcal{O}^d$ & Set of dynamic obstacles \\
$\xi$ & Extended corridor\\
$\mathcal{S}$  & Set of extended corridors \\
$J$ & Objective function (Sum of Costs) \\
\end{tabular}

\subsection*{ABH-CBS}
\begin{tabular}{ll}
$t^{\iota}$ & Initial computation time \\
$\lambda$ & CBS node \\
$\mathcal{H}$ & Set of trajectories \\
$\delta$ & Time horizon \\
$\mathcal{N}$ & Set of horizons \\
$\mu$ & Constraint \\
$\mathcal{M}$ & Set of constraints
\end{tabular}

\subsection*{Path Allocator}
\begin{tabular}{ll}
$\epsilon$ & Allocation horizon \\
$w$ & Allocated queue \\
$\mathcal{W}$ & Set of allocated queues \\
$u$ & Allocated trajectory\\
$\mathcal{U}$ & Set of allocated trajectories

\end{tabular}

\subsection*{Deadlock Detector and Handler}
\begin{tabular}{ll}
$\mathcal{G}^{P}$ & Precedence graph \\
$\mathcal{V}^{P}$ & Precedence vertex set \\
$\mathcal{E}^{P}$ & Precedence edge set \\
$\mathcal{B}$ & Set of AGVs involved in deadlock \\
\end{tabular}

\subsection*{Superscripts}

\begin{tabular}{ll}
$a,b$ & AGV in $\mathcal{A}$ \\
$m,n$ & Elements in $\mathcal{G}^{R}$ \\
$b,d$ & Elements in $\mathcal{J}$ \\
$I$ & Initial value\\
$F$ & Final value \\
$s$ & Starting value \\
$g$ & Goal value\\
$z$ & Target item\\
$r$ & Root for $\lambda$ \\ 
$p$ & Parent for $\lambda$\\
$c$ & Child for $\lambda$ \\
$\star$ & Optimal value\\
\end{tabular}

\subsection*{Subscripts}
\begin{tabular}{ll}
$i,j,m,n$ & Iteration numbers for $\pi$ \\
$k$ & Iteration number for $\gamma$, $\mathcal{C}$, $\mathcal{O}$, and $\mathcal{M}$ \\
$u$ & Iteration numbers for locations \\
$h$ & Iteration numbers for segments \\
$z$ & Target index for $\pi$ \\
$l$ & Currently occupied index for $\pi$ \\
\end{tabular}

\color{black}

\end{document}